\documentclass[english, 10pt]{article}
\usepackage[T1]{fontenc}
\usepackage[latin9]{inputenc}
\usepackage{lmodern}
\usepackage{geometry}
\usepackage{amsthm}
\usepackage{amsmath}
\usepackage{amssymb}
\usepackage{bm}
\usepackage{esint}
\usepackage{natbib}
\usepackage{authblk}
\usepackage{color}
\usepackage[dvipsnames]{xcolor}
\usepackage{caption}
\usepackage{subcaption}
\usepackage{wrapfig}
\usepackage{graphicx}
\usepackage{array}
\usepackage{multirow}
\usepackage{setspace}
\usepackage{algorithm}
\usepackage{algorithmicx}
\usepackage{algpseudocode}
\usepackage[pdftex, colorlinks=true, citecolor=MidnightBlue, linkcolor=MidnightBlue, urlcolor=MidnightBlue]{hyperref}
\usepackage{setspace}
\makeatletter
  \theoremstyle{plain}
  \newtheorem{Theorem}{\protect\theoremname}
  \theoremstyle{plain}
  \newtheorem{proposition}{\protect\propositionname}
  \theoremstyle{plain}
  \newtheorem*{prop*}{\protect\propositionname}
  \theoremstyle{plain}
  \newtheorem{lemma}{\protect\lemmaname}
  \theoremstyle{definition}
  
  \theoremstyle{plain}
  \newtheorem{corollary}{\protect\corollaryname}
  \theoremstyle{plain}
  \newtheorem{example}{\protect\examplename}
 \theoremstyle{remark}

\theoremstyle{plain}
\newtheorem{fact}{\protect\factname}
\theoremstyle{plain}
\newtheorem{assumption}{\protect\assumptionname}
\theoremstyle{plain}

\newtheoremstyle{PropositionNum}
        {\topsep}{\topsep}              %%% space between body and thm
        {\itshape}                      %%% Thm body font
        {}                              %%% Indent amount (empty = no indent)
        {\bfseries}                     %%% Thm head font
        {.}                             %%% Punctuation after thm head
        { }                             %%% Space after thm head
        {\thmname{#1}\thmnote{ \bfseries #3}}%%% Thm head spec
\theoremstyle{PropositionNum}
\newtheorem{propn}{Proposition}

\newtheoremstyle{LemmaNum}
        {\topsep}{\topsep}              %%% space between body and thm
        {\itshape}                      %%% Thm body font
        {}                              %%% Indent amount (empty = no indent)
        {\bfseries}                     %%% Thm head font
        {.}                             %%% Punctuation after thm head
        { }                             %%% Space after thm head
        {\thmname{#1}\thmnote{ \bfseries #3}}%%% Thm head spec
\theoremstyle{LemmaNum}
\newtheorem{lemman}{Lemma}

\newtheoremstyle{DefinitionNum}
        {\topsep}{\topsep}              %%% space between body and thm
        {\itshape}                      %%% Thm body font
        {}                              %%% Indent amount (empty = no indent)
        {\bfseries}                     %%% Thm head font
        {.}                             %%% Punctuation after thm head
        { }                             %%% Space after thm head
        {\thmname{#1}\thmnote{ \bfseries #3}}%%% Thm head spec
\theoremstyle{DefinitionNum}

\makeatother

\newcommand{\E}{\mathbf{E}}
\newcommand{\N}{\mathbb{N}}
\newcommand{\Prob}{\mathbf{P}}
\newcommand{\hist}{\mathcal{F}_{n-1}}

\newcommand{\thetabf}{\bm \theta}
\newcommand{\alphabf}{\bm \alpha}
\newcommand{\psibf}{\bm \psi}
\newcommand{\ybf}{\bm Y}
\usepackage{babel}
  \providecommand{\assumptionname}{Assumption}
  \providecommand{\definitionname}{Definition}
  \providecommand{\lemmaname}{Lemma}
  \providecommand{\propositionname}{Proposition}
\providecommand{\corollaryname}{Corollary}
  \providecommand{\examplename}{Example}
\providecommand{\factname}{Fact}
\providecommand{\conditionname}{Condition}
\providecommand{\theoremname}{Theorem}
\begin{document}

\newcommand{\drcomment}[1]{\noindent{\textcolor{red}{\textbf{ DR:} \textsf{#1}}}}

\title{Simple Bayesian Algorithms for Best-Arm Identification}
\author{Daniel Russo}

%\affil{Kellogg School of Management,Northwestern University. daniel.russo@kellogg.northwestern.edu}

%
%{\bf TO DO}
%\begin{enumerate} 
%\item Change intro and abstract to emphasize top-two Sampling theme instead of three distinct approaches. 
%\item Change first paragraph of lit review
%\item Add Reference to new work by Kaufmann/Garivier 
%\item Add discussion of posterior concentration / anti-concentration. 
%\end{enumerate}

%\pagebreak

\maketitle
\begin{abstract}
This paper considers the optimal adaptive allocation of measurement effort for identifying the best among a finite set of options or designs.  An experimenter sequentially chooses designs to measure and observes noisy signals of their quality with the goal of confidently identifying the best design  after a small number of measurements. This paper proposes three simple and intuitive Bayesian algorithms for adaptively allocating measurement effort, and formalizes a sense in which these seemingly naive rules are the best possible. One proposal is top-two probability sampling, which computes the two designs with the highest posterior probability of being optimal, and then randomizes to select among these two. One is a variant of top-two sampling which considers not only the probability a design is optimal, but the expected amount by which its quality exceeds that of other designs. The final algorithm is a modified version of Thompson sampling that is tailored for identifying the best design.

We prove that these simple algorithms satisfy a sharp optimality property. In a frequentist setting where the true quality of the designs is fixed, one hopes the posterior definitively identifies the optimal design, in the sense that that the posterior probability assigned to the event that some other design is optimal converges to zero as measurements are collected. We show that under the proposed algorithms this convergence occurs at an exponential rate, and the corresponding exponent is the best possible among all allocation rules. It should be highlighted that the proposed algorithms depend on a single tuning parameter, which determines the probability used when randomizing among the top-two designs. Attaining the optimal rate of posterior convergence requires either that this parameter is set optimally or is tuned adaptively toward the optimal value. The paper goes further, characterizing the exponent attained on any problem instance and for any value of the tunable parameter. This exponent is interpreted as being optimal among a constrained class of allocation rules. Finally, considerable robustness to this parameter is established through numerical experiments and theoretical results. When  this parameter is set to 1/2, the exponent attained is within a factor of 2 of best possible across all problem instances.

\end{abstract}

\section{Introduction}
This paper considers the optimal adaptive allocation of measurement effort in order to identify the best among a finite set of options or designs.  An experimenter sequentially chooses designs to measure and observes independent noisy signals of their quality. The goal is to allocate measurement effort intelligently so that the best design can be identified confidently after a small number of measurements. Just as the multi-armed bandit problem crystallizes the tradeoff between exploration and exploitation in sequential decision-making, this ``pure--exploration'' problem crystallizes the challenge of efficiently gathering information before committing to a final decision. It serves as a fundamental abstraction of issues faced in many practical settings. For example:
\begin{itemize}
	\item \emph{Efficient A/B/C Testing}: An e-commerce platform is considering a change to its website and would like to identify the best performing candidate among many potential new designs. To do this, the platform runs an experiment, displaying different designs to different users who visit the site. How should the platform decide what percentage of traffic to allocate to each  website design?
	\item \emph{Simulation Optimization}: An engineer would like to identify the best performing aircraft design among several proposals. She has access to a realistic simulator through which she can assess the quality of the designs, but each simulation trial is very time consuming and produces only noisy output. How should she allocate simulation effort among the designs?
	\item \emph{Design of Clinical Trials}: A medical research organization would like to find the most effective treatment out of several promising candidates. They run a clinical trail in which they experiment with the treatments. The results of the study may influence practice for many years to come, and so it is worth reaching a definitive conclusion. At the same time, clinical trails are extremely expensive, and careful experimentation can help to mitigate the associated costs.\footnote{Interpreted the context of clinical trials, this paper's results are stated in terms of the number of patients required to reach a confided conclusion of the best treatment. However, we will see that optimal rules from this perspective also allocate fewer patients to very poor treatments, potentially leading to more ethical trials \citep{berry2004bayesian}.} Multi-armed bandit models of clinical trails date back to \cite{thompson1933}, but bandit algorithms lack statistical power in detecting the best treatment at the end of the trial \citep{villar2015multi}. Can we develop adaptive rules with better performance?
\end{itemize}
We study Bayesian algorithms for adaptively allocating measurement effort. Each begins with a prior distribution over the unknown quality of the designs. The experimenter learns as measurements are gathered, and beliefs are updated to form a posterior distribution. This posterior distribution gives a principled mechanism for reasoning about the uncertain quality of designs, and for assessing the probability any given design is optimal. By formulating this problem as a Markov decision process whose state-space tracks posterior beliefs about the true quality of each design, dynamic programming could in principle be used to optimize many natural measures of performance. Unfortunately, computing or even storing an optimal policy is usually infeasible due to the curse of dimensionality. Instead, this work \emph{proposes three simple and intuitive rules for adaptively allocating measurement effort, and by characterizing fundamental limits on the performance of any algorithm, formalizes a sense in which these seemingly na\"{i}ve rules are the best possible.}

The first algorithm we propose is called \emph{top--two probability sampling}. It computes at each time-step the two designs with the highest posterior probability of being optimal. It then randomly chooses among them, selecting the design that appears most likely to be optimal with some fixed probability, and selecting the second most likely otherwise. Beliefs are updated as observations are collected, so the top-two designs change over time. The long run fraction of measurement effort allocated to each design depends on the true quality of the designs, and the distribution of observation noise. \emph{Top--two value sampling} proceeds in a similar manner, but in selecting the top-two designs it considers not only the probability a design is optimal, but the expected amount by which its quality exceeds that of other designs. The final algorithm we propose is a top-two sampling version of the \emph{Thompson sampling} algorithm for multi-armed bandits. Thompson sampling has attracted a great deal of recent interest in both academia and industry \citep{googleanalytics, tang2013automatic, graepel2010web, chapelle2011empirical, agrawal2012analysis, kaufmann2012thompson, gopalan2014thompson, russo2014learning}, but it is designed to maximize the cumulative reward earned while sampling. As a result, in the long run it allocates almost all effort to measuring the estimated-best design, and requires a huge number of total measurements to certify that none of the alternative designs offer better performance. We introduce a natural top-two variant of Thompson sampling that avoids this issue and as a result offers vastly superior performance for the best-arm identification problem. 

Remarkably, these simple heuristic algorithms satisfy a strong optimality property. Our analysis focuses on frequentist consistency and rate convergence of the posterior distribution, and therefore takes place in a setting where the true quality of the designs is fixed, but unknown to the experimenter. One hopes that as measurements are collected the posterior distribution definitively identifies the true best design, in the sense that the posterior probability assigned to the event that some other design is optimal converges to zero. We show that under the proposed algorithms, this convergence occurs at an exponential rate, characterize the exponent attained for each problem instance, and relate this to the best possible exponent among allocation rules. 

To make a precise statement, it is important to highlight that the top-two algorithms described above depend on a tunable parameter; each method identifies the top-two designs and then flips a biased coin to decide which of these to sample. The paper's theoretical results offer a fairly complete characterization of the asymptotic performance of these algorithms, and are summarized more precisely below. 
\begin{enumerate}
\item {\bf Optimality with tuning:} For any problem instance and any choice of tuning parameter, the proposed top-two algorithms attain an exponential rate of posterior convergence. This exponent is carefully characterized. If the tuning parameter is stet optimally, the exponent is optimal among all possible adaptive allocation rules. Moreover, it is possible to attain this rate of convergence by adaptively adjusting the tuning parameter. 
\item {\bf Robustness with an unbiased coin:} Uniformly across problem instances, the exponent attained by top--two sampling with an unbiased coin is within a factor of two of what could be attained by an optimal allocation rule. This robustness is further validated through numerical experiments: across fourteen problem instances top-two Thompson sampling with an unbiased coin offers similar performance to a version of top-two Thompson sampling that is applied with the best tuning parameter for that particular problem setting. 

\item {\bf Optimality among a restricted class of allocation rules for any tuning parameter:} To simplify the discussion, imagine top--two sampling is applied with an unbiased coin. Then, as the number of measurements tends to infinity, exactly half of measurement effort is allocated to the best design. Now, consider any possible adaptive allocation rule, which, like top-two sampling, allocates half of measurement effort to the true best design asymptotically. There is no problem instance for which this alternative algorithm attains an exponential rate of posterior convergence exceeding that of the proposed top-two sampling algorithms. An analogous result applies when a biased coin is used. 
\end{enumerate}
It is worth elaborating on the third result described above, as it is the main insight that prompted this paper. We face the problem of adaptively allocating measurements among $k$ competing designs. We can imagine decomposing this problem into two parts: first the experimenter chooses which fraction of measurements to dedicate to what is believed to be the best design, and second, given this choice, she chooses how to adaptively allocate remaining measurements among the $k-1$ competing designs. Roughly speaking, this paper shows that the allocation among the remaining $k-1$ designs is handled automatically and optimally by very simple top-two sampling algorithms. This offers substantial new insight into the structure of best arm identification problems and effectively reduces the problem to the choice of a single tuning parameter--the bias of the coin used by the top-two sampling algorithms. The paper establishes a surprising degree of robustness to this tuning parameter, and shows it is possible to attain a fully optimal exponent by setting it adaptively. However, the proposed tuning method is complex, spoiling some of the elegance of the top-two sampling algorithms. The search for simpler methods stands as an interesting open question.

\subsection{Main Contributions}
This paper makes both algorithmic and theoretical contributions. On the algorithmic side, we develop three new adaptive measurement rules. The top-two Thompson sampling rule, in particular, could have an immediate impact in application areas where Thompson sampling is already in use. For example, there are various reports of Thompson sampling being used in A/B testing \citep{googleanalytics} and in clinical trials \citep{berry2004bayesian}. But practitioners in these domains typically hope to commit to a decision after definitive period of experimentation, and top-two Thompson sampling can greatly reduce the number of measurements required to do so. In addition, because of their simplicity, the proposed allocation rules can be easily adapted to treat problems beyond the scope of this paper's problem formulation. See Section \ref{sec: conclusion} for examples. 

The paper also makes several theoretical contributions. Most importantly, it is of broad scientific interest to understand when very simple measurement strategies are the best possible. This paper provides sharp links between these top-two sampling rules and the limits of performance under any adaptive algorithm. In establishing these results, we exactly characterize the optimal rate of posterior convergence attainable by an adaptive algorithm, and provide interpretable bounds on this rate when measurement distributions are sub-Gaussian. The analysis also provides several intermediate results which may be of independent interest, including establishing consistency and exponential rates of convergence for posterior distributions with non-conjugate priors and under adaptive measurement rules. It should be highlighted, however, that the results do require some strong regularity properties on the prior distribution, and in particular only apply to priors defined over a compact set.

\subsection{Related Literature}
\paragraph{Sequential Bayesian Best-Arm Identification.} There is a sophisticated literature on algorithms for Bayesian multi-armed bandit problems. In discounted bandit problems with independent arms, Gittins indices characterize the Bayes optimal policy \citep{jones1972dynamic, gittins1979bandit}. Moreover, a variety of simpler Bayesian allocation rules have been developed, including Bayesian upper-confidence bound algorithms \citep{kaufmann2012bayesian, srinivas2012information, kaufmann2016bayesian}, Thompson sampling \citep{agrawal2012analysis, Korda2013Thompson,  gopalan2014thompson, johnson2015online}, information-directed sampling \citep{russo2014learning}, the knowledge gradient \citep{ryzhov2012knowledge}, and optimistic Gittins indices \citep{gutin2016optimistic}. These heuristic algorithms can be applied effectively to complicated learning problems beyond the specialized settings in which the Gittins index theorem holds, have been shown to have strong performance in simulation, and have theoretical performance guarantees. In several cases, they are known to attain sharp asymptotic limits on the performance of any adaptive algorithm due to \cite{lai1985asymptotically}.

The pure-exploration problem studied in this paper is not nearly as well understood. Recent work has cast this problem in a decision-theoretic framework \citep{chick2009economic}. However, because the conditions required for the Gittins index theorem do not hold, computing an optimal policy via dynamic programming is generally infeasible due to the curse of dimensionality. Papers in this area typically focus on problems with Gaussian observations and priors. They   formulate simpler problems that can be solved exactly -- like a problem where only a single measurement can be gathered \citep{gupta1996bayesian, frazier2008knowledge, chick2010sequential} or a continuous-time problem with only two alternatives \citep{chick2012sequential} -- and then extend those solutions heuristically to build measurement and stopping rules in more general settings.

For problems with Gaussian priors and noise distributions,  the expected-improvement (EI) algorithm is a popular Bayesian approach to sequential information-gathering.  Interesting recent work by \cite{ryzhov2016convergence} studies the long run distribution of measurement effort allocated by the expected-improvement and shows this is related to the optimal computing budget allocation of \cite{chen2000simulation}. This contribution is very similar in spirit to this paper, as it relates the long-run behavior of a simple Bayesian measurement strategy to a notion of an approximately optimal allocation.  Unfortunately, EI cannot match the performance guarantees in this paper. In fact, under EI the posterior converges only at a polynomial rate, instead of the exponential rate attained by the algorithms proposed here and by the OCBA. See appendix \ref{sec: EI} for a more precise discussion. 

\paragraph{Classical Ranking and Selection.} 
The problem of identifying the best system has been studied for many decades under the names \emph{ranking and selection} or \emph{ordinal optimization}. A full review of this literature is beyond the scope of this article. See \cite{kim2006selecting}, \cite{kim2007recent} or \cite{hong2015discrete} for thorough reviews. Part of this literature focuses on a problem called subset-selection, where the goal is not to identify the best-design, but to find a fairly small subset of designs that is guaranteed to contain the best design. Beginning with \cite{bechhofer1954single}, many papers have focused on an indifference zone formulation, where, for user-specified $\epsilon,\delta>0$, the goal is to guarantee with probability at least $1-\delta$ the algorithm returns the true arm mean as long as no suboptimal arm is within $\epsilon$ of optimal. Assuming measurement noise is Gaussian with known variance $\sigma^2$, one can guarantee this indifference-zone criterion by gathering $O\left((\sigma k/\epsilon^2)\log(k/\delta)\right)$ total measurements, divided equally among the $k$ designs, and then returning the design with highest empirical-mean. For the case of unknown variances, \cite{rinott1978two} proposes a two stage procedure, where the first stage is used to estimate the variance of each population, and the number of samples collected from each design in the second stage is scaled by its estimated standard deviation. 
In the machine learning literature, \cite{even2002pac} studies the number of samples required by algorithms delivering $\epsilon$--PAC guarantees. Such algorithms are sometimes said to ensure a specified \emph{probability of good selection} in the terminology of the simulation optimization literature, a strictly stronger guarantee than an indifference zone guarantee \citep{ni2017efficient}.  \cite{even2002pac} show that when measurement noise is uniformly bounded, an $\epsilon$--PAC guarantee is satisfied by a sequential elimination strategy that uses only $O\left((k/\epsilon^2)\log(1/\delta)\right)$ samples on average. \cite{mannor2004sample} provide a matching lower bound. Similar to minimax bounds, this shows the upper bound of \citet{even2002pac} is tight, up to a constant factor, for a certain worst case problem instance. Indifference zone formulations of ranking and selection problems remains an area of active research. See for example \cite{fan2016indifference} and some of the references therein. 

Since \cite{paulson1964sequential}, many authors have sought to reduce the number of samples required on easier problem instances by designing algorithms that sequentially eliminate arms once they are determined to be suboptimal with high confidence. See the recent work of  \cite{frazier2014fully} and the references therein. However, in a sense described below, \cite{jennison1982asymptotically} show formally that there are problems with Gaussian observations where any sequential-elimination algorithm will require substantially more samples than optimal adaptive allocation rules. See Section \ref{sec: conclusion} for modified top-two sampling algorithms designed for an indifference zone criterion.

\paragraph{The asymptotic complexity of best-arm identification.}
We described attainable rates of performance on a worst-case problem instance characterized by \cite{even2002pac} and \cite{mannor2004sample}. A great deal of work has sought ``problem dependent'' bounds, which reveal that the best-arm can be identified more rapidly when the true problem instance is easier. This is the case, for example, when some arms are of very low quality, and can be distinguished from the best using a small number of measurements. Asymptotic measures of the complexity of best-arm identification appear to have been derived independently in  statistics \citep{chernoff1959sequential, jennison1982asymptotically}, simulation optimization \citep{glynn2004large}, and, concurrently with this paper, in the machine learning literature \citep{garivier2016optimal}. Each of these papers studies a slightly different objective, but each captures a notion of the number of samples required to identify the best-arm as a function of the problem instance -- i.e. as a function the number of designs, each design's true quality, and the distribution of measurement noise.

\cite{glynn2004large} build on the optimal-computing-budget allocation (OCBA) of \cite{chen2000simulation} to provide a rigorous large-deviations derivation of the optimal fixed allocation. In particular, assuming the design with the highest empirical mean is returned, there is a fixed allocation under which the probability of incorrect selection decays exponentially, and the exponent is optimal under all fixed-allocation rules.  The setting studied by this paper is often called the ``fixed-budget'' setting in the recent multi-armed bandit literature. Unfortunately, it may be difficult to implement the allocation in \cite{glynn2004large} without additional prior knowledge. Later work by \cite{glynn2015ordinal} provides a substantial discussion of this issue. 

 %This gives great insight in the best-arm identification problem, but unfortunately, like the OCBA, the suggested allocation depends on the unknown true-quality of designs and therefore cannot be implemented. \cite{glynn2015ordinal} shows that as a result the rates suggested in \cite{glynn2004large} are not attainable in general, leaving open the question of whether they are attainable under additional assumptions. The setting studied by \cite{glynn2004large} is often called the ``fixed-budget'' setting in the recent multi-armed bandit literature. 

This paper was highly influenced by a classic paper by \cite{chernoff1959sequential} on the sequential design of experiments for binary hypothesis testing. Chernoff's asymptotic derivations give great insight into best-arm identification, which can be formulated as a multiple-hypothesis testing problem with sequentially chosen experiments, but surprisingly this connection does not seem to be discussed in the literature. Chernoff looks at a different scaling than \cite{glynn2004large}. Instead of taking the budget of available measurements to infinity, he allows the algorithm to stop and declare the hypothesis true or false at any time, but takes the cost of gathering measurements to zero while the cost of an incorrect terminal decision stays fixed. He constructs rules that minimize expected total costs in this limit. Chernoff makes restrictive technical assumptions, some of which have been removed in subsequent work \citep{albert1961sequential,kiefer1963asymptotically,keener1984second,nitinawarat2013controlled,naghshvar2013active}.

\cite{jennison1982asymptotically} study an indifference zone formulation of the problem of identifying the best-design. Like  \cite{chernoff1959sequential}, they allow the algorithm to stop and return an estimate of the best-arm at any time, but rather than penalize incorrect decisions, they require that the probability correct selection (PCS) exceeds $1-\delta>0$ for every problem instance. Intuitively, the expected number of samples required by an algorithm satisfying this PCS constraint must tend to infinity as $\delta \to 0$. In the case of Gaussian measurement noise, \cite{jennison1982asymptotically} characterize the optimal asymptotic scaling of expected number of samples  in this limit. The recent multi-armed bandit literature refers to this  formulation as the ``fixed-confidence'' setting.    

A large body of work in the recent machine learning literature 
has sought to characterize various notions of the complexity of best-arm identification \citep{even2002pac, mannor2004sample, audibert2010best, gabillon2012best, karnin2013almost, jamieson2014best}.  However, upper and lower bounds match only up to constant or logarithmic factors, and only for particular hard problem instances. Substantial progress was presented by \cite{kaufmann2013information} and \cite{kaufmann2014complexity}, who seek to exactly characterize the asymptotic complexity of identifying the best arm in both the fixed-budget and fixed-confidence settings. Still, the upper and lower bounds presented there do not match. A short abstract of the current paper appeared in the 2016 Conference on Learning Theory. In the same conference, independent work by \cite{garivier2016optimal} provided matching upper and lower bounds on the complexity of identifying the best arm in the ``fixed-confidence'' setting. Like the present paper, but unlike \cite{jennison1982asymptotically}, these results apply whenever observation distributions are in the exponential family and do not require an indifference zone. 

The current paper looks at a different measure.  We study a frequentist setting in which the true quality of each design is fixed, and characterize the rate of posterior convergence attainable for each problem instance. We also describe, as a function of the problem instance, the long-run fraction  of measurement effort allocated to each design by any algorithm attaining this rate of convergence. These asymptotic limits turn out to be closely related to some of the aforementioned results. In particular, the optimal exponent given in Subsection \ref{subsec: optimal allocation}  mirrors the complexity measure of \cite{chernoff1959sequential}. In the same subsection, this exponent is then simplified into a form that mirrors one derived by \cite{glynn2004large}, and, for Gaussian distributions, one derived by \cite{jennison1982asymptotically}.

\paragraph{Optimal Budget Allocations.}
While the complexity measure we derive is similar to past work, the proposed algorithms differ substantially. The allocation rules proposed by \cite{chernoff1959sequential}, \cite{jennison1982asymptotically} and \cite{glynn2004large} are essentially developed as a means of proving certain rates are attainable asymptotically. To derive these policies, the authors begin with a thought experiment: assuming the experimenter actually knew the true quality of every arm, what proportion of measurements should she allocate to each arm in order to gather the most definitive evidence concerning the identity of the optimal arm. One approach to constructing such rules in practice is to use some fraction of samples to estimate the arm means and then apply the asymptotically optimal sampling proportions assuming these estimates to be correct. Such an approach dates back to at least the work of \cite{kiefer1963asymptotically}, which followed Chernoff's work on the sequential design of experiments. 

Early authors made a point to highlight limitations of such an approach. \cite{jennison1982asymptotically} writes their proposed procedures ``typically...do not have good small sample size properties. A better procedure would have several stages and a more sophisticated sampling rule.'' In a 1975 review of the sequential design of experiments,  \cite{chernoff1975approaches} notes that asymptotic approaches to the optimal sequential design of experiments had been fairly successful in circumventing the need to compute Bayesian optimal designs via dynamic programming,  but ``the approach is very coarse for moderate sample size problems.'' He writes that two-stage procedures of \cite{kiefer1963asymptotically}, ``sidestep the issue of how to experiment in the early stages,'' while constructing the optimal allocations based on point estimates  ``treats estimates of $\thetabf$ based on a few observations with as much respect as that based on many observations.'' 

Closely related to these approaches is a large body of work on optimal computing budget allocations (OCBA) \citep{chen2000simulation}. Most of this literature studies problems with Gaussian observations. They derive an approximation to the optimal sampling proportions presented in \cite{chernoff1959sequential}, \cite{jennison1982asymptotically} and \cite{glynn2004large}, which appears to simplify computation. This allocation is often stated to be optimal as the number of arms grows large; more rigorous results to this effect are established in interesting work by \cite{pasupathy2015stochastically}, who shows that the sampling ratios of the OCBA coincide with those of \cite{glynn2004large} in the limit of a sequence of problem instances in which the number of arms tends to infinity but all suboptimal arms' means are bounded away from optimal by a fixed constant. Optimal budget allocations have been extended in various directions, for example to address Bayesian expected loss objectives \citep{chick2001new}, the problem of identifying an optimum subject to stochastic constraints \citep{hunter2013optimal}, and the problem of identifying the top $m$ alternatives \citep{chen2008efficient}. See \cite{chen2015ranking} for a more thorough review. 

In this paper, we study simple adaptive allocation rules which, ostensibly, have no relation to the asymptotic calculations used to derive these optimal budget allocations. The main insight is that these simple algorithms automatically adapt their measurement effort in such a way that their long run behavior is deeply linked to the ratios suggested in the work of \cite{chernoff1959sequential}, \cite{jennison1982asymptotically}, and \cite{glynn2004large}. A major advantage of top-two sampling algorithms, however, is that asymptotic analysis is used only to give insight into the algorithms, and any approximations have no impact on their practical performance. A suite of experiments in Section \ref{sec: further experiments} suggest the approach can substantially outperform the optimal allocations derived from asymptotic theory.

\section{Problem Formulation}
Consider the problem of efficiently identifying the best among a finite set of designs based on noisy sequential measurements of their quality. At each time $n \in \mathbb{N}$, a decision-maker chooses to measure the design $I_n \in \{1,...,k\}$, and observes a measurement $Y_{n,I_n}$. The measurement $Y_{n,i}\in\mathbb{R}$ associated with design $i$ and time $n$ is drawn from a fixed, unknown, probability distribution, and the  vector $\ybf_n \triangleq (Y_{n,1},...,Y_{n, k})$ is drawn independently across time. The decision-maker chooses a \emph{policy}, or \emph{adaptive allocation rule}, which is a (possibly randomized) rule for choosing a design $I_n$ to measure as a function of past observations $I_1, Y_{1, I_1}, ... I_{n-1}, Y_{n-1, I_{n-1}}$. The goal is to efficiently identify the design with highest mean.

We will restrict attention to problems where measurement distributions are in the canonical one dimensional exponential family. The marginal distribution of the outcome $Y_{n,i}$ has density $p(y | \theta_i^*)$ with respect to a base measure $\nu$, where $\theta_i^* \in \mathbb{R}$ is an unknown parameter associated with design $i$. This density takes the form
\begin{equation} \label{eq: exponential family density}
p(y| \theta) = b(y)\exp\{ \theta T(y) - A(\theta)\} \qquad \theta \in \mathbb{R}
\end{equation}
where $b$, $T$, and $A$ are known functions, and $A(\theta)$ is assumed to be twice differentiable. We will assume that $T$ is a strictly increasing function so that  $\mu(\theta) \triangleq \intop y p(y|\theta) d\nu(y)$ is a strictly increasing function of $\theta$. Many common distributions can be written in this form, including Bernoulli, normal (with known variance), Poisson, exponential, chi-squared, and Pareto (with known minimal value).

Throughout the paper, $\thetabf^* \triangleq \left(\theta_1^*,...,\theta_k^* \right)$ will denote the unknown true parameter vector, and $\thetabf$ and $\thetabf'$ will be used to denote possible alternative parameter vectors. Let $I^* = \arg\max_{1\leq i \leq k} \theta^*_i$ denote the unknown best design. We will assume throughout that $\theta^*_i \neq \theta^*_j$ for $i\neq j$ so that $I^*$ is unique, although this can be relaxed by considering an indifference zone formulation where the goal is to identify an $\epsilon$--optimal design, for some specified tolerance level $\epsilon>0$.

%Based on the sequence of observations $I_1, Y_{1, I_1}, ... I_{n-1}, Y_{n-1, I_{n-1}}$ the decision-maker selects a new alternative $I_{n}$ to measure, and a guess $\hat{I}_{n}$ of the identity of the best alternative $I^*$.  More formally, a randomized policy makes use of a sequence of random variables $\{\xi_{n}\}_{n\in \N}$ that is independent of $\{ Y_n \}_{n \in \N}$. A  randomized policy $\psi$ is a sequence of functions $\{ \psi_{n}\}_{n\in\N}$ where for each $n$, $\psi_{n}$ maps a history of observations $I_1, Y_{1, I_1}, ... I_{n-1}, Y_{n-1, I_{n-1}}$ and $\xi_n$ to a pair $(I_{n}, \hat{I}_{n}) \in \{1,...,k\}^2$.

%computes a guess $\hat{I}_{n}$ of the identity of the best alternative $I^*(\thetabf)$, and decides whether to continue sampling or stop and return the estimate $\hat{I}_n$.

%It's useful to think of a policy for this problem consisting of an \emph{allocation rule}, a \emph{prediction rule}, and a \emph{stopping rule}. The allocation rule determines, based on the information in $\mathcal{F}_n$, a probability of measuring each alternative.  The prediction rule determines the sequence $\hat{I}_1, \hat{I_2},...$ and the stopping rule determines the stopping time $\tau$.

\paragraph{Prior and Posterior Distributions.}
The policies studied in this paper make use of a prior distribution $\Pi_{1}$ over a set of possible parameters $\Theta$ that contains $\thetabf^*$. Based on a sequence of observations $(I_1, Y_{1, I_1},...,I_{n-1}, Y_{n-1, I_{n-1}})$, beliefs are updated to attain a posterior distribution $\Pi_n$. We assume $\Pi_1$ has density $\pi_1$ with respect to Lebesgue measure. In this case,  the posterior distribution $\Pi_n$ has corresponding density
\begin{equation}\label{eq: posterior density}
\pi_{n}(\thetabf) = \frac{ \pi_{1}(\thetabf) L_{n-1}(\thetabf) }{\intop_{\Theta} \pi_{1}(\thetabf') L_{n-1}(\thetabf') d\thetabf' } \quad n\geq 2,
\end{equation}
where
\[
L_{n-1}(\thetabf) = \prod_{l=1}^{n-1} p(Y_{l, I_l} | \theta_{I_l})
\]
is the likelihood function. While this formulation enforces some technical restrictions to facilitate theoretical analysis, it allows for very general prior distributions, and in particular allows for the quality of different designs to be correlated under the priors.

\paragraph{Optimal Action Probabilities.}

Let
\[
\Theta_i \triangleq \left\{ \thetabf \in \Theta \bigg| \theta_i > \underset{j \neq i}{\max} \theta_j  \right\}
\]
denote the set of parameters under which design $i$ is optimal, and let
\begin{equation}\label{eq: optimal action prob}
\alpha_{n,i}\triangleq \Pi_{n}(\Theta_i) = \intop_{\Theta_i} \pi_n(\thetabf) d\thetabf
\end{equation}
denote the posterior probability assigned to the event that action $i$ is optimal. Our analysis will focus on
$\Pi_{n}( \Theta_{I^*}^{c}) = 1-\alpha_{I^*}$,
which is the posterior probability assigned to the event that an action other than $I^*$ is optimal.  The next section will introduce policies under which $\Pi_{n}( \Theta_{I^*}^{c}) \to 0$ as $n \to \infty$, and the rate of convergence is essentially optimal.

\paragraph{Further Notation.}
Before proceeding, we introduce some further notation. Let $\mathcal{F}_{n}$ denote the sigma algebra generated by
$(I_1, Y_{1, I_1}, ... I_{n}, Y_{n, I_{n}})$. For all $i \in \{1,...,k\}$ and $n \in \N$, define
\[
\psi_{n, i} \triangleq \Prob(I_n = i |  \hist) \qquad \Psi_{n,i} \triangleq \sum_{\ell=1}^{n} \psi_{\ell,i} \qquad  \overline{\psi}_{n,i} \triangleq n^{-1}\Psi_{n, i} .
\]
Each of these measures the effort allocated to design $i$ up to time $n$.

\section{Algorithms}

This section proposes three algorithms for allocating measurement effort. Each depends on a tuning parameter $\beta>0$, which will sometimes be set to a default value of $1/2$. Each algorithm is based on the same high level principle. At every time step, each algorithm computes an estimate $\hat{I} \in \{1,...,k\}$ of the optimal design, and measures that with probability $\beta$. Otherwise, we consider a counterfactual: in the (possibly unlikely) event that $\hat{I}$ is not the best design, which alternative $\hat{J}\neq \hat{I}$ is most likely to be the best design? With probability $1-\beta$, the algorithm measures the alternative $\hat{J}$. The algorithms differ in how they compute $\hat{I}$ and $\hat{J}$. The most computationally efficient is the modified version of Thompson sampling, under which $\hat{I}$ and and $\hat{J}$ are themselves randomly sampled from a probability distribution.

We will see that asymptotically all three algorithms allocate fraction $\beta$ of measurement effort to measuring the estimated-best design, and the remaining fraction to gathering evidence about alternatives. The algorithms adjust how measurement effort is divided among these alternative designs as evidence is gathered, allocating less effort to measuring clearly inferior designs and greater effort to measuring designs that are more difficult to distinguish from the best.

\subsection{Top-Two Probability Sampling (TTPS)}
With probability $\beta$, the top-two probability sampling (TTPS) policy plays the action $\hat{I}_{n} = \arg\max_{i} \alpha_{n,i}$ which, under the posterior, is most likely to be optimal. When the algorithm does not play $\hat{I}_n$, it plays the most likely alternative $\hat{J}_n= \arg\max_{j \neq \hat{I}_n} \alpha_{n,j}$, which is the action that is second most likely to be optimal under the posterior. Put differently, the algorithm sets $\psi_{n, \hat{I}_n}=\beta$, and $\psi_{n, \hat{J}_n} = 1-\beta$.

\subsection{Top-Two Value Sampling (TTVS)}
We now propose a variant of top-two sampling that considers not only the probability a design is optimal, but the expected amount by which its quality exceeds that of other designs. In particular, we will define below a measure $V_{n,i}$ of the value of design $i$ under the posterior distribution at time $n$. Top-two value sampling computes the top-two designs under this measure: $\hat{I}_n = \arg\max_{i} V_{n,i}$ and $\hat{J}_n = \arg\max_{j \neq \hat{I}_n} V_{n,j}$. It then plays the top design $\hat{I}_n$ with probability $\beta$ and the best alternative $\hat{J}_n$ otherwise. As observations are gathered, beliefs are updated and so the top two designs change over time. The measure of value $V_{n,i}$ is defined below.

The definition of TTVS depends on a choice of (utility) function $u: \theta \mapsto \mathbb{R}$, which encodes a measure of the value of discovering a design with quality $\theta_i$. Two natural choices of $u$ are $u(\theta)=\theta$ and $u(\theta) = \mu(\theta)$.  The paper's theoretical results allow $u$ to be a general function, but we assume that it is \emph{continuous} and \emph{strictly increasing}.
For a given choice of $u$, and any $i \in \{1,...,k\}$, the function
\[
v_{i}(\thetabf) = \max_{j} u(\theta_j) - \max_{j\neq i} u(\theta_j) = \left\{\begin{array}{lr}
0 & \text{if } \thetabf \notin \Theta_i\\
u(\theta_i) - \max_{j\neq i}u(\theta_j) & \text{if } \thetabf \in \Theta_i
\end{array}\right.
\]
provides a measure of the value of design $i$ when the true parameter is $\thetabf$. It captures the improvement in decision quality due to design $i$'s inclusion in the choice set. Let
\begin{equation}\label{eq: def of V}
V_{n,i} = \intop_{\Theta} v_{i}(\thetabf) \pi_n(\thetabf) d\thetabf = \intop_{\Theta_i} v_{i}(\thetabf) \pi_n(\thetabf) d\thetabf
\end{equation}
denote the expected value of $v_{i}(\thetabf)$ under the posterior distribution at time $n$. This can be viewed as the option-value of design $i$: it is the expected additional value of having the option to choose design $i$ when it is revealed to be the best design. Note that the integral \eqref{eq: def of V} defining $V_{n,i}$ is a weighted version of the integral defining $\alpha_{n,i}$. The paper will formalize a sense in which $V_{n,i}$ and $\alpha_{n,i}$ are asymptotically equivalent as $n\to \infty$, and as a result the asymptotic analysis of top-two value sampling essentially reduces to the analysis of top-two probability sampling.

\subsection{Thompson Sampling}\label{subsec: TS}
Thompson sampling is an old and popular heuristic for multi-armed problems. The algorithm simply samples actions according to the posterior probability they are optimal. In particular, it selects action $i$ with probability $\psi_{n,i}= \alpha_{n,i}$, where $\alpha_{n,i}$ denotes the probability action $i$ is optimal under under a parameter drawn from the posterior distribution.

Thompson sampling can have very poor asymptotic performance for the best arm identification problem. Intuitively, this is because once it estimates that a particular arm is the best with reasonably high probability, it selects that arm in almost all periods at the expense of refining its knowledge of other arms. If $\alpha_{n,i} =.95$, then the algorithm will only select an action other than $i$ roughly once every 20 periods, greatly extending the time it takes until $\alpha_{n,i}> .99$. This insight can be made formal; our results imply that Thompson sampling attains a only attains a polynomial, rather exponential, rate of posterior convergence.  A similar reasoning applies to other multi-armed bandit algorithms.  The work of \citet{bubeck2009pure} shows formally that algorithms satisfying regret bounds of order $\log(n)$ are necessarily far from optimal for the problem of identifying the best arm. 

With this in mind, it is natural to consider a modification of Thompson sampling that simply restricts the algorithm from sampling the same action too frequently. One version of this idea is proposed below.

\subsection{Top-Two Thompson Sampling (TTTS)}
This section proposes top-two Thompson sampling (TTTS), which modifies standard Thompson sampling by adding a re-sampling step. As with TTPS and TTVS, this algorithm depends on a tuning parameter $\beta>0$ that will sometimes be set to a default value of $1/2$.

As in Thompson sampling, at time $n$, the algorithm samples a design $I \sim {\bm \alpha_{n}}$. Design $I$ is measured with probability $\beta$, but, in order to prevent the algorithm from exclusively focusing on one action, with probability $1-\beta$, an alternative design is measured. To generate this alternative, the algorithm continues sampling designs $J \sim {\bm \alpha_{n}}$ until the first time $J\neq I$. This can be viewed as a top-two sampling algorithm, where the top-two are chosen by executing Thompson sampling until two distinct designs are drawn. 

Under top-two Thompson sampling, the probability of measuring design $i$ at time $n$ is 
\[
\psi_{n,i} = \alpha_{n,i}\left(\beta + (1-\beta) \sum_{j\neq i} \frac{\alpha_{n,j}}{1-\alpha_{n,j}}\right).
\]

This expression simplifies as the algorithm definitively identifies the best design. As $\alpha_{n,I^{*}}\to 1$, $\psi_{n,I*}\to \beta$, and for each $i\neq I^*$, 
\[ 
\frac{\psi_{n,i}}{1-\psi_{n,I^*}} \sim \frac{\alpha_{n,i}}{1-\alpha_{n,I^*}}.
\] 
In this limit, the true best design is sampled with probability $\beta$.  The probability $i$ is sampled given $I^*$ is not is equal to the posterior probability $i$ is optimal given $I^*$ is not. %Roughly, as the algorithm becomes confident a particular design is optimal, the set of alternative designs is sampled $1-\beta$ fraction of the time, and individual de according to their relative probability of being optimal.  

\subsection{Computing and Sampling According to Optimal Action Probabilities}
Here we provide some insight into how to efficiently implement the proposed top-two rules in important problem classes. We begin with top-two Thompson sampling, which is often the easiest to implement. Note that given an ability to sample from $\Pi_{n}$, it is easy to sample from the posterior distribution over the optimal design $\alphabf_n$. In particular, if $\hat{\thetabf}\sim \Pi_{n}$ is drawn randomly from the posterior, then $\arg\max_{i} \hat{\theta}_i$ is a random sample from ${\bm \alpha}_n$. Either through the choice of conjugate prior distributions, or through the use of Markov chain Monte Carlo, it is possible to efficiently sample from the posterior for many interesting models.  Algorithm \ref{alg: TTTS}  shows how to directly sample an action according to TTTS by sampling from the posterior distribution. It is worth highlighting that this algorithm does not require computing or approximating the distribution ${\bm \alpha}_n$.

\begin{algorithm}
	\caption{Top-Two Thompson Sampling $(\beta)$}\label{alg: TTTS}
	\begin{algorithmic}[1]
		%\Procedure{OP}{$a,b$}\Comment{The g.c.d. of a and b}
		\State{Sample $\hat{\thetabf} \sim \Pi_n$ and set $I \leftarrow \arg\max_{i} \hat{\theta}_{i}$} \Comment{Apply Thompson sampling}
		\State Sample $B \sim {\rm Bernoulli}(\beta)$
		\If{$B =1$ } \Comment{Occurs with probability $\beta$. }
		\State{Play $I$}
		\Else
		\Repeat
		\State{Sample $\hat{\thetabf} \sim \Pi_n$ and set $J \leftarrow \arg\max_j \hat{\theta}_j$} \Comment{Repeat Thompson sampling}
		\Until{$J\neq I$}
		\State{Play $J$}
		\EndIf
	\end{algorithmic}
\end{algorithm}

The optimal action probabilities $\alpha_{n,i}$ and values $V_{n,i}$ are defined by $k$-dimensional integrals, which may be difficult to compute in general even if the posterior $\Pi_n$ has a closed form. Algorithm \ref{alg: SampleApprox} shows how to approximate $\alpha_{n,i}$ and $V_{n,i}$ using samples $\thetabf^{1}\ldots \thetabf^M$, which enables efficient approximations to TTPS and TTVS whenever posterior samples can be efficiently generated. 

\begin{algorithm}[H]
	\caption{$\text{SampleApprox}(K, M, u, \thetabf^1,\ldots,\thetabf^M)$}\label{alg: SampleApprox}
	\begin{algorithmic}[1]
		\State $\mathcal{S}_{i}\leftarrow \{m | i=\arg\max_j \theta^m_j \} \qquad \forall i\in\{1,..,K\}$
		\State $\hat{\alpha}_i \leftarrow |\mathcal{S}_i|/m \qquad \forall i\in \{1,..K\}$
		\State $\hat{V}_i \leftarrow M^{-1}\sum_{m\in \mathcal{S}_i} \left( u(\theta^m_i)-\max_{j\neq i}u(\theta^m_j)\right) \qquad \forall i\in\{1,..,K\}$\\
		\Return $\hat{\bm \alpha}, \hat{\bm V}$
	\end{algorithmic}
\end{algorithm}

Thankfully, the computation of $\alpha_{n,i}$ and $V_{n,i}$ simplifies when the algorithm begins with an independent prior over the qualities $\theta_1,...\theta_k$ of the $k$ designs. To understand this fact, suppose $X_1,...,X_k\in \mathbb{R}$ are independently distributed and continuous random variables. Then 
\begin{equation}\label{eq: optimal probability integral}
\Prob(X_1 = \max_{i} X_i) = \intop_{x\in \mathbb{R}} f_1(x) \prod_{j=2}^{k} F_{j}(x) dx
\end{equation}
where $f_1$ denotes the PDF of $X_1$ and $F_2,...,F_K$ are the CDFs of $X_2,..,X_k$. In particular, $\Prob(X_1 = \max_{i} X_i)$  can be computed by solving a 1-dimensional integral. Based on this insight, Appendix \ref{sec: implementation of TTVS} provides an efficient implementation of TTPS and TTVS for a problem with independent Beta priors and binary observations. That implementation approximates integrals like  \eqref{eq: optimal probability integral} using quadrature with $n$ points, and has the time and space complexity that scale as $O(kn)$.

%\subsection{Extensions}
%
%\subsubsection{Extensions to an Indifference Zone Formulation}
% Consider a problem with an indifference zone (IZ) where the goal is to confidently identify an $\epsilon$--optimal design, for a pre-specified tolerance level $\epsilon>0$. The proposed algorithms can easily be tailored to this setting.
%
% Let $\Theta_{\epsilon, -i} = \{ \thetabf \in \Theta : \theta_i+\epsilon < \max_{j \neq i} \theta_j\}$ denote the set of parameters under which the quality of design action $i$ exceeds that of all others by at least $\epsilon$. In the IZ-formulation, one would like to gather evidence to drive $\Pi_{n}\left( \Theta_{\epsilon, -I^*} \right)$ to zero as quickly as possible.
%
%
% One could modify TTPS so that, with probability $\beta$, sample $\hat{I} = \arg\max_{i} \alpha_{n,i}$ as before, but with probability $1-\beta$ it selects $\arg\max_{j\neq \hat{I}} \Pi_{n}(\Theta_{j, \epsilon})$.
%
%
%\subsubsection{Best-M Arm Identification}
%Sample $\hat{\thetabf} \sim \Pi_n$ and set $S_1$ to be the top $m$ designs under $\hat{\thetabf}$. Then continue sampling $\hat{\thetabf}\sim \Pi_n$ and computing the best $m$ designs $S_2$ under $\hat{\thetabf}$. Then choose a design to measure by sampling uniformly at random from the symmetric difference between these sets $S_1 \Delta S_2 = (S_1 \setminus S_2)\cup (S_2 \setminus S_1)$. In this problem, each design is classified as either one of the top $m$ designs, or not. The This method tries to identify the top $m$ designs by sampling those that are likely to be mis-classified.

\section{A Numerical Experiment}
Some of the paper's main insights are reflected in a simple numerical experiment. Consider a problem where observations are binary $Y_{n,i}\in \{0,1\}$, and the unknown vector
$
\thetabf^* = (.1, .2, .3, .4, .5)
$
defines the true success probability of each design. Each algorithm begins with an independent uniform prior over the components of $\thetabf^*$. The experiment compares the performance of top-two probability sampling (TTPS), top-two value sampling (TTVS)\footnote{TTVS is executed with the utility function $u(\theta)=\theta$}, and top-two Thompson sampling (TTTS) with $\beta=1/2$ against Thompson sampling and a uniform allocation rule which allocates equal measurement effort ($\psi_{n,i}=1/5)$ to each design. The uniform allocation is a natural point of comparison, since is the most commonly used strategy in practice.

Figure \ref{fig: measurements required} displays the average number of measurements required for the posterior to reach a given confidence level. In particular, the experiment tracks the first time when $\max_{i} \alpha_{n,i} \geq c$ for various confidence levels $c\in (0,1)$. Figure \ref{fig: measurements required} displays the average number of measurements required for each algorithm to reach each fixed confidence level, where the average was taken over 100 trials in Panel (a) and 500 in Panel (b). Even for this simple problem with five designs, the proposed algorithms can reach the same confidence level using fewer than half the measurements required by a uniform allocation rule. While all the top-two rules attain the same asymptotic rate of convergence, we can see that top-two probability sampling is slightly outperformed in this experiment. Panel (a) compares Thompson sampling to Top-Two Thompson sampling. TS appears to reach low confidence levels as rapidly as top-two TS, but as suggested in Subsection \ref{subsec: TS}, is very slow to reach high levels of confidence. It requires over than 60\% more measurements to reach confidence .95 and over 250\% more measurements to reach confidence .99. TS requires an onerous number of measurements to reach confidence .999, and so we omit this experiment. 

%
%\begin{figure}[H]
%	\centering
%	\includegraphics[width=5in]{total_measurements_TSvsTTTS.pdf}
%	\caption{Number of measurements required to reach given confidence level.}\label{fig: measurements required}
%\end{figure}

\begin{figure}[h!]
	\centering
	\begin{subfigure}{.5\textwidth}
		\includegraphics[width=1\linewidth]{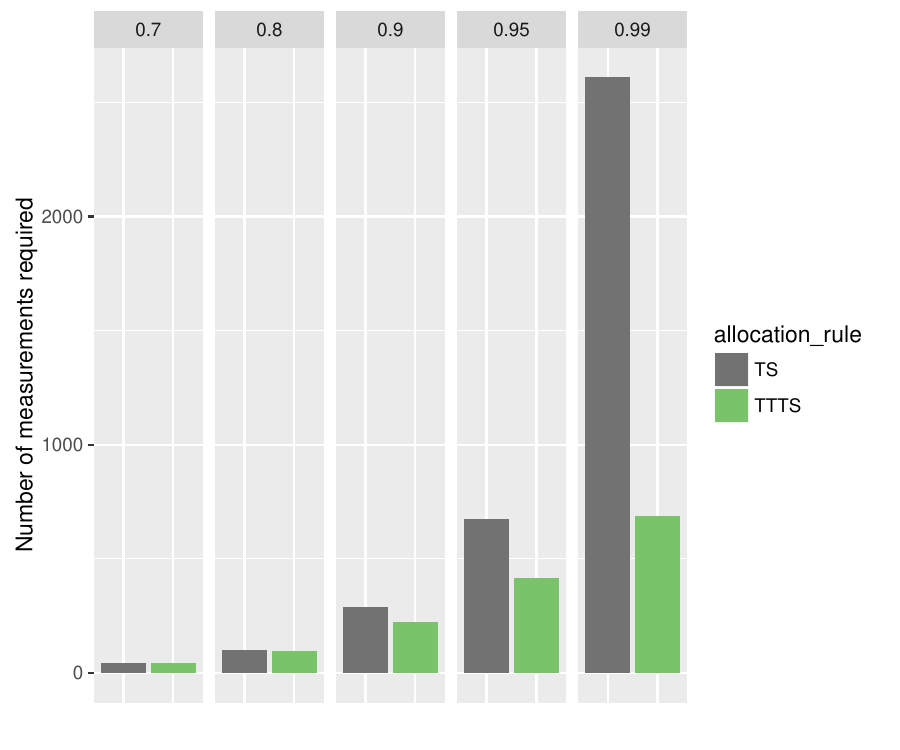}
		\caption{TS vs Top-Two TS.}
		\label{fig:sub1}
	\end{subfigure}%
	\begin{subfigure}{.5\textwidth}
		\includegraphics[width=1\linewidth]{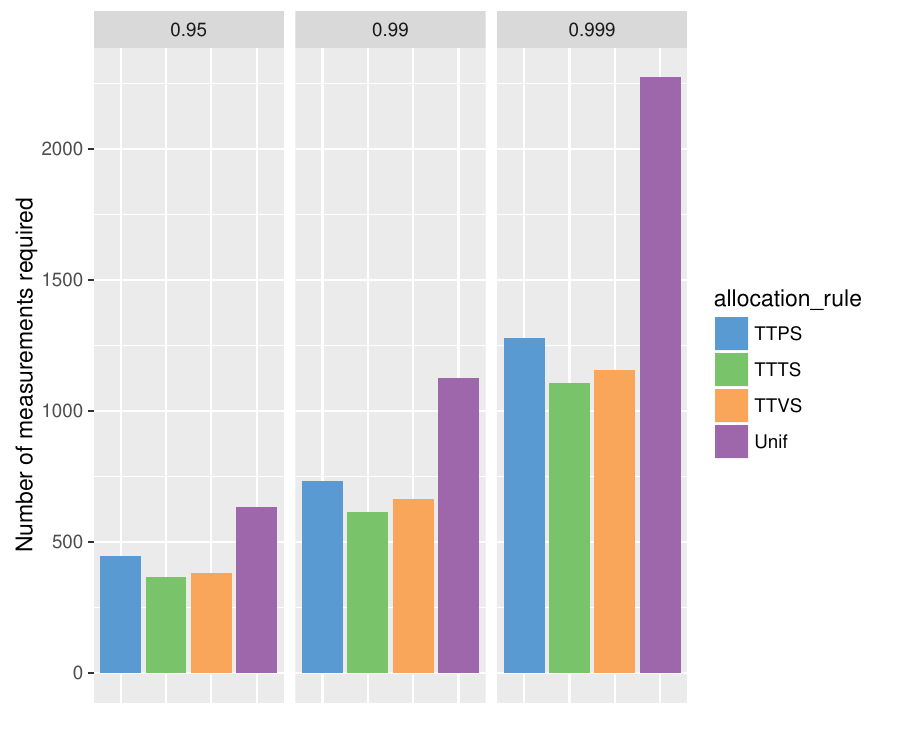}
		\caption{Comparison with uniform allocation.}
		\label{fig:sub2}
	\end{subfigure}
	\caption{Number of measurements required to reach given confidence level.}
	\label{fig: measurements required}
\end{figure}

Figure \ref{fig: distribution at termination} provides insight into how the proposed algorithms differ from the uniform allocation. It displays the distribution of measurements and posterior beliefs at the first time when a confidence level of .999 is reached. Again, all results are averaged across 500 trials. Panel (a) displays the average number of measurements collected from each design. It is striking that although TTTS, TTPS, and TTVS seem quite different, they all settle on essentially the same distribution of measurement effort. Because $\beta= 1/2$, roughly one half of the measurements are collected from $I^*=5$. Moreover, fewer measurements are collected from designs that are farther from optimal, and most of the remaining half of measurement effort is allocated to design 4. Notice that using the same number of  noisy samples it is much more difficult certify that $\theta^*_4 < \theta^*_5$ than that $\theta^*_{1}< \theta^*_5$, both because $\theta^*_4$ is closer to $\theta^*_5$, and because observations from a Bernoulli distribution with parameter .4 have higher variance than under a Bernoulli distribution with parameter .1.

Panel (b) investigates the posterior probability $\alpha_{n,i}$ assigned to the event that design $i$ is optimal. To make the insights more transparent, these are plotted on log-scale, where the value $\log(1/\alpha_{n,i})$ can roughly be interpreted as the magnitude of evidence that alternative $i$ is not optimal. By using an \emph{equal allocation} of measurement effort across the designs,  the uniform sampling rule gathers an enormous amount of evidence to rule out design 1, but an order of magnitude less evidence to rule out design 4. Instead of allocating measurement effort equally across the alternatives, TTTS, TTPS, and TTVS appear to exactly adjust measurement effort to gather \emph{equal evidence} that each of the first four designs is not optimal.

Intuitively, in the long run each of the proposed algorithms will allocate measurement effort to design 5--the true best design--and to whichever other designs could most plausibly be optimal. If too much measurement effort has been allocated to a particular design, then the posterior will indicate that it is clearly suboptimal, and effort will be allocated elsewhere until a similar amount of evidence has been gathered about other designs. In this way, measurement effort is automatically adjusted to the appropriate level.

%The analysis in the later section will formally prove many of the properties shown in this experiment. It will be shown that, in a precise sense, the optimal allocation gathers equal evidence to rule out each suboptimal action.

%
%\begin{figure}[H]
%	\centering
%	\includegraphics[width=5in]{averageMeasurementsRequired.pdf}
%	\caption{Number of measurements required to reach given confidence level.}\label{fig: measurements required}
%\end{figure}

\begin{figure}[h!]
	\centering
	\begin{subfigure}{.5\textwidth}
		\includegraphics[width=1\linewidth]{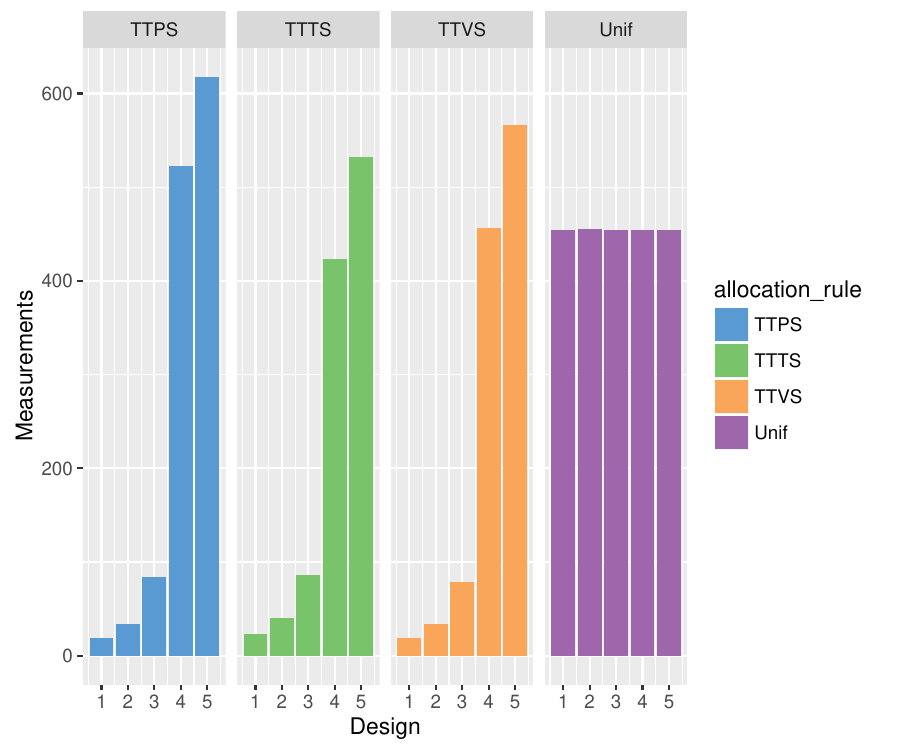}
		\caption{Measurements collected of each design.}
		\label{fig:sub1}
	\end{subfigure}%
	\begin{subfigure}{.5\textwidth}
		\includegraphics[width=1\linewidth]{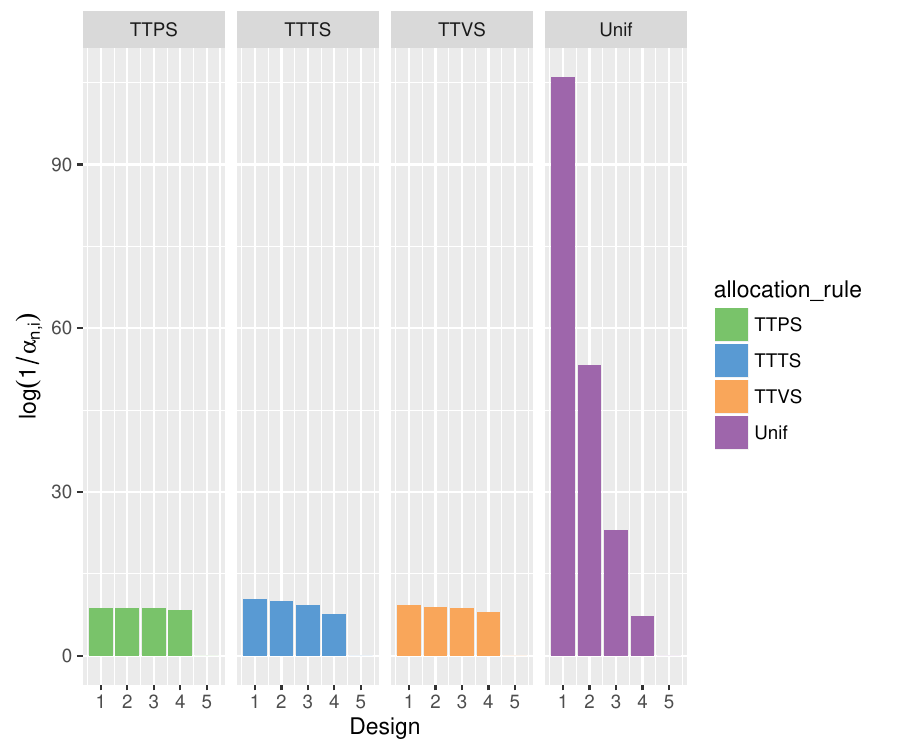}
		\caption{Value of $\log(1/\alpha_{n,i})$ for each design $i$.}
		\label{fig:sub2}
	\end{subfigure}
	\caption{Distribution of measurements and posterior beliefs at termination.}
	\label{fig: distribution at termination}
\end{figure}

\section{Main Theoretical Results}
Our main theoretical results concern the frequentist consistency and rate of convergence of the posterior distribution.
Recall that
\[
\Pi_{n}(\Theta_{I^*}^{c}) = \sum_{i\neq I^*} \alpha_{n,i}
\]
captures the posterior mass assigned to the event that an action other than $I^*$ is optimal. One hopes that $
\Pi_{n}(\Theta_{I^*}^{c}) \to 0 $ as the number of observations $n$ tends to infinity, so that the posterior distribution converges on the truth. We will show that under the TTTS, TTPS, and TTVS allocation rules, $ \Pi_{n}(\Theta_{I^*}^{c})$ converges to zero an exponential rate
 and that the exponent governing the rate of convergence is nearly the best possible.

To facilitate theoretical analysis, we will make three additional boundedness assumptions, which are assumed throughout all formal proofs. This rules out some cases of interest, such the use of multivariate Gaussian prior. However, we otherwise allow for quite general correlated priors, expressed in terms of a density over a compact set. This stands in contrast, for example, to previous analyses of Thompson sampling, which tpyically rely heavily on the use of independent conjugate priors. Assumption \ref{assumption: main} is used only in establishing certain asymptotic results concerning the rate of posterior concentration. Analogous results are easily established for certain unbounded conjugate priors\footnote{See for example \cite{qin2017improving}, which is a follow up to the current paper.}, but the author still has not identified the right technical conditions that generalize these results.   
\begin{assumption}\label{assumption: main}The parameter space is a bounded open hyper-rectangle $\Theta = (\underline{\theta}, \overline{\theta})^k$, the prior density is uniformly bounded with
	\[
	0< \inf_{\thetabf \in \Theta} \pi_{1}(\thetabf) < \sup_{\thetabf \in \Theta} \pi_{1}(\thetabf)< \infty,
	\]
	and the log-partition function has bounded first derivative with
	$
	\sup_{\theta \in [\underline{\theta}, \overline{\theta}]} | A'(\theta)| < \infty. $
\end{assumption}

%The exponent $\Gamma^*$ will be explicitly characterized in Section \ref{sec: analysis}

The paper's main results, as stated in the next theorem, characterize the rate of posterior convergence under the proposed algorithms, formalize a sense in which this is the fastest possible rate, and bound the impact of the tuning parameter $\beta\in (0,1)$. The statement depends on distribution-dependent constants $\Gamma^{*}_{\beta}>0$ and $\Gamma^*>0$ that are presented here but will be more explicitly characterized in Section \ref{sec: analysis}.

The first part of the theorem shows that there is an exponent $\Gamma^*>0$ such that $\Pi_{n}(\Theta_{I^*}^{c})$ cannot converge to zero at a rate faster than $e^{-n\Gamma^*}$ under any allocation rule, and shows that TTPS, TTVS and TTTS attain this optimal rate of convergence when the tuning parameter $\beta$ is set optimally. This optimal exponent is shown to equal 
\[
\Gamma^* = \max_{\psibf} \min_{\thetabf \in \Theta_{I^*}^{c}} \sum_{i=1}^{k}\psi_i d(\theta^*_i || \theta_i), 
\]
where $d(\theta_i || \theta'_i )$ denotes the Kullback Leibler divergence between the observations distributions $p(y| \theta_i)$ and $p(y| \theta'_i)$. This can be viewed as the value of a game between two players. An experimenter first chooses a probability distribution over arms $\psibf$ determining the frequency with which arms are measured. An adversary then chooses the worst case configuration of arm means, selecting an alternative $\thetabf=(\theta_1,\ldots, \theta_k)$ that is hard to distinguish from $\thetabf^*$ under the measurement allocation $\psibf$, but under which the arm $I^*$ is no longer optimal. 

The remainder of the theorem investigates the role of the tuning parameter $\beta \in (0,1)$. Part 2 shows that there is an exponent $\Gamma^*_{\beta}>0$ such that $\Pi_{n}(\Theta_{I^*}^{c})\to 0$ at rate $e^{-n\Gamma^*_{\beta}}$ under TTPS, TTVS, or TTTS with parameter $\beta$, and this is shown to be optimal among a restricted class of allocation rules. In particular,
 we observe that  $\beta$ controls the fraction of measurement effort allocated to the true best design $I^*$, in the sense that $\overline{\psi}_{n,I^*} \to  \beta$ as $n \to \infty$ under each of the proposed algorithms. These algorithms attain the error exponent
 \[
 \Gamma^*_{\beta} = \max_{\psibf: \psi_{I^*}=\beta} \min_{\thetabf \in \Theta_{I^*}^{c}} \sum_{i=1}^{k}\psi_i d(\theta^*_i || \theta_i), 
 \]
which is analogous to $\Gamma^*$ except that the experimenter is constrained to measure the true best arm with fraction $\beta$ of measurement effort. A lower bound shows this exponent is optimal among a constrained class: precisely, on any sample path on which an adaptive algorithm allocates a faction $\beta$ of overall effort to measuring $I^*$, the posterior cannot converge at rate faster than $e^{-n\Gamma^*_{\beta}}$. In this sense, while a tuning parameter controls the long-run measurement effort allocated to the true best design, TTPS, TTVS, and TTTS all automatically adjust how remaining the measurement effort is allocated among the $k-1$ suboptimal designs in an asymptotically optimal manner. 

The final part of the theorem shows that the constrained exponent $\Gamma_{\beta}^*$ is close to the largest possible exponent $\Gamma^*$ whenever $\beta$ is close to the optimal value. The choice of $\beta=1/2$ is particularly robust: $\Gamma^*_{1/2}$ is never more than a factor of 2 away from the optimal exponent.

\begin{Theorem}\label{thm: main thm}
	There exist constants $\{\Gamma^*_{\beta} >0 : \beta \in (0,1) \}$ such that $\Gamma^* = \max_{\beta} \Gamma^*_{\beta}$ exists, $\beta^* = \arg\max_{\beta} \Gamma^*_{\beta}$ is unique, and the following properties are satisfied with probability 1:
\end{Theorem}
\begin{enumerate}
	\item Under TTTS, TTPS, or TTVS with parameter $\beta^*$,
	\[
	\lim_{n\to \infty} \,-\frac{1}{n} \log \Pi_{n}(\Theta_{I^*}^{c})
	= \Gamma^*.
	\]
	Under any adaptive allocation rule,
	\[
	\underset{n\to \infty}{\lim\sup} \, - \frac{1}{n} \log \Pi_{n}(\Theta_{I^*}^{c})
	\leq \Gamma^*. \]
	\item Under TTTS, TTPS, or TTVS with parameter $\beta \in (0,1)$,
	\[
	\lim_{n\to \infty} \,-\frac{1}{n} \log \Pi_{n}(\Theta_{I^*}^{c})
	= \Gamma^*_{\beta} \quad \text{and} \quad \lim_{n\to\infty} \overline{\psibf}_{n, I^*} = \beta. \]
	Under any adaptive allocation rule,
	\[
	\underset{n\to \infty}{\lim\sup} \, - \frac{1}{n} \log \Pi_{n}(\Theta_{I^*}^{c})
	\leq \Gamma^*_{\beta} \quad \text{on any sample path with} \quad
	\lim_{n\to \infty} \overline{\psibf}_{n, I^*} = \beta. \]
	
	\item $\Gamma^*\leq 2 \Gamma^*_{\frac{1}{2}} $ and
	\[
	\frac{\Gamma^*}{\Gamma^*_{\beta}}  \leq \max\left\{\frac{\beta^*}{\beta}, \frac{1-\beta^*}{1-\beta} \right\}.
	\]
\end{enumerate}
This theorem is established in a sequence of results in Section \ref{sec: analysis}. The lower bounds in parts 1 and 2 are given respectively in Propositions \ref{prop: optimal allocation} and \ref{prop: optimal constrained allocation}. Proposition \ref{prop: TS converges to optimal allocation} shows the top-two rules attain these optimal exponents. Part 3 is stated as Lemma \ref{lem: relating the constrained exponent to the unconstrained} in Section \ref{sec: analysis}.

\subsection{An upper bound on the error exponent}
Before proceeding, we will state an upper bound on the error exponent when $\beta=1/2$ that is closely related to complexity terms that have appeared in the literature on best--arm identification (e.g. \cite{audibert2010best}). This bound depends on the gaps between the means of the different observation distributions.

We say that a real valued random variable $X$ is $\sigma$--sub--Gaussian if
$ \E \left[\exp\{\lambda (X- \E[X])\}\right] \leq  \exp\left\{ \frac{\lambda^2 \sigma^2}{2 }  \right\}$
so that the moment generating function of $X-\E[X]$ is dominated by that of a zero mean Gaussian random variable with variance $\sigma^2$. Gaussian random variables are sub-Gaussian, as are uniformly bounded random variables. The next result applies to both Bernoulli and Gaussian distributions, as each can be parameterized with sufficient statistic $T(y)=y$.
\begin{proposition}\label{prop: subgaussian bound}
	Suppose the exponential family distribution is parameterized with $T(y)=y$ and that each $\theta \in [\underline{\theta}, \overline{\theta}]$, if $Y\sim p(y|\theta)$, then $Y$ is sub-Gaussian with parameter $\sigma.$
	Then
	\[
	\Gamma^*_{\frac{1}{2}} \geq \frac{1}{16 \sigma^2 \sum_{i\neq I^*} \Delta_i^{-2}}
	\]
	where for each $i\in \{1,...,k\}$,
	\[
	\Delta_i = \E[ Y_{n,I^*}] - \E [ Y_{n,i}]
	\]
	is the difference between the mean under $\theta_{I^*}^*$ and the mean under $\theta_i^*$.
\end{proposition}
This shows that
$\Pi_{n}(\Theta_{I^*}^{c})$ decays at asymptotic rate faster than $\exp\{-\frac{ n \min_{i} \Delta_i^2}{ 16 k \sigma^2}\},$ so convergence is rapid when there is a large gap between the means of different designs. In fact, Proposition \ref{prop: subgaussian bound} replaces the dependence on $(1/k)$ times the smallest gap $\Delta_i$ with a dependence on $
\left( \sum_{i=2}^{k} \Delta_i^{-2} \right)^{-1}
$, which captures the average inverse gap.
This rate is attained only by an intelligent adaptive algorithm which allocates more measurement effort to designs that are nearly optimal and less to designs that are clearly suboptimal. In fact, the next result shows that the asymptotic performance of uniform allocation rule  depends only on the smallest gap $\min_{i\neq I^*} \Delta_i^2,$ and therefore even if some designs could be quickly ruled out, the algorithm can't leverage this to attain a faster rate of convergence.
\begin{proposition}\label{prop: convergence rate of uniform allocation}
	If $Y_{n,I^*} \sim \mathcal{N}(0,\sigma^2)$ and $Y_{n,i}\sim \mathcal{N}(- \Delta_i, \sigma^2)$ for each $i\neq I^*$,
	\[
	\lim_{n\to \infty} \,-\frac{1}{n} \log \Pi_{n}(\Theta_{I^*}^{c}) =  \frac{-n \min_{i} \Delta_i^2}{4k\sigma^2} 
	\]
	under a uniform allocation rule which sets $\psi_{n,i} = 1/k$ for each $i$ and $n$.
\end{proposition}

\subsection{Consistent Tuning of $\beta$}
Our previous results show that if the top-two sampling algorithms are applied with the optimal problem dependent tuning parameter $\beta^* = \arg\max_{\beta} \Gamma^*$, then these algorithms attain the optimal rate of posterior convergence $e^{-\Gamma^* n}$. Unfortunately $\beta^*$ is is typically unknown, and so we  also  investigate robustness to the choice of $\beta$, both in theory as in Theorem \ref{thm: main thm} above, and in simulation experiments presented in Section \ref{sec: further experiments}. Still, a natural question is whether this tuning parameter can be adjusted in a dynamic fashion to converge on $\beta^*$.  We begin the study of such extensions in this subsection. 

First, note that it is easy to extend the definition of each top to sampling algorithm so that they use an adaptive sequence of tuning parameters $(\beta_n : n\in \mathbb{N})$. For example, top-two probability sampling identifies $\hat{I}_{n}=\arg\max_{i} \alpha_{n,i}$ and $\hat{J}_n = \arg\max_{j\neq \hat{I}_n} \alpha_{n,j}$ and then chooses among these with respective probabilities $\psi_{n, \hat{I}_n} = \beta_n$ and  $\psi_{n, \hat{J}_n} = 1- \beta_n$. The next lemma confirms that, if applied with such a sequence of tuning parameters such that $\beta_{n} \to \beta^*$, the top-two sampling algorithms attain the optimal convergence rate $e^{-n \Gamma_{\beta^*}}$.

\begin{proposition}\label{prop: optimal rate with adaptive tuning}
	Suppose TTTS, TTVS, TTPS are applied with an adaptive sequence of tuning parameters $(\beta_{n} : n \in \mathbb{N})$ where for each $n$, $\beta_{n}$ is $\hist$ measurable. Then, with probability 1, on any sample path on which $\beta_n \to \beta^*$, 
	\[
	\Pi_{n}(\Theta_{I^*}^{c}) \doteq e^{-n \Gamma^*}.
	\]
\end{proposition}

The next lemma confirms that such consistent tuning is possible. The method for tuning $\beta$, presented Algorithm \ref{alg: TTTS tuned}, simply solves numerically for the optimal value of $\beta$ assuming that the true values of the parameters $(\theta_1, \ldots \theta_k)$ are given by their respective posterior means. 

Unfortunately, this tuning method is complex, spoiling some of elegance of the top-two sampling algorithms. A significant open question is whether simpler methods for adapting $\beta$ could be adopted. 

\begin{lemma}\label{lem : consistency of adaptive tuning}
	Under TTTS, TTPS, or TTVS with an adaptive sequence of tuning parameters $(\beta_{n} : n \in \mathbb{N})$ adjusted according to Algorithm \ref{alg: TTTS tuned}, $\beta_{n}\to \beta^*$ almost surely. Therefore $\Pi_{n}(\Theta_{I^*}^{c}) \doteq e^{-n \Gamma^*}$.
\end{lemma}

\begin{algorithm}
	\caption{Top-Two Sampling with  $\beta$-Tuning}\label{alg: TTTS tuned}
	\begin{algorithmic}[1]
		%\Procedure{OP}{$a,b$}\Comment{The g.c.d. of a and b}
		\State{Input $\kappa \geq 2$, $\hat{\beta} \in (0,1)$.}
		\State{Set counter $\ell=1$}
		\For{$n \in  \{1,2,3,4,\ldots \}  $}
		\State{Sample $I_{n} \sim {\rm TopTwo}(\pi_{n}, \hat{\beta}) $} 
		\State{Measure $I_n$ and observe $Y_{n, I_n}$}
		\State{Update play-count $S_{n+1, I_n} \leftarrow S_{n, I_n}+1 $}
		\State{Update posterior $\pi_{n+1}(\thetabf)  \propto \pi_{n}(\thetabf)p\left( Y_{n, I_n} \mid \theta_{I_n} \right)$ }
		\If{ $\min_{i}S_{n,i} \geq \kappa^{\ell} $} %\Comment{Occurs with probability $\beta$. }
		\State{ $\ell \leftarrow \ell+1 $ } 
		\State{Compute Posterior mean $\hat{\thetabf}\leftarrow \intop_{\Theta} \thetabf  \pi_{n+1}(\thetabf)d\thetabf $ }
		\If{$\arg\max_{i} \hat{\theta}_{i}$ is unique}
		\State{Estimate best arm $\hat{I}\leftarrow \arg\max_{i} \hat{\theta}_{i} $ }
		\State{Estimate best allocation $\hat{\psibf} \leftarrow \arg\max_{\psibf} \min_{\thetabf \in \Theta_{\hat{I}}^{c}} D_{\psi}(\hat{\thetabf} || \thetabf)$ }
		\State{$\hat{\beta} \leftarrow \hat{\psi}_{\hat{I}} $}
		\EndIf 
		\EndIf
		\EndFor 
	\end{algorithmic}
\end{algorithm}

\section{Analysis}\label{sec: analysis}

\subsection{Asymptotic Notation.}
To simplify the presentation, it is helpful to introduce additional asymptotic notation. We say two sequences $a_n$ and $b_n$ taking values in $\mathbb{R}$ are \emph{logarithmically equivalent}, denoted by $a_n \doteq b_n$,  if $\frac{1}{n}\log(\frac{a_n}{b_n}) \to 0$  as $n\to \infty$. This notation means that $a_n$ and $b_n$ are \emph{equal up to first order in the exponent}. With this notation, Theorem \ref{thm: main thm} implies the top-two sampling rules with parameter $\beta$ attain the convergence rate $\Pi_{n}(\Theta^c_{I^*})\doteq e^{-n\Gamma^*_{\beta}}$. 
This is an equivalence relation, in the sense that if $a_n \doteq b_n$ and $b_n \doteq c_n$ then $a_n \doteq c_n$. Note that $
a_n + b_n \doteq \max\{a_n, b_n\}$,
so that the sequence with the largest exponent dominates. In addition for any positive constant $c$,
$c a_n \doteq a_n$, so that constant multiples of sequences are equal up to first order in the exponent. When applied to sequences of random variables, these relations are understood to apply almost surely. 

It is natural to wonder whether the proposed algorithms asymptotically minimize expressions like $\sum_{i\neq I^*} (\theta^*_{I^*}- \theta_i) \alpha_{n,i}$, which account for how far some designs are from optimal. We note in passing, that
\[ 
\sum_{i\neq I^*} c_i \alpha_{n,i} \doteq \max_{i\neq I^*} \alpha_{n,i}
\]
for any positive costs $c_{i}>0$, and so any such performance measures are equal to first order in the exponent. Similar observations have been used to justify the study of the probability of incorrect selection, rather than notions of the expected cost of an incorrect decision \citep{glynn2004large, audibert2010best}. 

\subsection{Posterior Consistency}
The next proposition provides a consistency and anti-consistency result for the posterior distribution. The first part says that if design $i$ receives infinite measurement effort, the marginal posterior distribution of its quality concentrates around the true value $\theta^*_i$. The second part says that when restricted to designs that are not measured infinitely often, the posterior does not concentrate around any value. The posterior converges to the truth as infinite evidence is collected, but nothing can be ruled out with certainty based on finite evidence.

\begin{proposition}\label{prop: posterior consistency}
	With probability 1, for any $i \in \{1,..,k\}$ if $\Psi_{n,i} \to \infty$, then, for all $\epsilon>0$
	\[
	\Pi_{n}(\{ \thetabf \in \Theta | \theta_i \notin (\theta^*_{i}-\epsilon, \theta^*_{i}+\epsilon ) \} ) \to 0.
	\]
	If $\mathcal{I} = \{i \in \{1,...,k\} | \lim_{n\to \infty} \Psi_{n,i}<\infty \}$ is nonempty, then
	\[
	\inf_{n\in\mathbb{N}} \, \Pi_{n}(\{ \thetabf \in \Theta | \theta_i \in (\theta_{i}', \theta_{i}'') \,\, \forall i \in \mathcal{I} \} ) > 0
	\]
	for any collections of open intervals $
	(\theta_i', \theta_i'') \subset (\underline{\theta}, \overline{\theta})$ ranging over $i \in \mathcal{I}$.
	\end{proposition}
 This result is the key to establishing that $\alpha_{n,I^*}\to 1$ under each of the proposed algorithm. The next subsection gives a more refined result that allows us to to characterize the rate of convergence. 
 	
\subsection{Posterior Large Deviations}
This section provides an asymptotic characterization of posterior probabilities $\Pi_{n}(\tilde{\Theta})$ for any open set  $\tilde{\Theta} \subset \Theta$ and under any adaptive measurement strategy. The characterization depends on the notion of Kullback-Leibler divergence. For two parameters $\theta, \theta' \in \mathbb{R}$, the log-likelihood ratio, $\log \left( p(y| \theta)/p(y | \theta') \right)$,
provides a measure of the amount of information $y$ provides in favor of $\theta$ over $\theta'$. The Kullback-Leibler divergence
\[
d( \theta || \theta') \triangleq \intop  \log \left( \frac{p(y| \theta)}{p(y | \theta')} \right) p(y | \theta)d\nu(y).
\]
is the expected value of the log-likelihood under observations drawn $p(y| \theta)$. Then, if the design to measure is chosen by sampling from a probability distribution $\psibf$ over $\{1,..,k\}$,
\[
D_{\psibf}(\thetabf || \thetabf') \triangleq \sum_{i=1}^{k} \psi_i d(\theta_i || \theta'_i)
\]
is the average Kullback-Leibler divergence between $\thetabf$ and $\thetabf'$ under $\psibf$.

%Under the allocation $\psibf$, $D_{\psi}(\thetabf^* || \thetabf)$ captures the amount of information acquired that distinguishes $\thetabf$ from the true parameter $\thetabf^*$.

Under the algorithms we consider, the effort allocated to measuring design $i$, $\psi_{n, i} \triangleq \Prob(I_n = i | \hist)$, changes over time as data is collected. Recall that $\overline{\psi}_{n,i} \triangleq n^{-1} \sum_{\ell=1}^{n} \psi_{\ell,i}$ 
captures the fraction of overall effort allocated to measuring design $i$ over the first $n$ periods. Under an adaptive allocation rule, $\overline{\psibf}_n$ is function of the history $(I_1, Y_{1, I_1},...I_{n-1}, Y_{n-1, I_{n-1}})$ and is therefore a random variable. Given that measurement effort has been allocation according to $\overline{\psibf}_n$, $D_{\overline\psibf_n}(\thetabf^* || \thetabf)$ quantifies the average information acquired  that distinguishes $\thetabf$ from the true parameter $\thetabf^*$. The following proposition relates the posterior mass assigned to $\tilde{\Theta}$ to $\inf_{\thetabf \in \tilde{\Theta}} D_{\overline{\psibf}_n}( \thetabf^* || \thetabf)$, which captures the element in $\tilde{\Theta}$ that is hardest to distinguish from $\thetabf^*$ based on samples from $\overline{\psibf}_n$.
\begin{proposition}\label{prop: posterior concentration}
	For any open set $\tilde{\Theta} \subset \Theta$,
	\[
	\Pi_{n}(\tilde{\Theta}) \doteq \exp\left\{ -n \underset{\thetabf \in \tilde{\Theta}}{\inf}  D_{\overline\psibf_n}(\thetabf^* || \thetabf) \right\}.
	\]
\end{proposition}
To understand this result, consider a simpler setting where the algorithm measures design $i$ in every period, and consider some $\thetabf$ with $\theta_i \neq \theta^*_i$. Then the log-ratio of posteriors densities
\[
\log\left( \frac{\pi_{n}(\thetabf)}{\pi_{n}(\thetabf^*)}  \right) = \log\left( \frac{\pi_{1}(\thetabf)}{\pi_1(\thetabf^*) } \right) + \sum_{\ell =1}^{n-1}  \log\left( \frac{p(Y_{\ell, i} | \theta_{i}) }{ p(Y_{\ell, i} | \theta_{i}^*)}\right)
\]
can be written as the sum of the log-prior-ratio and the log-likelihood-ratio. The log-likelihood ratio is negative drift random walk: it is the sum of $n-1$ i.i.d terms, each of which has mean
\[
\E\left[ \log\left( \frac{p(Y_{1, i} | \theta_{i}) }{ p(Y_{1, i} | \theta_{i}^*)}\right) \right]= \E\left[ - \log\left( \frac{p(Y_{1, i} | \theta_{i}^*) }{ p(Y_{1, i} | \theta_{i})}\right) \right] = -d(\theta_i^* || \theta_i).
\]
Therefore, by the law of large numbers, as $n\to \infty$, $n^{-1}\log\left( \pi_{n}(\thetabf)/\pi_{n}(\thetabf^*)\right) \to -d(\theta^*_i || \theta_i)$, or equivalently, the ratio of the posterior densities decays exponentially as
\[
\frac{\pi_{n}(\thetabf)}{\pi_{n}(\thetabf^*)} \doteq \exp\{-n d(\theta^*_i || \theta_i  \}.
\]
This calculation can be carried further to show that if the designs measured ($I_1, I_2, I_3,...$) are drawn independently of the observations ($\ybf_1, \ybf_2, \ybf_3,...$)  from a fixed probability distribution $\psibf$, then
\begin{equation}\label{eq: initial posterior limit}
\frac{\pi_{n}(\thetabf)}{\pi_{n}(\thetabf^*)} \doteq \exp\left\{ - n D_{\psibf}(\thetabf^* || \thetabf) \\
\right\}.
\end{equation}
Now, by a Laplace approximation, one might expect that the integral $\intop_{\tilde{\Theta}} \pi_{n}(\thetabf)d\thetabf$ is extremely well approximated by integrating around a vanishingly small ball around the point
\[
\hat{\thetabf} = \underset{\thetabf \in \tilde{\Theta} }{\arg\min}\, D_{\psibf}(\thetabf^* || \thetabf) .
\]
These are the main ideas behind Proposition \ref{prop: posterior concentration}, but there are several additional technical challenges
involved in a rigorous proof. First, we need that a property like \eqref{eq: initial posterior limit} holds when the allocation rule is adaptive to the data. Next, convergence of the integral of the posterior density requires a form of uniform convergence in \eqref{eq: initial posterior limit}. Finally, since $\overline{\psibf}_n$ changes over time, the point
$
\underset{\thetabf \in \tilde{\Theta} }{\arg\min}\, D_{\overline{\psibf}_n}(\thetabf^* || \thetabf)
$
changes over time and basic Laplace approximations don't directly apply.

\subsection{Characterizing the Optimal Allocation}\label{subsec: optimal allocation}
Throughout this paper, an experimenter wants to gather enough evidence to certify that $I^*$ is optimal, but since she does not know $\thetabf^*$, she does not know which measurements will provide the most information. To characterize the optimal exponent $\Gamma^*$, however, it is useful to consider the easier problem of gathering the most effective evidence when $\thetabf^*$ is known. We can cast this as a game between two players:

\begin{itemize}
	\item An experimenter, who knows the true parameter $\thetabf^*$, chooses a (possibly adaptive) measurement rule.
	\item A referee observes the resulting sequence of observations $(I_1,Y_{1, I_1},...,I_{n-1}, Y_{n-1, I_{n-1}})$ and computes posterior beliefs $(\alpha_{n,1},..,\alpha_{n,k})$ according to Bayes rule (\ref{eq: posterior density}, \ref{eq: optimal action prob}).
	\item How can the experimenter gather the most compelling evidence? A rule which is optimal asymptotically should maximize the rate at which $\alpha_{n,I^*}\to 1$ as $n \to \infty.$
\end{itemize}
In order to drive the posterior probability $\alpha_{n, I^*}$ to 1, the decision-maker must be able to rule out all parameters in $\Theta_{I^*}^{c}$ under which the optimal action is not $I^*$. Our analysis shows that the posterior probability assigned to $\Theta_{I^*}^{c}$  is dominated by the parameter that is hardest to distinguish from $\thetabf^*$ under $\overline{\psibf}_n$. In particular, by Proposition \ref{prop: posterior concentration},
\[
\Pi_{n}(\Theta_{I^*}^{c}) \doteq \exp\left\{-n \left(\min_{\thetabf \in \Theta_{I^*}^{c}} D_{\overline{\psibf}_n}(\thetabf^* || \thetabf) \right) \right\}
\]
as $n\to \infty.$ Therefore, the solution to the max-min problem
\begin{equation}\label{eq: first optimal error exponent}
\max_{\psibf} \min_{\thetabf \in \Theta_{I^*}^{c}} D_{\psi}(\thetabf^* || \thetabf)
\end{equation}
represents an asymptotically optimal allocation rule.  As highlighted in the literature review, the max-min problem \eqref{eq: first optimal error exponent} closely mirrors the main sample complexity term in Chernoff's classic paper on the sequential design of experiments (\cite{chernoff1959sequential}).

\paragraph{Simplifying the optimal exponent.}
Thankfully, the best-arm identification problem has additional structure which allows us to simplify the optimization problem \eqref{eq: first optimal error exponent}. Much of our analysis involves the posterior probability assigned to the event some action $i\neq I^*$ is optimal. This can be difficult to evaluate, since the set of parameter vectors under which $i$ is optimal
\[
\Theta_{i} = \left\{ \thetabf \in \Theta | \theta_i \geq \theta_1,... \theta_i \geq \theta_k \right\}
\]
involves $k$ separate constraints. Consider instead a simpler problem of comparing the parameter $\theta_i^*$ against $\theta_{I^*}^*$. For each $i\neq I^*$ define the set
\[
\overline{\Theta}_{i} \triangleq  \left\{ \thetabf \in \Theta | \theta_i \geq \theta_{I^*}\right\}  \supset \Theta_i
\]
under which the value at $i$ exceeds that at $I^*$. Since, ignoring the boundary of the set, $\Theta_{I^*}^{c} = \cup_{i\neq I^*}\overline{\Theta}_{i} $, \[
\max_{i \neq I^*} \Pi_{n}(\overline{\Theta}_{i}) \leq  \Pi_{n}(\Theta_{I^*}^{c})  \leq k\max_{i \neq I^*}  \Pi_{n}(\overline{\Theta}_{i})
\]
and therefore
\begin{equation}\label{eq: error probability to theta bar}
\Pi_{n}(\Theta_{I^*}^{c})  \doteq \max_{i \neq I^*}  \Pi_{n}(\overline{\Theta}_{i}).
\end{equation}
This yields an analogue of \eqref{eq: first optimal error exponent} that will simplify our subsequent analysis. Combining \eqref{eq: error probability to theta bar} with Proposition \ref{prop: posterior concentration} shows the solution to the max-min problem
\begin{equation}\label{eq: second optimal error exponent}
\Gamma^* \triangleq \max_{\psibf} \min_{i\neq I^*} \min_{\thetabf \in \overline{\Theta}_{i}} D_{\psibf}(\thetabf^* || \thetabf)
\end{equation}
represents an asymptotically optimal allocation rule. Because the set $\overline{\Theta}_{i}$ involves only a constraints on $\theta_i$ and $\theta_{I^*}$, we can derive an expression the inner minimization problem over $\thetabf$ in terms of the measurement effort allocated to $i$ and $I^*$. Define
\begin{equation}\label{eq: definition of C}
C_{i}(\beta, \psi)  \triangleq \min_{x\in \mathbb{R}} \,\beta d(\theta^*_{I^*} || x) + \psi d(\theta^*_{i} || x).
\end{equation}
The next lemma shows that the function $C_{i}$ arises as the solution to the minimization problem over $\thetabf \in \overline{\Theta}_{i}$ in \eqref{eq: second optimal error exponent}. It also shows that the minimum in \eqref{eq: definition of C} is attained by a parameter $\overline{\theta}$ under which the mean observation is a weighted combination of the means under $\theta^*_{I^*}$ and $\theta^*_{i}$. Recall that, for an exponential family distribution $A'(\theta) = \intop T(y)p(y| \theta)d\nu(y)$ is the mean observation of the sufficient statistic $T(y)$ under $\theta$.
\begin{lemma}\label{lem: solves minimum over theta bar} For any $i\in \{1,..,k\}$ and probability distribution $\psibf$ over $\{1,...,k\}$
	\[\min_{\thetabf \in \overline{\Theta}_{i}} D_{\psibf}(\thetabf^* || \thetabf) = C_{i}(\psi_{I^*}, \psi_i)\]
	In addition, each $C_i$ is a strictly increasing concave function satisfying
	\[
	C_{i}(\psi_{I^*}, \psi_i) = \psi_{I^*}d(\theta^*_{I^*}||\overline{\theta}) + \psi_{i}d(\theta^*_{i}||\overline{\theta}),
	\]
	where $\overline{\theta}\in [\theta^*_i, \theta^*_{I^*}]$ is the unique solution to
	\[
	A'(\overline{\theta}) = \frac{\psi_{I^*} A'(\theta^*_{I^*}) + \psi_{i}A'(\theta^*_{i})}{\psi_{I^*}+\psi_i}.
	\]
\end{lemma}
Lemma \ref{lem: solves minimum over theta bar} and equation \eqref{eq: second optimal error exponent} immediately imply
\begin{equation}\label{eq: final optimal exponent}
\Gamma^* = \max_{\psibf} \min_{i\neq I^*} C_{i}(\psi_{I^*}, \psi_i).
\end{equation}
This result essentially shows that the earlier form of hte exponent, which is similar to a problem complexity measure in  \cite{chernoff1959sequential}, is equivalent to the large deviations exponent suggested in \citet{glynn2004large}. The function $C_{i}(\beta, \psi)$ captures the effectiveness with which one can certify $\theta^*_{I^*}  \geq \theta^*_{i}$ using an allocation rule that measures actions $I^*$ and $i$ with respective frequencies $\beta$ and $\psi$. Naturally, it is an increasing function of the measurement effort $(\beta,\psi)$ allocated to designs $I^*$ and $i$. For given $\beta$ and $\psi$, $C_{i}(\beta, \psi) \geq C_{j}(\beta, \psi)$ when $\theta_i^* \leq \theta_j^*$, reflecting that $\theta^*_i$ is easier to distinguish from $\theta^*_{I^*}$ than $\theta^*_j$. 

\begin{example}(Gaussian Observations) \label{ex: Gaussian rate}
Suppose each outcome distribution $p(y | \theta^*_i)$ is Gaussian with unknown mean $\theta^*_i$. Then direct calculation using Lemma \ref{lem: solves minimum over theta bar} shows
\[
C_i(\beta, \psi_i ) = \left(\frac{\beta \psi_i}{\beta+\psi_i}\right) \frac{(\theta_{I^*}^* - \theta^*_i)^2}{2}. 
\]
To understand this formula, imagine we use a deterministic allocation rule that collects $n\beta$ and $n\psi_i$ observations from $I^*$ and $i$. Let $X_{I^*}$ and $X_i$ denote the respective sample means. The empirical 
difference is normally distributed $X_{I^*}-X_i \sim \mathcal{N}\left(\Delta, \sigma^2/n   \right)$ where $\Delta=\theta^*_{I^*} -\theta^*_i$ and $\sigma^2= 1/\beta +1/\psi_i=(\beta+\psi_i)/(\beta \psi_i)$. Standard Gaussian tail bounds imply that as $n\to \infty$, $\Prob(X_{I^*}-X_i<0) \doteq \exp(-n /2(\sigma \Delta)^2)$, and so $C_i(\beta, \psi_i)$ appears to characterize the probability of error. 
\end{example}

The next proposition formalizes the derivations in this section, and states that the solution to the above maximization problem attains the optimal error exponent.  Recall that $\psi_{n,i} \triangleq \Prob( I_n = i | \hist)$ denotes the measurement effort assigned design $i$ at time $n$.
\begin{proposition}\label{prop: optimal allocation}
	Let $\psibf^*$ denote the optimal solution to the maximization problem \eqref{eq: final optimal exponent}.
	If $\psibf_n = \psibf^*$ for all $n$, then
	\[
	\Pi_{n}(\Theta_{I^*}^{c}) \doteq \exp\{ -n \Gamma^* \}.
	\]
	Moreover under any other adaptive allocation rule,
	\[
	\underset{n\to \infty}{\lim \sup}  \, -\frac{1}{n} \log \Pi_{n}(\Theta_{I^*}^{c}) \leq \Gamma^*.
	\]
\end{proposition}
This shows that under the fixed allocation rule $\psibf^*$ error decays as $e^{-n\Gamma^*}$, and that no faster rate of decay is possible, even under an adaptive allocation.

\paragraph{An Optimal Constrained Allocation.}
Because the algorithms studied in this paper always allocate $\beta$--fraction of their samples to measuring $I^*$ in the long run, they may not exactly attain the optimal error exponent. To make rigorous claims about their performance, consider a modified version of the error exponent \eqref{eq: final optimal exponent} given by the constrained max-min problem
\begin{equation}\label{eq: optimal constrained error exponent}
\Gamma^*_{\beta} \triangleq \max_{\psibf: \psi_{I^*}=\beta} \min_{i\neq I^*} C_{i}(\beta, \psi_i).
\end{equation}
This optimization problem yields the optimal allocation subject to a constraint that $\beta$--fraction of the samples are spent on $I^*$. The next subsection will show that TTTS, TTPS, and TTVS attain the error exponent $\Gamma^*_{\beta}$. The next proposition formalizes that the solution to this optimization problem represents an optimal constrained allocation. In addition, it shows that the solution is the unique feasible  allocation under which $C_{i}(\beta, \psi_i)$ is equal for all suboptimal designs $i\neq I^*$. To understand this result, consider the case where there are three designs and $\theta^*_1 > \theta^*_2 > \theta^*_3$. If $\psi_2=\psi_3$, then $C_2(\beta, \psi_2)< C_{3}(\beta, \psi_3)$, reflecting that it is more difficult to certify that $\theta^*_2 \leq \theta^*_{I^*}$ than $\theta^*_3 \leq \theta^*_{I^*}$. The next proposition shows it is optimal to decrease $\psi_2$ and increase $\psi_1$, until the point when $C_{2}(\beta, \psi_2)=C_{3}(\beta, \psi_3)$. Instead of allocating \emph{equal measurement effort} to each alternative, it is optimal to adjust measurement effort to gather \emph{equal evidence} to rule out each suboptimal alternative. The results in this proposition are closely related to those in \citet{glynn2004large}, in which large deviations rate functions take the place of the functions $C_i$.

\begin{proposition}\label{prop: optimal constrained allocation}
	The solution to the optimization problem \eqref{eq: optimal constrained error exponent} is the unique allocation $\psibf^*$  satisfying $\psi^*_{I^*}=\beta$ and
	\[
	C_{i}(\beta , \psi_i) = C_{j}(\beta, \psi_j) \qquad \forall \, i,j \neq I^*.
	\]
	If $\psibf_n = \psibf^*$ for all $n$, then
	\[
	\Pi_{n}(\Theta_{I^*}^{c}) \doteq \exp\{ -n \Gamma^*_{\beta} \}.
	\]
	Moreover under any other adaptive allocation rule, if $\overline{\psi}_{n, I^*}\to\beta$ then
	\[
	\underset{n\to \infty}{\lim \sup}  \, -\frac{1}{n} \log \Pi_{n}(\Theta_{I^*}^{c}) \leq \Gamma^*_\beta
	\]
	almost surely.
\end{proposition}
The following lemma relates the constrained exponent $\Gamma^*_{\beta}$ to $\Gamma^*$.
\begin{lemma}\label{lem: relating the constrained exponent to the unconstrained} For $\beta^* = \arg\max_{\beta} \Gamma^*_{\beta}$ and any $\beta \in (0,1)$,
	\[
	\frac{\Gamma^*}{\Gamma^*_{\beta}}  \leq \max\left\{\frac{\beta^*}{\beta}, \frac{1-\beta^*}{1-\beta} \right\}.
	\]
	Therefore $\Gamma^* \leq 2\Gamma^*_{1/2}$.
\end{lemma}

\subsection{Convergence of Top-Two Algorithms} \label{subsec: optimal adaptive allocation}
Instead of attempting to directly solve the optimization problem \eqref{eq: final optimal exponent}, this paper focuses on simple and intuitive sequential strategies. These algorithms have the potential to explore much more intelligently in early stages, as they carefully measure and reason about uncertainty. While they ostensibly have no connection to the derivations earlier in this section, we establish that remarkably all three automatically converge to the unknown optimal allocation. This is shown formally in the next result.

We are now ready to establish the paper's main claim, which shows that TTTS, TTPS, and TTVS each attain the error exponent $\Gamma^*_{\beta}$.
\begin{proposition}\label{prop: TS converges to optimal allocation}
	Under the TTTS, TTPS, or TTVS algorithm with parameter $\beta>0$, $\overline{\psibf}_{n} \to \psibf^{\beta}$, where $\psibf^\beta$ is the unique allocation with $\psi^{\beta}_{I^*}=\beta$ satisfying
	\[
	C_{i}(\beta, \psi^\beta_i)=C_{j}(\beta, \psi^\beta_j) \qquad  \forall i,j\neq I^*.
	\]
	Therefore, 
	\[
	\Pi_{n}(\Theta_{I^*}^{c}) \doteq e^{-n \Gamma^*_{\beta}}.
	\]
\end{proposition}
%The proof proceeds by appealing to Proposition \ref{prop: optimal constrained allocation} and showing

To understand this result, imagine that $n$ is very large, and $\overline{\psi}_{n,I^*} \approx\beta$. If the algorithm has allocated too much measurement effort to a suboptimal action $i$, with $\overline{\psi}_{n,i} > \psi_{i}^\beta + \delta$ for a constant $\delta>0$,  then it must have allocated too little measurement effort to at least one other suboptimal design $j\neq i$. Since much less evidence has been gathered about $j$ than $i$, we expect $\alpha_{j,n} >> \alpha_{j,i}$. When this occurs, TTTS, TTPS and TTVS essentially never sample action $i$ until the average effort $\overline{\psi}_{n,i}$ allocated to design $i$ dips back down toward $\psi_{i}^\beta$. This seems to suggest that the algorithm cannot allocate too much effort to any alternative, but that in turn implies that it never allocates too little effort to measuring any alternative.

\subsection{Asymptotics of the Value Measure}
The proof for top-two value sampling relies on the following lemma, which shows that the posterior value of any suboptimal design is logarithmically equivalent to its probability of being optimal. 
\begin{lemma}\label{lem: value and probability are log equivalent}
	For any $i\neq I^*$, $V_{n,i} \doteq \alpha_{n,i}$
\end{lemma}
Note that by this lemma, 
\[
\Pi_{n}(\Theta^c_{I^*}) = \sum_{i\neq I^*} \alpha_{n,i} \doteq \sum_{i\neq I^*} V_{n,i}, 
\]
and so all of the asymptotic results in this could be reformulated as statements concerning the value assigned to suboptimal alternatives under the posterior. 

The lemma is not so surprising, as $V_{n,i} = \intop_{\Theta_i} v_{i}(\thetabf)\pi_{n}(\thetabf)d\thetabf$ differs from $\alpha_{n,i}=\intop_{\Theta_i} \pi_{n}(\thetabf)d\thetabf$ only because of the function $v_{i}(\thetabf)$. The $\pi_{n}(\thetabf)$ term dominates this integral as $n\to \infty$, since it tends to zero at an exponential rate in $n$ whereas $v_{i}(\thetabf)$ is a fixed function of $n$.

\section{Further Simulation Experiments}\label{sec: further experiments}

This section presents further simulation results. The focus is not on competitive benchmarking across the wide array of algorithms that have been proposed by researchers in statistics, operations research and computer science. While this could be enormously valuable, carrying out such experiments in a fair manner has proved challenging, as these algorithms are often derived under differing modeling assumptions and differing problem objectives, as well as with numerous tuning parameters that muddle comparisons. We instead aim here to focus on gaining clear insight into two questions. Namely: 
\begin{enumerate}
	\item How robust is the performance of the proposed top-two sampling algorithm to the choice of tuning parameter? Precisely, across a range of problem instances, how does top-two sampling with the default choice of $\beta=1/2$ compare relative to an omniscient version of the algorithm, which uses the optimal tuning parameter $\beta^*$ for that instance?
	\item How do top-two sampling algorithms, which need to learn and adapt to the long run optimal sampling proportions on each problem instance  $\thetabf^*$, perform relative to an omniscient policy that knows and tracks the ideal sampling proportions $\psibf^*(\thetabf^*)$ on each problem instance? 	%.  perform relative to the asymptotically optimal sampling proportions? In particular, the paper's theoretical results relate optimality to convergence of the long run measurement effort $(\overline{\psi}_{n,1},\ldots \overline{\psi}_{n,k})$ to an the optimal allocation $\psibf^*(\thetabf^*)$. This optimal allocation is not implementable in practice, as it depends on the true qualities of the arms, but can serve as a benchmark in simulation. How do top-two sampling policies fare relative to an optimal oracle policy, which tracks the ideal sampling proportions $\psibf^*(\thetabf^*)$?
\end{enumerate}
This section presents simulation results across 14 problem settings. To reduce computational burden, as well as simplify the presentation of the results, the section focuses on top-two Thompson sampling and omits the other two variants of top-two sampling.  The results reveal strong performance of top-two Thompson sampling with the ad-hoc choice of tuning parameter $\beta=1/2$. Interesting, this method also consistently, and often substantially, outperforms the oracle policy $\psibf^*(\thetabf^*)$.

Each of the fourteen experiments investigates a different problem setting as described in Table \ref{table: experiment specs} below. The problems are divided between those with binary observations and those with standard Gaussian observation noise . For the binary experiments an independent uniform prior is used, while an independent $N(0,1)$ prior is used for the second experiment. We consider several types of configurations for the arm means. Experiments 10-14 present randomly drawn instances, where each $\theta^*_i$ was sampled independently for standard normal distribution. These were drawn using the numpy.random.normal function with seeds 1,2,3,4 and 5, respectively. In the configurations labeled ``ascending'', the arm means increase from lowest to highest with uniform separation between the arms. The slippage configuration was included specifically to investigate cases where top-two sampling performs poorly. In such settings, an equal allocation across arms attains an exponent that is quite competitive, as there are no very poor arms that can be easily ruled out using fewer samples. In addition, the exponent $\Gamma_{\frac{1}{2}}$ attained by TTTS with $\beta=1/2$ can be farther from the optimal $\Gamma^*$ than under other problem instances. The ratio of exponents $\Gamma^*/ \Gamma_{\frac{1}{2}}$ is displayed for each instance. 
  \begin{table}[]
 	\centering
 	\caption{Experiment Specifications}
 	\label{table: experiment specs}
 	\begin{tabular}{llllll}
 		& Noise    & Configuration & $k$ & True Arm Means $(\theta_1^*, \ldots \theta_k^*)$ & $\Gamma^*/\Gamma_{\frac{1}{2}}$ \\ \hline
 		1          & Binary   & Slippage      & 5    & (0.3, 0.3, 0.3, 0.3, 0.5)                                                                                    & 1.12  \\
 		2          & Binary   & Slippage      & 10   & (0.3, 0.3, 0.3, 0.3, 0.3, 0.3, 0.3, 0.3, 0.3, 0.5)                                                           & 1.26  \\
 		3          & Binary   & Slippage      & 15   & (0.3, 0.3, 0.3, 0.3, \ldots, 0.3, 0.3, 0.3, 0.3, 0.5)                                                             & 1.34  \\
 		4          & Binary   & Ascending     & 5    & (0.1, 0.2, 0.3, 0.4, 0.5)                                                                                    & 1.01  \\
 		5          & Binary   & Ascending     & 10   & (0.05, 0.1, 0.15, 0.2, 0.25, 0.3, 0.35, 0.4, 0.45, 0.5)                                                      & 1.01  \\
 		6          & Gaussian & Ascending     & 5    & (-0.5, -0.25, 0, 0.25, 0.5)                                                                                  & 1.01  \\
 		7          & Gaussian & Ascending     & 10   & (-0.5, -0.5, -0.5, -0.5, -0.5, -0.5, -0.25, 0, 0.25, 0.5)                                                    & 1.03  \\
 		8          & Gaussian & Slippage      & 5    & (0, 0, 0, 0, 0.5)                                                                                            & 1.11  \\
 		9          & Gaussian & Slippage      & 10   & (0, 0, 0, 0, 0, 0, 0, 0, 0, 0.5)                                                                             & 1.25  \\
 		10         & Gaussian & Random        & 10   & (-2.3, -1.1, -0.8, -0.6, -0.5, -0.2, 0.3,  0.9,  1.6,  1.7) & 1.00\\
 		11         & Gaussian & Random        & 10   & (-2.1, -1.8, -1.2, -1.1, -0.9, -0.8, -0.4, -0.1,  0.5,   1.6)   & 1.10   \\
 		12         & Gaussian & Random        & 10   & (-1.9, -0.6, -0.5, -0.4, -0.3, -0.1, -0. ,  0.1,  0.4,  1.8)  & 1.19  \\
 		13         & Gaussian & Random        & 10   & (-1.6, -1.1, -1. , -0.6, -0.4, 0.1,  0.3,  0.5, 0.6,  0.7) & 1.01  \\
 		14         & Gaussian & Random        & 10   & (-0.9, -0.6, -0.3, -0.3, -0.3,  0.1,  0.2,  0.4, 1.6,  2.4)& 1.04 
 	\end{tabular}
 \end{table}
 
Figure \ref{fig: Average Sample Size vs Oracle} displays the average number of measurements required for the posterior to reach a given confidence level. In particular, the experiment tracks the frist time when $\max_{i} \alpha_{n,i} \geq c$ for confidence levels $c=.9$ and $c=.99$. All results are averaged over 400 trials. 

The ``Large deviations oracle,'' labeled ``LD oracle'' in Figure \ref{fig: Average Sample Size vs Oracle}, implements the optimal fixed allocation $\psi^*(\thetabf^*)$ as prescribed by large deviations theory. At each time $n$, the algorithm constructs the target proportions $n \cdot\psi^*(\thetabf^*)$ and plays the arm that is most under-sampled relative to these proportions. For problems with Gaussian noise, the optimal computing budget allocation (OCBA) of  \citet{chen2000simulation} is a widely used approximation to the fixed allocation $\psi^*(\thetabf^*)$. The algorithm labeled OCBA oracle implements the true sampling proportions specified by \citet{chen2000simulation} for each problem instance. We also compare the uniform, or equal allocation, TTTS with tuning parameter $\beta=1/2$ and TTTS Oracle, which is TTTS with the optimal problem dependent tuning parameter $\beta^*$. 

At a high level, there are two key findings from these experiments. In all cases, sample size comparisons refer to the confidence level $c=.99$. 
\begin{enumerate}
	\item Top-two Thompson sampling with tuning parameter 1/2 generally offers similar performance to top-two Thompson sampling with the optimal tuning parameter $\beta^*$. The most significant separation in performance was on slippage configurations, where TTTS with optimal tuning parameter saved up to 15\% of samples on average. On most other instances, using the optimal tuning parameter offered no improvement. 
	\item The large deviations oracle  and the OCBA oracle were consistently, and sometimes dramatically, outperformed. Each one required least 19\% more samples on average than ${\rm TTTS}(1/2)$ for \emph{all 14 experiments}. In their worst experiments, the LD oracle and OCBA oracle used respectively more than 200\% and 300\% the average number of samples used by ${\rm TTTS}(1/2)$.
\end{enumerate}

The second finding may be quite surprising to some readers. There is a quite a large literature that aims to implement optimal large deviations allocations derived in \cite{glynn2004large}, or a simpler approximation to these in the Gaussian case known as the OCBA \citep{chen2000simulation}. Such approaches have also been extended to a number of related problem settings. A major challenge, however, is that the allocations cannot be directly implemented as they require knowledge of the true problem instance $\thetabf^*$. Researchers typically implement an approach that solves for the optimal budget allocation under point estimate $\hat{\thetabf}$ of $\thetabf^*$, aiming to converge to the prescribed optimal sampling proportions as rapidly as possible. Here, we instead compete against an oracle that knows and carefully follows the asymptotically optimal sampling proportions for each problem instance. Even these oracle policies are significantly outperformed by Top-two Thompson with the an ad-hoc choice of tuning parameter. 

It is an open question to formalize the reasons for this empirical finding. It is worth offering some possible intuition, however. First, the oracle allocations are based on a number of approximations, including tail approximations to the posterior of each arm and certain union bounds or Bonferonni approximations. By contrast, Thompson sampling uses exact samples from the posterior distribution, and may more accurately reflect uncertainty in early stages. Second, even if the oracle allocations know the true-arm means, they do not adapt in response to unusual observations. Thompson sampling, on the other hand is fully adaptive, and can gather fewer samples from an arm if early samples provide strong evidence that arm is suboptimal. Some of the benefits of adaptivity are suggested in \cite{}.

\begin{figure}[h!]
	\centering
	\begin{subfigure}{.5\textwidth}
		\includegraphics[width=1\linewidth]{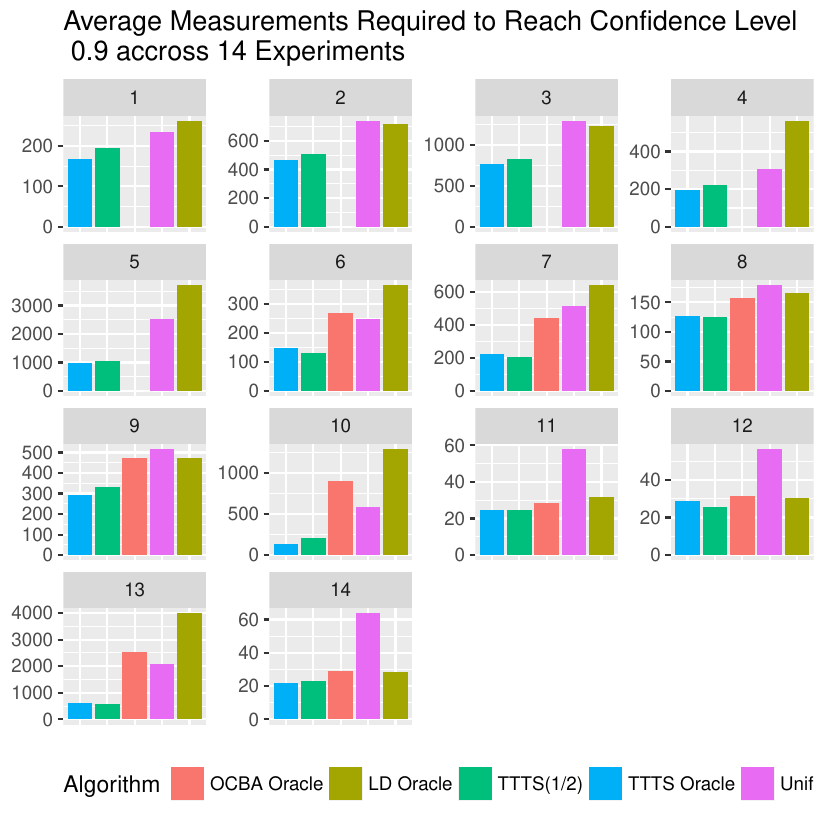}
		%\caption{Measurements collected of each design.}
		\label{fig:sub1}
	\end{subfigure}%
	\begin{subfigure}{.5\textwidth}
		\includegraphics[width=1\linewidth]{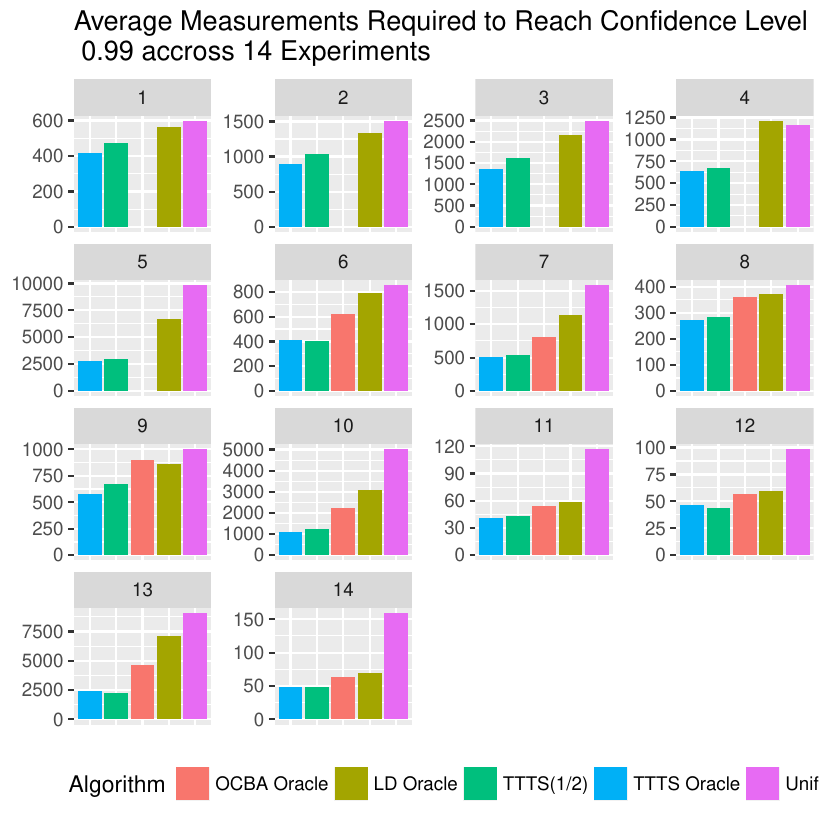}
		%\caption{Value of $\log(1/\alpha_{n,i})$ for each design $i$.}
		\label{fig:sub2}
	\end{subfigure}
	\caption{Average sample size required to reach confidence relative to ``oracle'' allocations.} 
	\label{fig: Average Sample Size vs Oracle}
\end{figure}

\section{Extensions and Open Problems}\label{sec: conclusion}
This paper studies efficient adaptive allocation of measurement effort for identifying the best among a finite set of options or designs. We propose three simple Bayesian algorithms. Each is a variant of what we call top-two sampling, which, at each time-step, measures one of the two designs that appear most promising given current evidence. Surprisingly, these seemingly naive algorithms are shown to satisfy a strong asymptotic optimality property. 

Top two sampling appears to be a general design principle that can be extended to address a variety of problems beyond to the scope of this paper. To spur research in this area, we briefly discuss a number of extensions and open questions below.

%I hope that due to its simplicity and ease of implementation, the top-two Thompson sampling rule in particular can have immediate impact among practitioners. Going forward, I hope that these simple rules can be adapted to treat much more complicated problems, and more broadly that this work provides useful insight into the design and analysis of adaptive Bayesian methods for gathering information.  

\paragraph{Top-Two Sampling Via Constrained MAP Estimation.}
Here we present a version of top-two sampling that uses MAP estimation.  This can simplify computations, as MAP estimates can be computed without solving for the normalizing constant of the posterior density $\pi_{n}(\thetabf)$. Consider the following procedure for selecting a design at time $n$:
\begin{enumerate}
	\item Compute $\hat{\thetabf} \in \arg\max_{\thetabf \in \Theta} \pi_{n}(\thetabf)$ and set $\hat{I}_n = \arg\max_{i} \hat{\theta}_i$. 
	\item Compute $\hat{\thetabf}' \in \arg\max_{\thetabf \in \Theta^c_{\hat{I}_n}} \pi_{n}(\thetabf)$ and set $\hat{J}_n = \arg\max_{i} \hat{\theta}'_i$. 
	\item Play $(\hat{I}_n, \hat{J}_n)$ with respective probabilities $(\beta, 1-\beta)$. 
\end{enumerate}
The first step uses MAP estimation to make a prediction $\hat{I}_n$ of the best design, while the second uses constrained MAP estimation to identify the alternative design that is most likely to be optimal when $\hat{I}_n$ is not. 
Many of the asymptotic calculations in the previous section appear to extend to this algorithm, but proving this formally is left as an open problem.

\paragraph{Indifference Zone Criterion.} Suppose our goal is to confidently identify an $\epsilon$--optimal arm, for a user specified indifference parameter $\epsilon>0$. Much of the paper investigates the set of parameters $\Theta_{i}$ under which arm $i$ is optimal, and studies the rate at which $\Pi_{n}(\Theta_{I^*}) \to 1$. Now, let us instead consider the set of parameters
\[
\Theta_{\epsilon, i}= \{\thetabf | \theta_i \geq  \max_{j} \theta_{j} - \epsilon \}
\]
under which $i$ is $\epsilon$--optimal. It is easy to develop a variety of modified top-two sampling rules under which $\max_{i} \Pi_{n}(\Theta_{\epsilon,i}) \to 1$ rapidly. For example, we can extend TTPS as follows: set $\hat{I}_n = \arg\max_{i} \Pi_{n}(\Theta_{\epsilon,i})$. Define $\hat{J}_n=\arg\max_{j\neq \hat{I}_n} \Pi_{n}(\thetabf | \theta_j = \max_{i}\theta_i\,\, \& \,\, \theta_j> \theta_{\hat{I}_n} + \epsilon   ) $ to be the alternative design that is most likely to be optimal and offer an $\epsilon$--improvement over $\hat{I}_n$. A top-two Thompson sampling approach might instead continue sampling $\thetabf \sim \Pi_n$ until $\max_{i} \theta_i > \theta_{\hat{I}_n}+\epsilon$ and then set $J_{n} = \arg\max_{i} \theta_i$. 

\paragraph{Top $m$--arm identification.} Suppose now that our goal is to identify the top  $m < k$ designs. Consider choosing a design to measure at time $n$ by the following steps: 
\begin{enumerate}
	\item Sample $\thetabf \sim \Pi_n$ and compute the top $m$ designs under $\thetabf$. 
	\item Continue sampling $\thetabf' \sim \Pi_n$ until the top $m$ designs under $\thetabf'$ differ from those under $\thetabf$. 
	\item Identify the set of designs that are in the top $m$ under $\thetabf$ or under $\thetabf'$, but not under both. Choose a design to measure by sampling one uniformly at random from this set. 
\end{enumerate} 
This is the natural extension of top-two Thompson sampling to the top-$m$ arm problem. In fact, when $m=1$, this is exactly TTTS with $\beta=1/2$. I conjecture that like the case where $m=1$, this algorithm attains a rate of posterior convergence within a factor of 2 of optimal for general $m$. The optimal exponent for this problem can be calculated by mirroring the steps in Subsection \ref{subsec: optimal allocation}. 
 
\paragraph{Extremely Correlated Designs.}
While our results apply in the case of correlated priors, the proposed algorithms may be wasteful when there are a large number of designs whose qualities are extremely correlated. As an example, consider an extension of our techniques to a pure-exploration variant of a linear bandit problem. Here we associate each action $i$ with a feature vector $x_i \in \mathbb{R}^d$ and seek an action that maximizes $x_i^T \theta$. The vector $\theta \in \mathbb{R}^d$ is unknown, but we begin with a prior $\theta\sim N(0,I)$ and see noisy observations of $x_i ^T \theta$ whenever action $i$ is selected. To apply top-two sampling to this problem, we should modify the algorithm's second step. For example, under top-two Thompson sampling, we usually begin drawing a design according to $\hat{i}\sim \alphabf_n$, and then continue drawing designs $\hat{j}\sim \alphabf_n$ until $\hat{i}\neq \hat{j}$. These are played with respective probabilities $(\beta,1-\beta)$. But even if $\hat{i}\neq \hat{j}$, their features may be nearly identical. A more natural extension of top-two Thompson sampling would modify the second step, and continue sampling $\hat{j}\sim \alphabf_n$, until a sufficiently different action is drawn -- for example until the angle between $x_{\hat{j}}$ and $x_{\hat{i}}$ exceeds a threshold. 

\paragraph{Tuning $\beta$.}
The most glaring gap in this work may be arbitrary choice of tuning parameter $\beta$. Optimal asymptotic rates can be attained by adjusting this parameter over time by solving for an optimal allocation as in \eqref{eq: final optimal exponent}. It is an open problem to instead develop simple algorithms that set $\beta$ automatically through value of information calculations, or avoid the need for such a parameter altogether. 

\paragraph{Adaptive Stopping.} This paper proposed only an allocation rule, which determines the sequence of measurements to draw, but this can be coupled with a rule that determines when to stop sampling. One natural stopping rule in a Bayesian framework is to stop when $\max_{i}\alpha_{n,i}>1-\delta$ for some $\delta>0$. Let $\tau_{\delta}$ be a random variable indicating the stopping time under constraint $\delta$. Since 
$1-\max_{i}\alpha_{n,i} \doteq e^{-n\Gamma^*_\beta}$ under top-two sampling, our results imply that for each sample path $\tau_{\delta} \sim \Gamma^*_{\beta} \log(1/\delta)$ as $\delta \to 0$. It is natural to conjecture that  $\E[\tau_{\delta}] \sim \Gamma^*_{\beta} \log(1/\delta)$ as well. This closely mirrors optimal results in \cite{chernoff1959sequential, jennison1982asymptotically} and \cite{kaufmann2016bayesian}. Does this rule also yield a frequentist probability of incorrect selection that is $O(\delta)$ as $\delta\to 0$? More generally, an open problem is to show that when combined with an appropriate stopping rule, top-two sampling schemes nearly minimize the expected number of samples $\E[\tau_{\delta}]$  as in \cite{jennison1982asymptotically} or \cite{kaufmann2016bayesian}.

\setlength{\bibsep}{8pt}
{%\small
\singlespacing
\bibliography{references}
\bibliographystyle{plainnat}
}

\appendix
\section{Outline}
This technical appendix is organized as follows. 
\begin{enumerate}
	\item Section \ref{sec: implementation of TTVS} describes a numerical algorithm that can be used to implement TTPS. 
	\item Section \ref{sec: EI} provides a more precise discussion of related work by \cite{ryzhov2016convergence}. 
	\item The theoretical analysis begins in Section \ref{section: appendix preliminaries}. There we begin by noting some basic facts of exponential family distributions, as well as some results relating martingales to their quadratic variation process. 
	\item Section \ref{sec: appendix posterior concentration} establishes results related to the concentration of the posterior distribution, including the proofs of Prop.~\ref{prop: posterior consistency}, Prop.~\ref{prop: posterior concentration}, and Lemma~\ref{lem: value and probability are log equivalent}. 
	\item Section \ref{sec: appendix error exponent} studies and simplifies the optimal exponents $\Gamma^*$ and $\Gamma^*_{\beta}$, including the proofs of Lemma \ref{lem: solves minimum over theta bar}, Prop.~\ref{prop: optimal constrained allocation}, Lemma \ref{lem: relating the constrained exponent to the unconstrained}, Prop.~\ref{prop: subgaussian bound}, and Prop.~\ref{prop: convergence rate of uniform allocation}. 
	\item  We conclude with Section \ref{sec: appendix proof of main result}, which studies the top-two allocation rules and provides a proof of Prop.~\ref{prop: TS converges to optimal allocation}. 
\end{enumerate}

\section{An Implementation of TTPS}\label{sec: implementation of TTVS}
This section describes an implementation of the top-two probability sampling for a problem with a Beta prior and binary observations. In this problem, measurements are binary with success probability given by $\Prob(Y_{n,i}=1) = \theta^*_i$. The algorithm begins with an independent prior, under which the $i$th component of $\thetabf$ follows a Beta distribution with parameters $(\lambda_{i}^1, \lambda_{i}^2)$. When $\lambda_{i}^2=\lambda_{i}^2=1$, this specifies a uniform prior over $[0,1]$. This prior distribution can be easily updated to form a posterior distribution according to the update rule given in line 19 of Algorithm \ref{alg: BernoulliTTPS}. 

This algorithm uses quadrature to approximate the integral defining $\alpha_{n,i}$. To understand this implementation, consider a random vector $(X_1,..,X_K)$ whose components are independently distributed with $X_{i}\sim {\rm Beta}(\lambda_{i}^{1}, \lambda^{2}_i)$. Then, the probability component $i$ is maximal can be computed according to  
\begin{eqnarray*}
	\Prob(X_i = \max_{j} X_j) &=&  \intop_{x \in \mathbb{R}}   \Prob(\cap_{j\neq i} \{X_j \leq x\})\Prob(X_i = dx) \\
	&=& \intop_{x \in \mathbb{R}} \left[ \prod_{j\neq i} \Prob(X_j \leq x) \right] \Prob(X_i = dx)\\
	&=&\intop_{x \in \mathbb{R}}   \left[ \left(\prod_{j=1}^{K} \Prob(X_j \leq x) \right)/ \Prob(X_i \leq x)\right] \Prob(X_i = dx).
\end{eqnarray*}
Algorithm \ref{alg: BernoulliTTPS} takes as input a vector of ${\bm x}$ consisting of $M$ points in $(0,1)$ and approximates the above integral using quadrature at these points. The algorithm computes and updates the posterior PDF and CDF of $\theta_i$ in an $M$ dimensional vectors ${\bm f}_{i}$ and ${\bm F}_i$. It also stores and updates a vector  $\overline{\bm F} = \prod_{i=1}^{K} F_{i,m}$, where $\overline{F}_m$ is the posterior probability all the designs have quality below $x_m$. Using these quantities, the posterior probability design $i$ is optimal is approximated by a sum in line 8. Lines 11-15 select an action according to TTPS and lines 18-21 update the stored statistics of the posterior using Bayes rule. The algorithm continues for $N$ time steps, and upon stopping returns the posterior parameters ${\bm \lambda}^{1}$ and ${\bm \lambda}^{2}$, which summarize all evidence gathered throughout the measurement process.  The algorithm has $O(NKM)$ space and time complexity. It is worth noting that most operations in this algorithm can be implemented in a ``vectorized'' fashion in languages like MATLAB, NumPy, and Julia.

\begin{algorithm}[H]
	\caption{$\text{BernoulliTTPS}(\beta, K, M, N, {\bm \lambda}^1, {\bm \lambda}^{2}, {\bm x})$}\label{alg: BernoulliTTPS}
	\begin{algorithmic}[1]
		\State \textcolor{blue}{$\backslash\backslash$Initialize:}
		\State $f_{i,m}\leftarrow \text{Beta.pdf}(x_m| \lambda^1_i, \lambda^2_i) \qquad \forall i,m$
		\State $F_{i,m}\leftarrow \text{Beta.cdf}(x_m| \lambda^1_i, \lambda^2_i) \qquad \forall i,m$
		\State $\overline{F}_{m} \leftarrow \prod_{i} F_{i,m} \qquad \forall m$\\
		\For{$n=1 \ldots N$}
		\State \textcolor{blue}{$\backslash\backslash$Compute Optimal Action Probabilities:}
		\State $ \alpha_i \leftarrow \sum_{m} f_{i,m} \overline{F}_m / F_{i,m}  \qquad \forall i$\\
		\State \textcolor{blue}{$\backslash\backslash$Act and Observe:}
		\State $J_1 \leftarrow \arg\max_{i} \alpha_i$
		\State $J_2 \leftarrow \arg\max_{i\neq J_1} \alpha_i$
		\State Sample $B\sim \text{Bernoulli}(\beta)$
		\State $I \leftarrow BJ_1 + (1-B)J_2$. 
		\State Play $I$ and Observe $Y_{n,I}\in \{0,1\}$.  \\ 
		\State \textcolor{blue}{$\backslash\backslash$Update Statistics:}
		\State $(\lambda^{1}_{I}, \lambda^{2}_{I}) \leftarrow  (\lambda^{1}_{I}, \lambda^{2}_{I})+(Y_{n,I}, 1-Y_{n,I})$ 
		\State $\overline{F}_m \leftarrow  (\overline{F}_m/F_{I,m}) \times \text{Beta.cdf}(x_m | \lambda^{1}_{I}, \lambda^{2}_{I}) \qquad \forall m $ 
		\State  $F_{I,m}\leftarrow \text{Beta.cdf}(x_m| \lambda^1_I, \lambda^2_I) \qquad \forall m$
		\State $f_{I,m}\leftarrow \text{Beta.pdf}(x_m| \lambda^1_I, \lambda^2_I) \qquad \forall m$
		\EndFor \\
		\Return ${\bm V}, {\bm \lambda}^1, {\bm \lambda}^2$
	\end{algorithmic}
\end{algorithm}

\section{Discussion of the Expected Improvement Algorithm}\label{sec: EI}
Here, we briefly discuss interesting recent results of \cite{ryzhov2016convergence}. He studies a setting with an uncorrelated Gaussian prior, and Gaussian observation noise $Y_{n,i} \sim N(\theta_i , \sigma_i^2)$. To simplify our discussion, let us restrict attention to the case of common variance $\sigma_1=...=\sigma_k = \sigma$. \cite{ryzhov2016convergence} shows that under the the expected-improvement algorithm, in the limit as $n \to \infty$
\begin{equation}\label{eq: log effort}
 \sum_{i\neq I^*} \Psi_{n, i} = O(\log n) 
\end{equation}
and 
\begin{equation}\label{eq: OCBA ratios}
\Psi_{n, i}(\theta_I^* - \theta_i)^2 \sim \Psi_{n, j}(\theta_I^* - \theta_j)^2 \qquad \forall i,j\neq I^* 
\end{equation}
Recall that $\Psi_{n,i}=\sum_{\ell=1}^{n} \psi_{n,i}$ denotes the total measurement effort allocated to design $i$. The sampling ratios \eqref{eq: OCBA ratios} are the ratios suggested in the optimal computing budget allocation of \cite{chen2000simulation}. This work therefore establishes an interesting link between EI and OCBA, which appear quite different on the surface. 

Unfortunately, property \eqref{eq: log effort} is \emph{not suggested} by the OCBA, and implies that $\Pi_{n}(\Theta^c_{I^*})$ cannot tend to zero at an exponential rate. To see this precisely, assume without loss of generality that $I^*=1$. Then \eqref{eq: log effort} implies $\overline{\psibf}_n \to \mathbf{e}_i \equiv (1,0,0,...,0)$. It is easy to show that 
$\min_{\thetabf \in \Theta_1^c} D_{\mathbf{e}_i}( \thetabf^* || \thetabf)  = 0$
and therefore, by Proposition \ref{prop: posterior concentration},
\[
\lim_{n\to \infty} -n^{-1} \log \Pi_{n}( \Theta_1^c) = 0.
\]

It is also worth noting that the sampling ratios in \eqref{eq: OCBA ratios} are not actually optimal for any finite number of designs $k$. Specifying our calculations as in Example \ref{ex: Gaussian rate}, one can show that under an optimal fixed allocation $(\psi_i,...,\psi_k)$,
\[
 \frac{(\theta_{I^*}^* - \theta^*_i)^2}{1/\psi_{I^*}+1/\psi_{i}} = \frac{(\theta_{I^*}^* - \theta^*_j)^2}{1/\psi_{I^*}+1/1/\psi_{j}} \qquad \forall i,j \neq I^*. 
\]
These calculations match those in \cite{glynn2004large} and \cite{jennison1982asymptotically}. As a result, there is no problem with finite $k$ for which the sampling ratios in \eqref{eq: OCBA ratios} are optimal\footnote{There does exists a sequence of problem instances with $K\to \infty$ on which the OCBA ratios converge to those of \cite{glynn2004large}. See \cite{pasupathy2015stochastically}. } One can show, in fact, that any optimal multi-armed bandit algorithm that attains the lower bound of \cite{lai1985asymptotically} also satisfies equations 
\eqref{eq: log effort} and \eqref{eq: OCBA ratios}. The main innovation in this paper is to show how to build on such bandit algorithms to attain near-optimal rates for the best-arm identification problem. 

\cite{ryzhov2016convergence} also studies the knowledge gradient policy, which could offer improved performance as \eqref{eq: log effort} no longer holds, but shows that as $n\to \infty$
\[
\Psi_{n, i}(\theta_I^* - \theta_i) \sim \Psi_{n, j}(\theta_I^* - \theta_j) \qquad \forall i,j\neq I^*, 
\]
which could be very far from the optimal sampling proportions.

\section{Preliminaries}\label{section: appendix preliminaries}
This section presents some basic results which will be used in the subsequent analysis. First, unless clearly specified, all statements about random variables are meant to hold with probability 1. So for sequences of random variables $\{X_n\}$ and $\{Y_n\}$, if we say that $X_n\to \infty$ whenever $Y_n \to \infty$, this means that the set $\{ \omega : Y_n(\omega) \to \infty, X_n(\omega) \nrightarrow \infty\}$ has measure zero.

\paragraph{Facts about the exponential family.}
The log partition function $A(\theta)$ is strictly convex and differentiable, with
\begin{equation}\label{eq: derivative is mean}
A'(\theta) = \intop T(y) p(y| \theta) d\nu(y)
\end{equation}
equal to the mean under $\theta$. The Kullback-Leibler divergence is equal to
\begin{equation}\label{eq: KL for exponential families}
d(\theta || \theta') = (\theta-\theta')A'(\theta) - A(\theta) + A(\theta')
\end{equation}
and satisfies
\begin{eqnarray}\label{eq: KL monotone}
\theta'' > \theta' \geq \theta &\implies& d(\theta || \theta'') > d(\theta || \theta') \\
\theta'' < \theta' \leq \theta &\implies & d(\theta || \theta'') < d(\theta || \theta').
\end{eqnarray}
Finally, since $[\underline{\theta}, \overline{\theta}]$ is bounded, and we have assumed $\sup_{\theta \in [\underline{\theta}, \overline{\theta}]} |A'(\theta)| < \infty $,
\begin{equation}\label{eq: uniform bound on A and KL}
\sup_{\theta \in [\underline{\theta}, \overline{\theta}]} |A(\theta)| < \infty \quad \text{and} \quad \sup_{\theta, \theta' \in [\underline{\theta}, \overline{\theta}]} d(\theta || \theta') < \infty.
\end{equation}
This effectively guarantees no single observation can provide enough information to completely rule out a parameter.

\paragraph{Some martingale convergence results.}
The next fact relates the behavior of a martingale $M_n$ to its quadratic variation $\langle M \rangle_n$. 
\begin{fact}\label{fact: martingale quadratic variation}(\citet{williams1991probability}, 12.13-12.14)
	Let $\{M_n\}$ be a square-integrable martingale adapted to the filtration $\{\mathcal{H}_{n}\}$ and let 
	\[
	\langle M \rangle_{n} = \sum_{\ell=1}^{n} \E[ \left( M_{\ell} - M_{\ell-1} \right)^2 | \mathcal{H}_{\ell-1}] 
	\] 
	denote the corresponding quadratic variation process.  Then
	\[ 
	\frac{M_n}{\langle M \rangle_n} \to \infty
	\]
	almost surely if $\langle M \rangle_n \to \infty$ and $\lim_{n\to \infty} M_n$ exists and is finite almost surely if $\lim_{n\to \infty} \langle M \rangle_n <\infty$. 
\end{fact}

The next lemma is crucial to our analysis. To draw the connection with our setting, imagine an adaptive-randomized rule is used to determine when to draw samples from a population. Here $Y_n \in \mathbb{R}$ denotes the sample at time $n$, $X_n\in \{0,1\}$ indicates whether the sample was measured, and $Z_n\in [0,1 ]$ determines the probability of measurement conditioned on the past. This lemma provides a law of large numbers when measurement effort $\sum_{\ell=1}^{n} Z_\ell$ tends to infinity, but shows that if measurement effort is finite then $\sum_{\ell=1}^{\infty} X_\ell Y_\ell$ is also finite; in this sense the observations collected from ${Y_n}$ are inconclusive when measurement effort is finite. 
\begin{lemma}\label{lem: adaptive LLN}
	Let $\{Y_n \}$ be an i.i.d sequence of real-valued random variables with finite variance and let $\{X_n\}$ be a sequence of binary random variables. Suppose each sequence is adapted to the filtration $\{\mathcal{H}_n\}$, and define $Z_n = \Prob(X_n = 1 | \mathcal{H}_{n-1})$. If, conditioned on $\mathcal{H}_{n-1}$, each $Y_n$ is independent of $X_n$, then with probability 1,
	\[
	\lim_{n\to \infty }\sum_{\ell=1}^{n} Z_\ell = \infty \implies \lim_{n\to \infty } \frac{\sum_{\ell=1}^{n} X_\ell Y_\ell}{\sum_{\ell=1}^{n} Z_\ell} = \E[Y_{1}]
	\]  
	and 
	\[ 
	\lim_{n\to \infty} \sum_{\ell=1}^{n} Z_\ell < \infty \implies  \sup_{n \in \N} \left|\sum_{\ell=1}^{n} X_\ell Y_\ell\right| < \infty.  
	\]
\end{lemma}
\begin{proof}
Let $\mu= \E[Y_1]$ and $\sigma^2= \E[(Y_1 - \E[Y_1])^2]$ denote the mean and variance of each $Y_n$. Define the martingale
\[
	M_n = \sum_{\ell=1}^{n} (X_\ell Y_\ell - Z_\ell \mu)  
\] 
with $M_0=0$ and put $S_n = \sum_{\ell=1}^{n} Z_\ell$. This martingale has quadratic variation
\begin{eqnarray*}
\langle M \rangle_n &=& \sum_{\ell=1}^{n}\E[(M_{\ell} - M_{\ell-1})^2 | \mathcal{H}_{\ell-1} ] \\
&=& \sum_{\ell=1}^{n}\E[ \left( X_\ell(Y_\ell -\mu)  + (Y_\ell -Z_\ell)\mu \right)^2  |  \mathcal{H}_{\ell-1}] \\
&=& \sum_{\ell=1}^{n} Z_\ell  \sigma^2 + \sum_{\ell=1}^{n} Z_{\ell}(1-Z_\ell)\mu^2 \\ 
&\leq&  (\sigma^2 + \mu^2) S_n.
\end{eqnarray*}
We use the shorthand $S_\infty = \lim_{n\to \infty} S_n$ and $\langle M \rangle_\infty=\lim_{n\to \infty} \langle M \rangle_n$.

Suppose $S_\infty < \infty$ so $\langle M \rangle_\infty < \infty$.  By Fact \ref{fact: martingale quadratic variation}, $\lim_{n\to \infty} M_n$ exists and is finite almost surely, which implies $\sup_{n\in \N} | M_n| < \infty$.  Since $|\sum_{\ell=1}^{n} X_\ell Y_\ell| \leq |M_n| + |\mu S_\infty|$, this shows $\sup_{n\in \N} |\sum_{\ell=1}^{n} X_\ell Y_\ell| <\infty$ as desired. 

Now, suppose $S_\infty = \infty$. If $\langle M \rangle_\infty <\infty$, then again by Fact \ref{fact: martingale quadratic variation}, $\lim_{n\to \infty} M_n < \infty$ and it is immediate that $S_n^{-1} M_n \to 0$. However, if $\langle M \rangle_\infty = \infty$ then 
\[
\frac{M_n}{\langle M \rangle_n} \to 0,
\]
which implies $S_{n}^{-1} M_n \to 0$ since $S_{n} \geq (\sigma^2 + \mu^2) \langle M\rangle_n$. 
\end{proof}
Taking $Y_{n}=1$ in the lemma above yields Levy's extension of the Borel--Cantelli lemmas (\citet{williams1991probability}, 12.15). Specialized to our setting, this result relates the long run measurement effort $\Psi_{n,i} =\sum_{\ell =1}^{n} \psi_{n,i}$ to the number of times alternative $i$ is actually measured $\sum_{\ell=1}^{n} \mathbf{1}(I_n=i)$.  
\begin{corollary}\label{cor: convergence of action selection probabilities to the empirical distribution}
	For $i \in \{1,...,k\}$, set $S_{n,i} = \sum_{\ell=1}^{n}\mathbf{1}(I_n=i)$. Then, with probability 1,
\[
\Psi_{n,i} \to \infty \iff S_{n,i} \to \infty
\]	
and 
\[
\Psi_{n,i} \to \infty  \implies  \frac{S_{n,i}}{\Psi_{n,i}} \to 1.
\]
\end{corollary}
\begin{proof}
	Apply Lemma \ref{lem: adaptive LLN} with $Y_n=1$, $X_n=\mathbf{1}(I_n =1)$, and $\mathcal{H}_n = \mathcal{F}_n$. Then $Z_n = \psi_{n,i}$ by definition. 
\end{proof}
%The following result shows that in addition $S_{n,i} \to \infty $ if and only if $\sum_{\ell = 1}^{n} \psi_{n,i}\to \infty$. This is
%\begin{fact}\label{fact: williams}(\citet{williams1991probability}, 12.15)
%	Let $(X_{n} : n \in \mathbb{N})$ be a sequence of binary random variables adapted to the filtration $(\mathcal{H}_n : n=0,1,2...)$. Then
%	\[
%	\lim_{n\to \infty} \sum_{\ell=1}^{n} X_\ell = \infty \iff \lim_{n\to \infty} \sum_{\ell=1}^{n} \Prob(X_\ell=1 | \mathcal{H}_{\ell-1}) = \infty.
%	\]
%\end{fact}

\section{Posterior Concentration and anti-Concentration}\label{sec: appendix posterior concentration}

\subsection{Uniform Convergence of the Log-Likelihood}	
We study the log-likelihood
\[
\Lambda_{n}(\thetabf^* || \thetabf ) \triangleq \log \left(\frac{L_{n}(\thetabf^*)}{L_{n}(\thetabf)} \right)= \sum_{\ell=1}^{n} \log\left( \frac{p(Y_{\ell,I_\ell}| \theta^*_{I_\ell})}{ p(Y_{\ell,I_\ell}| \theta_{I_\ell})}  \right)
\]
and the log-likelihood from observations of design $i$
\[ 
\Lambda_{n,i}(\theta_i^* || \theta_i) \triangleq \sum_{\ell=1}^{n} \mathbf{1}(I_n =i) \log\left( \frac{p(Y_{n,i}| \theta_{i}^*)}{ p(Y_{n,i}| \theta_{i})}  \right).
\]	
A Doob-decomposition expresses $\Lambda_{n,i}(\theta_i)=A_{n}(\theta_i)+M_{n}(\theta_i)$ as the sum of an $\hist$ predictable process $A_n(\theta_i)$ and a Martingale $M_{n}(\theta_i)$. Moreover, an easy calculation shows $A_{n}(\theta_i) = \Psi_{n,i}d(\theta^*_i || \theta_i)$ and $M_{n}(\theta_i) = \Lambda_{n,i}(\theta^*_i ||\theta_i)-\Psi_{n,i}d(\theta^*_i || \theta_i)$. Applying Lemma \ref{lem: adaptive LLN} shows $\Psi_{n,i}^{-1} M_{n}(\theta_i)\to 0$ if $\Psi_{n,i}\to \infty$, which shows the log-likelihood ratio tends to infinity at rate $\Psi_{n,i}d(\theta^*_i || \theta_i)$.  The next lemma strengthens this, and provides a link between these quantities that holds uniformly in $\theta_i$. 
\begin{lemma}\label{lem: uniform convergence of marginal log likelihood} With probability 1, if $\Psi_{n,i} \to \infty$ then	
\[
\sup_{\theta_i \in [\underline{\theta}, \overline{\theta}]}\Psi_{n,i}^{-1} \left|  \Lambda_{n,i}(\theta^*_i || \theta_i)   - \Psi_{n,i} d(\theta_i^* || \theta_i) \right| \to 0 ,
\]
and if $\lim_{n\to \infty } \Psi_{n,i} < \infty$ then
\[ 
\sup_{\theta_i \in [\underline{\theta}, \overline{\theta}]} \sup_{n\in \mathbb{N}} \left| \Lambda_{n,i}(\theta_i)\right|  +  \left|\Psi_{n,i} d(\theta_i^* || \theta_i) \right| < \infty.
\] 
\end{lemma}
\begin{proof}
	Define $\xi_n \triangleq T(Y_{n, i})-\E[T(Y_{n, i})]$  and $X_n \triangleq \mathbf{1}(I_n =i)$. Note that $\E[\xi_n | \hist]=0$, $\E[X_n | \hist]= \psi_{n,i}$, and, conditioned on $\hist$, $X_n$ is independent of $\xi_n$.  Using the form of the exponential family density given in equation \eqref{eq: exponential family density}, and the form of the KL-divergence given in equation \eqref{eq: KL for exponential families}, the log-likelihood ratio can be written as
	\begin{eqnarray*} 
	\log\left( \frac{p(Y_{n,i}| \theta_{i}^*)}{ p(Y_{n,i}| \theta_{i})}\right)  & =&  (\theta^*_i - \theta_i)T(Y_{n, i}) - (A(\theta_i^*)-A(\theta_i)) \\
	& =& d(\theta^*_i || \theta_i)  +  (\theta_i^*-\theta_i)\left(T(Y_{n, i})-\E[T(Y_{n, i})]\right) \\ 
	&=&  d(\theta^*_i || \theta_i) +  (\theta_i^* - \theta_i)\xi_n
	\end{eqnarray*}
	
	Therefore, 
	\begin{eqnarray*}
	 \Lambda_{n,i}(\theta_i^* || \theta_i)   - \Psi_{n,i} d(\theta_i^* || \theta_i) &=& \sum_{\ell=1}^{n} X_\ell \log\left( \frac{p(Y_{\ell,i}| \theta^*_{i})}{ p(Y_{\ell,i}| \theta_{i})}\right) - \sum_{\ell=1}^{n} \psi_{\ell,i}d(\theta^*_i || \theta_i) \\
	 &= & \sum_{\ell=1}^{n} (X_\ell - \psi_{\ell,i})d(\theta^*_i || \theta_i) + \sum_{\ell=1}^{n} X_\ell \xi_\ell (\theta_i^* - \theta_i). 
	\end{eqnarray*}
	Here  $|\theta^*_i - \theta_i|\leq \overline{\theta}-\underline{\theta} \equiv C_2$ is bounded uniformly. Similarly, as shown in Appendix \ref{section: appendix preliminaries}, $d(\theta^*_i || \theta_i)$ is bounded uniformly in $\theta_i$ by
	\[
	C_1 \equiv  \max_{ \theta' \in [\underline{\theta}, \overline{\theta}]} d(\theta^*_i || \theta'_i) < \infty.
	\]
	This implies, 
	\begin{eqnarray}  
	|\Lambda_{n,i}(\theta_i)  - \Psi_{n,i} d(\theta_i^* || \theta_i)| & \leq & C_1 \left| \sum_{\ell=1}^{n} (X_\ell - \psi_{\ell,i}) \right| + C_2 \left| \sum_{\ell=1}^{n} X_\ell \xi_\ell\right| \\ \label{eq: bound on log-likelihood}
	|\Lambda_{n,i}(\theta_i) | &\leq& C_1\Psi_{n,i} +  C_1 \left| \sum_{\ell=1}^{n} (X_\ell - \psi_{\ell,i}) \right| + C_2 \left| \sum_{\ell=1}^{n} X_\ell \xi_\ell\right|.
   \end{eqnarray}
   Since $\E[\xi_n^2]<\infty$, the result then follows by applying Lemma \ref{lem: adaptive LLN} and Corollary \ref{cor: convergence of action selection probabilities to the empirical distribution}. In particular, when $\Psi_{n,i} \to \infty$, 
   \[
   \lim_{n\to \infty } \Psi_{n,i}^{-1} \sum_{\ell=1}^{n} (X_\ell - \psi_{\ell,i}) =0 \quad \text{and} \quad \lim_{n\to \infty } \Psi_{n,i}^{-1}\sum_{\ell=1}^{n} X_\ell \xi_\ell =0
   \]
   When $\lim_{n\to \infty} \Psi_{n,i} <\infty$, 
   \[ 
   \sup_{n \in \N} \left|\sum_{\ell=1}^{n} (X_\ell - \psi_{\ell,i})\right|<\infty  \quad \text{and} \quad  \sup_{n \in \N} \left|  \sum_{\ell=1}^{n} X_\ell \xi_\ell\right| < \infty.
   \]
It is also immediate that $d(\theta^*_i || \theta_i)\Psi_{n,i}  \leq C_1\Psi_{n,i} \nrightarrow \infty$, which by \eqref{eq: bound on log-likelihood} implies the second part of the result. 
\end{proof}

A corollary of the previous lemma relates the log-likelihood ratio $\Lambda_{n}(\thetabf^* || \thetabf)$ to the Kullback-Leibler divergence $D_{\overline{\psi}_n}(\thetabf^* || \thetabf)$. 
\begin{corollary}\label{cor: uniform convergence of the log posterior} With probability 1,
	\[
	\sup_{\thetabf\in \Theta}  \,\,| n^{-1} \Lambda_{n}(\thetabf^* || \thetabf)  - D_{\overline{\psibf}_n}(\thetabf^* || \thetabf) | \to 0
	\]
\end{corollary}
\begin{proof}
\begin{eqnarray*}
\left|n^{-1} \Lambda_{n}(\thetabf^* || \thetabf) - D_{\overline{\psi}_n}(\thetabf^* || \thetabf) \right| &=& \left|n^{-1} \sum_{i=1}^{k}\left( \Lambda_{n,i}(\theta_i^* || \theta_i)- \Psi_{n,i} d(\theta_i^* || \theta_i) \right) \right|  \\
&\leq & \sum_{i=1}^{k}  \, n^{-1}|\Lambda_{n,i}(\theta_i^* || \theta_i)- \Psi_{n,i} d(\theta_i^* || \theta_i)|.
\end{eqnarray*}
Lemma \ref{lem: uniform convergence of marginal log likelihood} implies 
\[
\sup_{\theta_i \in [\underline{\theta}, \overline{\theta} ]} n^{-1} |\Lambda_{n,i}(\theta_i^* || \theta_i)- \Psi_{n,i} d(\theta_i^* || \theta_i)| \to 0,
\]	
which completes the proof. 
\end{proof}

\subsection{Posterior Consistency: Proof of Prop.~\ref{prop: posterior consistency}}\label{subsec: posterior consistency} 

\begin{propn}[\ref{prop: posterior consistency}]
	For any $i \in \{1,..,k\}$ if $\Psi_{n,i} \to \infty$, then, for all $\epsilon>0$
	\[
	\Pi_{n}(\{ \thetabf \in \Theta | \theta_i \notin (\theta^*_{i}-\epsilon, \theta^*_{i}+\epsilon ) \} ) \to 0,
	\]
	with probability 1. If $\mathcal{I} = \{i \in \{1,...,k\} | \lim_{n\to \infty} \Psi_{n,i}<\infty \}$ is nonempty, then
	\[
	\inf_{n\in\mathbb{N}} \,\, \Pi_{n}(\{ \thetabf \in \Theta | \theta_i \in (\theta_{i}', \theta_{i}'') \,\, \forall i \in \mathcal{I} \} ) > 0
	\]
	for any collections of open intervals $
	(\theta_i', \theta_i'') \subset (\underline{\theta}, \overline{\theta})$ ranging over $i \in \mathcal{I}$.
\end{propn}
Because we don't assume an independent prior across the designs, $\Pi_1$ is not a product measure and therefore neither is $\Pi_n$. This makes it challenging to reason about the marginal posterior of each design, which is required for Proposition \ref{prop: posterior consistency}.  Thankfully, since the prior density is bounded, $\Pi_n$ behaves like a product measure. Note that the likelihood function can be written as the product of $k$ terms:
\[
L_{n}(\thetabf) = \prod_{i=1}^{k} L_{n,i}(\theta_i)
\]
where
\[
L_{n,i}(\theta_i) \triangleq \prod_{\substack{\ell \leq n \\I_{\ell=i}}}
p(Y_{\ell, 1} | \theta_i)
\]
with the convention that $L_{n,i}(\theta_i) =1$ when $\sum_{\ell=1}^{n} \mathbf{1}(I_\ell = i)=0$. Therefore $L_{n}(\thetabf)$ forms the density of a product measure. By normalizing, this induces a probability measure
over $\Theta$,
\[
\mathcal{L}_n(\tilde{\Theta}) \triangleq \frac{\intop_{\tilde{\Theta}} L_{n}(\thetabf)d\thetabf}{\intop_{\Theta} L_{n}(\thetabf)d\thetabf}  \qquad \tilde{\Theta} \subset \Theta,
\]
which, as we argue in the next lemma, behaves like the posterior $\Pi_{n}$.

\begin{lemma}\label{lem: posterior to product measure}
	For any set $\tilde{\Theta} \subset \Theta$,
	\[
	C^{-1} \mathcal{L}_{n}(\tilde{\Theta}) \leq  \Pi_{n+1}(\tilde{\Theta}) \leq C \mathcal{L}_{n}(\tilde{\Theta}),
	\]
	where
	\[
	C= \frac{\sup_{\thetabf \in \Theta} \pi_{1}(\thetabf)}{\inf_{\thetabf \in \Theta} \pi_{1}(\thetabf)} <\infty
	\]
	is independent of $n$ and $\tilde{\Theta}$. %In addition, for any interval $[\theta', \theta''] \subset [\underline{\theta}, \overline{\theta}]$ and $i\in \{1,..,k\}$,
	%\[
	% C^{-1}\frac{\intop_{\theta'}^{\theta''} L_{n,i}(\theta)d\theta}{\intop_{\underline{\theta}}^{\overline{\theta}} L_{n,i}(\theta)d\theta} \leq \Pi_{n}(\{ \thetabf \in \Theta | \theta_i \in (\theta', \theta'') \} ) \leq C\frac{\intop_{\theta'}^{\theta''} L_{n,i}(\theta)d\theta}{\intop_{\underline{\theta}}^{\overline{\theta}} L_{n,i}(\theta)d\theta}.
	%\]
\end{lemma}
\begin{proof}
	This follows immediately by bounding $\pi_1(\thetabf)$ from above and below in the relation
	\[
	\Pi_{n+1}(\tilde{\Theta}) = \frac{\intop_{\tilde{\Theta}} \pi_1(\thetabf) L_{n}(\thetabf)d\thetabf}{\intop_{\Theta} \pi_{1}(\thetabf) L_{n}(\thetabf)d\thetabf}.
	\]
\end{proof}

We can now prove Proposition \ref{prop: posterior consistency}.
\begin{proof}[Proof of Proposition \ref{prop: posterior consistency}]

	We begin with the first part of the result. For simplicity of notation, we focus on the upper interval $\tilde{\Theta} = \{ \thetabf \in \Theta : \theta_i > \theta_i^*+\epsilon  \}$, but results follow identically for the lower interval. We want to show
	$ \Pi_n(\tilde{\Theta})\to 0$, which occurs if and only if $ \mathcal{L}_{n}(\tilde{\Theta})\to 0$. Since $\mathcal{L}_{n}$ is a product measure, 
\begin{equation}\label{eq: likelihood integral of interval}
	\mathcal{L}_{n}(\tilde{\Theta})
	=  \frac{\intop_{\theta^*_{i}+\epsilon}^{\overline{\theta}} L_{n,i}(\theta)d\theta}{\intop_{\underline{\theta}}^{\overline{\theta}} L_{n,i}(\theta)d\theta}  = \frac{\intop_{\theta^*_{i}+\epsilon}^{\overline{\theta}} \left(L_{n,i}(\theta)/L_{n,i}(\theta^*_i)\right)d\theta}{\intop_{\underline{\theta}}^{\overline{\theta}} \left(L_{n,i}(\theta)/L_{n,i}(\theta^*_i)\right)d\theta}  =  \frac{\intop_{\theta^*_{i}+\epsilon}^{\overline{\theta}} \exp\{ -\Lambda_{n,i}(\theta^*_i || \theta) \}d\theta}{\intop_{\underline{\theta}}^{\overline{\theta}} \exp\{ -\Lambda_{n,i}(\theta^*_i || \theta)\}d\theta}
\end{equation}
where $\Lambda_{n,i}(\theta_i^* || \theta_i) = \log(L_{n,i}(\theta_i^*) / L_{n,i}(\theta_i)$. 
	By Lemma \ref{lem: uniform convergence of marginal log likelihood}, with probability 1 there is a sequence $a_n \to 0$ such that $| \Lambda_{n,i}(\theta^*_i || \theta) - \Psi_{n,i}d(\theta^*_i ||\theta )| \leq a_n$ for all $\theta$.  Then, for $b_n = e^{a_n}/ e^{-a_n} \to 1$, one has
\[ 
\mathcal{L}_{n}(\tilde{\Theta})
 \leq \frac{b_n \intop_{\theta^*_{i}+\epsilon}^{\overline{\theta}} \exp\{ -\Psi_{n,i}d(\theta^*_i || \theta) \}d\theta}{\intop_{\underline{\theta}}^{\overline{\theta}} \exp\{ -\Psi_{n,i}d(\theta^*_i || \theta) \}d\theta}  \leq \frac{b_n \intop_{\theta^*_{i}+\epsilon}^{\overline{\theta}} \exp\{ -\Psi_{n,i}d(\theta^*_i || \theta)  \}d\theta}{\intop_{\theta^*_i}^{\theta^*_i+\epsilon/2} \exp\{ -\Psi_{n,i}d(\theta^*_i || \theta) \}d\theta}.
\]	
The integral in the numerator is upper bounded by  $(\overline{\theta}-\theta^*_i - \epsilon)\exp\{ -\Psi_{n,i}d(\theta^*_i || \theta^*_{i}+\epsilon)$ while the integral in the denominator is lower bounded by $(\epsilon/2)\exp\{ -\Psi_{n,i}d(\theta^*_i || \theta^*_i+\epsilon/2) \}$. This shows
\[ 
\mathcal{L}_{n}(\tilde{\Theta}) \leq  c_0 b_n  \exp\{ -\Psi_{n,i}\left(d(\theta^*_i || \theta^*_{i}+\epsilon) -d(\theta^*_i || \theta^*_{i}+\epsilon/2)  \right)\} \to 0
\]
where $c_0 =2\epsilon^{-1}(\overline{\theta}-\theta^*_i - \epsilon)$. 

The second part of the claim follows from the
lower bound in Lemma \ref{lem: posterior to product measure} of
\begin{eqnarray}\label{eq: product measure upper bound}
\Pi_{n+1}(\{ \thetabf \in \Theta | \theta_i \in (\theta', \theta'') \,\, \forall i \in \mathcal{I} \} ) &\geq& C^{-1}\mathcal{L}_{n}(\{ \thetabf \in \Theta | \theta_i \in (\theta_{i}', \theta_{i}'') \,\, \forall i \in \mathcal{I} \}) \\
&= &  C^{-1} \prod_{i \in \mathcal{I}} \mathcal{L}_{n}(\{ \thetabf \in \Theta | \theta_i \in (\theta_{i}', \theta_{i}'')  \}).
\end{eqnarray}
As in \eqref{eq: likelihood integral of interval}, 
\[ 
\mathcal{L}_{n}(\{ \thetabf \in \Theta | \theta_i \in (\theta_{i}', \theta_{i}'')  \}) = \frac{\intop_{\theta'}^{\theta''} \exp\{ -\Lambda_{n,i}(\theta^*_i || \theta) \}d\theta}{\intop_{\underline{\theta}}^{\overline{\theta}} \exp\{ -\Lambda_{n,i}(\theta^*_i || \theta)\}d\theta}.
\]
When $\lim_{n\to \infty }\Psi_{n,i} < \infty$, Lemma \ref{lem: uniform convergence of marginal log likelihood} shows that for each $i\in \mathcal{I}$,
\[ 
\sup_{\theta_i \in [\underline{\theta}, \overline{\theta}]} \sup_{n\in \mathbb{N}} \left| \Lambda_{n,i}(\theta_i^* || \theta_i) \right| < \infty.
\] 
 This implies 
 \[
 \inf_{n} \mathcal{L}_{n}(\{ \thetabf \in \Theta | \theta_i \in (\theta_{i}', \theta_{i}'')  \}) >0
 \]
 and establishes the claim. 

\end{proof}

 \subsection{Large Deviations: Proof of Proposition \ref{prop: posterior concentration}}

All statements in this section hold when observations are drawn under the parameter $\thetabf^*$. Since $\thetabf^*$ is fixed throughout, we simplify notation and write
\[
W_n(\thetabf) \triangleq D_{\overline{\psibf}_{n}}(\thetabf^* || \thetabf).
\]
Note that $W_{n}(\thetabf^*) = 0$. As shown in the next lemma $n^{-1} \log\left(\pi_{n}(\thetabf) / \pi_{n}(\thetabf^*) \right)-  W_{n}(\thetabf)  \to 0$ uniformly in $\thetabf$. 
\begin{lemma}With probability 1,
\[
\sup_{\thetabf \in \Theta}  n^{-1}\left|\log\left( \frac{\pi_{n}(\thetabf^*)}{ \pi_{n}(\thetabf)} \right)-  W_{n}(\thetabf) \right| \to 0.
\]
\end{lemma}
\begin{proof} We have
	\[
	\log\left( \frac{\pi_{n}(\thetabf^*)}{ \pi_{n}(\thetabf)} \right) -  W_{n}(\thetabf) = \log\left( \frac{\pi_{1}(\thetabf^*)}{ \pi_{1}(\thetabf)} \right)+ \left(\Lambda_{n-1}(\thetabf^* || \thetabf)-W_{n-1}(\thetabf) \right) + \left( W_{n-1}(\thetabf) - W_{n}(\thetabf) \right).
	\]
	Since $\inf_{\thetabf \in \Theta}\pi_{1}(\thetabf)>0$ and $\sup_{\thetabf \in \Theta} \pi_{1}(\thetabf) < \infty$, 
	 $n^{-1} \log\left(\pi_{1}(\thetabf) / \pi_{1}(\thetabf^*) \right) \to 0$  uniformly in $\thetabf$. By  Corollary \ref{cor: uniform convergence of the log posterior}, $n^{-1}\left(\Lambda_{n-1}(\thetabf)-W_{n-1}(\thetabf) \right) \to 0$ uniformly as well. Finally, by equation \eqref{eq: uniform bound on A and KL},  $n^{-1}(W_{n}(\thetabf)-W_{n-1}(\thetabf)) \leq n^{-1}\max_{i}d(\theta^*_i || \theta_i) \to 0$ uniformly in $\thetabf$. 	 
\end{proof}
The remaining proof of Proposition \ref{prop: posterior concentration} follows from a sequence of lemmas. The next observes a form of uniform continuity of $W_n$ that follows from the uniform bound on $A'(\theta)$ in Assumption \ref{assumption: main}. 
\begin{lemma}\label{lem: uniform continuity of W}
	For all $\epsilon > 0$, there exists $\delta>0$ such that for $\thetabf, \thetabf' \in \Theta$
	\[
	\| \thetabf - \thetabf' \|_{\infty} \leq \delta \implies  \underset{n\in\N}{\sup} |W_{n}(\thetabf)- W_{n}(\thetabf')| \leq \epsilon .
	\]
\end{lemma}
\begin{proof}
	We have that
	\begin{eqnarray*}
		|W_{n}(\thetabf)- W_{n}(\thetabf')| &\leq& \max_{1\leq i \leq k} | d(\theta^*_i || \theta_i) - d(\theta^*_i || \theta'_i) | \\
		&=& \max_{1\leq i \leq k} \left| (\theta'_i - \theta_i)A'(\theta^*_i) + A(\theta_i) - A(\theta'_i) \right| \\
		&\leq& 2C\delta
	\end{eqnarray*}
	where $C= \sup_{\theta\in (\underline{\theta}, \overline{\theta})}|A'(\theta)|< \infty$.
\end{proof}

\begin{lemma}For any open set $\tilde{\Theta} \subset \Theta$,
\[
\intop_{\thetabf \in \tilde{\Theta}} \frac{\pi_{n}(\thetabf)}{\pi_{n}(\thetabf^*)} d\thetabf \doteq  \intop_{\thetabf \in \tilde{\Theta}} \exp\{-n W_{n}(\thetabf) \} d\thetabf.
\]
\end{lemma}
\begin{proof}
By Corollary \ref{cor: uniform convergence of the log posterior}, we can fix a sequence $\epsilon_n \geq 0 $ with $\epsilon_n \to 0$ such that,
\[
\exp\{-n (W_{n}(\thetabf) +\epsilon_n) \}  \leq  \frac{\pi_{n}(\thetabf)}{\pi_{n}(\thetabf^*)} \leq \exp\{-n (W_{n}(\thetabf) -\epsilon_n) \}.
\]
Integrating over $\tilde{\Theta}$ yields,
\[
\exp\{-n \epsilon_n \} \intop_{ \tilde{\Theta}} \exp\{-n W_{n}(\thetabf) \} d\thetabf
\leq \intop_{ \tilde{\Theta}} \frac{\pi_{n}(\thetabf)}{\pi_{n}(\thetabf^*)} d\thetabf \leq \exp\{n \epsilon_n \} \intop_{ \tilde{\Theta}} \exp\{-n W_{n}(\thetabf) \} d\thetabf.
\]
Taking the logarithm of each side implies
\[
\frac{1}{n}\left|\log\intop_{ \tilde{\Theta}} \frac{\pi_{n}(\thetabf)}{\pi_{n}(\thetabf^*)} d\thetabf  -  \log \intop_{ \tilde{\Theta}} \exp\{-n W_{n}(\thetabf) \} d\thetabf  \right| \leq \epsilon_n \to 0.
\]
\end{proof}

\begin{lemma}
For any open set $\tilde{\Theta} \subset \Theta$,
\[
\intop_{\thetabf \in \tilde{\Theta}} \exp\{-n W_{n}(\thetabf) \} d\thetabf \doteq \exp\{ -n \inf_{\thetabf \in \tilde{\Theta}} W_{n}(\thetabf) \}
\]
\end{lemma}
\begin{proof}
Let $\hat{\thetabf}_n$ be a point in the closure of $\tilde{\Theta}$, satisfying
\[
W_{n}(\hat{\thetabf}_n) = \inf_{\thetabf \in \tilde{\Theta}} W_{n}(\thetabf).
\]
Such a point always exists, since $W_{n}$ is continuous, and the closure of $\tilde{\Theta}$ is compact. Let
\[
\gamma_n \triangleq \intop_{\thetabf \in \tilde{\Theta}} \exp\{-n W_{n}(\thetabf) \} d\thetabf.
\]
Our goal is to show
\[
\frac{1}{n}\log(\gamma_n) + W_n(\hat{\thetabf}_n) \to 0.
\]
We have
\[
\gamma_n \leq {\rm Vol}(\tilde{\Theta}) \exp\{ -n W_{n}(\hat\thetabf_n )  \}
\]
where for any $\Theta' \subset \Theta$, ${\rm Vol}(\Theta') = \intop_{\tilde{\Theta}} d\thetabf \in (0,\infty)$ denotes the volume of $\Theta$. This shows
\[
\underset{n \to \infty}{\lim\sup} \left( \frac{1}{n}\log(\gamma_n) + W_n(\hat{\thetabf}_n) \right) \leq 0.
\]
We now show the reverse. Fix an arbitrary $\epsilon>0$. By Lemma \ref{lem: uniform continuity of W}, there exists $\delta>0$ such that
\[
|W_{n}(\thetabf)- W_{n}(\hat\thetabf_n)| \leq \epsilon  \qquad \forall n \in \N
\]
for any $\thetabf  \in \Theta$ with
\[
\| \thetabf - \hat\thetabf_n \|_{\infty} \leq \delta.
\]
Now, choose a finite $\delta$--cover $O$ of $\tilde{\Theta}$ in the norm $\| \cdot \|_{\infty}$. Remove any set in $O$ that does not intersect $\tilde{\Theta}$. Then, for each $o\in O$,
\[
{\rm Vol}(o \cap \tilde{\Theta}) > 0 \implies C_{\delta} \triangleq \min_{o \in O} {\rm Vol}(o \cap \tilde{\Theta}) > 0.
\]
Choose $o_n \in O$ with $\hat\thetabf_n \in {\rm closure}(o_n)$. Then, for every $\thetabf \in o_n$, $W_{n}(\thetabf) \leq W_{n}(\hat\thetabf_n)+ \epsilon$. This shows
\[
\gamma_n \geq \intop_{o} \exp\{ -n W_{n}(\thetabf  \}d\thetabf \geq C_{\delta}\exp\{-n(W_{n}(\hat\thetabf_n)-\epsilon)).
\]
Taking the logarithm of both sides implies
\[
\frac{1}{n}\log(\gamma_n) + W_{n}(\hat\thetabf_n)  \geq \frac{C_\delta}{n} - \epsilon \to -\epsilon.
\]
Since $\epsilon$ was chosen arbitrarily, this shows
\[
\underset{n \to \infty}{\lim\inf} \left( \frac{1}{n}\log(\gamma_n) + W_n(\hat{\thetabf}_n) \right) \geq 0,
\]
and completes the proof.
\end{proof}

We now complete the proof of Proposition \ref{prop: posterior concentration}.

\begin{proof}[Proof of Proposition \ref{prop: posterior concentration}]
	We begin with a simple observation. For any sequences of real numbers $\{a_n\}, \{b_n\}$, and $\{\tilde{a}_n\}, \{\tilde{b}_n\}$, if $a_n \doteq \tilde{a}_n$ and $b_n \doteq \tilde{b}_n \in \mathbb{R}$, then $a_n/b_n \doteq \tilde{a}_n/\tilde{b}_n$. 
	
	Therefore, we have 
\[
\Pi_{n}(\tilde{\Theta}) = \frac{\Pi_{n}(\tilde{\Theta})}{\Pi_{n}(\Theta)} 
= \frac{\intop_{\tilde{\Theta}}{\pi_{n}(\thetabf) d\thetabf}}{\intop_{\Theta}{\pi_{n}(\thetabf) d\thetabf}}
= \frac{\intop_{\tilde{\Theta}}{(\pi_{n}(\thetabf)/\pi_{n}(\thetabf^*)) d\thetabf}}{\intop_{\Theta}{(\pi_{n}(\thetabf)/\pi_{n}(\thetabf^*)) d\thetabf}} \doteq \frac{\exp\{ -n \inf_{\thetabf \in \tilde{\Theta}} W_{n}(\thetabf) \}}{\exp\{ -n \inf_{\thetabf \in \Theta} W_{n}(\thetabf) \}}
\]
where the final equality follows from the previous two lemmas. Since $W_{n}(\thetabf) \geq 0$ and $W_{n}(\thetabf^*) = 0$, $\exp\{ -n \inf_{\thetabf \in \Theta} W_{n}(\thetabf)\} =1$. 
	 
\end{proof}

\subsection{Large Deviations of the Value Measure: Proof of Lemma \ref{lem: value and probability are log equivalent}}\label{subsec: asymptotics of value}
\begin{lemman}[\ref{lem: value and probability are log equivalent}]
	For any $i\neq I^*$, $V_{n,i} \doteq \alpha_{n,i}$.
\end{lemman}
\begin{proof}
First, since
\[
V_{n,i} = \intop_{\Theta_i} v_{i}(\thetabf)\pi_{n}(\thetabf)d\thetabf \leq (u(\overline{\theta})-u(\underline{\theta})) \intop_{\Theta_i} \pi_{n}(\thetabf)d\thetabf = (u(\overline{\theta})-u(\underline{\theta}))\alpha_{n,i}
\]
it is immediate that
\begin{equation}\label{eq: limsup is non positive}
\underset{n\to\infty}{\lim\sup} \,\, n^{-1}(\log V_{n,i} - \log \alpha_{n,i}) \leq 0.
\end{equation}
The other direction is more subtle. Define $\Theta_{i, \delta} \subset \Theta_i$ by
\[
\Theta_{i, \delta} = \{ \thetabf \in \Theta : \theta_i \geq \max_{j\neq i}\theta_j + \delta\}.
\]
For any $\thetabf \in \Theta_{i,\delta}$, $v_{i}(\thetabf) \geq C_{\delta}$ where
\[
C_{\delta} \equiv \min_{\theta \in [\underline{\theta}, \overline{\theta}]} u(\theta + \delta) - u(\theta) > 0.
\]
Because $u(\theta + \delta) - u(\theta)$ is continuous and is strictly positive for each $\theta$, this minimum exists and the objective value is strictly positive. Then
\[
V_{n,i} \geq \intop_{\Theta_{i,\delta}} v_{i}(\thetabf) \pi_{n}(\thetabf) d\thetabf \geq C_{\delta} \intop_{\Theta_{i,\delta}} \pi_{n}(\thetabf) d\thetabf = C_{\delta}\Pi_{n}(\Theta_{i,\delta}) \qquad \forall \delta>0.
\]
Combining this with Proposition \ref{prop: posterior concentration} shows
\[
\underset{n\to \infty}{\lim\inf}\, \frac{1}{n}(\log V_{n,i} - \log \alpha_{n,i}) \geq \underset{n\to \infty}{\lim\inf}\, \frac{1}{n}(\log \Pi_{n}(\Theta_{i,\delta}) - \log \Pi_{n}(\Theta_{i}) ) = - \min_{\thetabf \in \Theta_{i,\delta}} D_{\overline{\psi}_n}(\thetabf^* || \thetabf) - \min_{\thetabf \in \Theta_{i}} D_{\overline{\psi}_n}(\thetabf^* || \thetabf).
\]
The final term can be made arbitrarily small by taking $\delta\to 0$. Precisely, by Lemma \ref{lem: uniform continuity of W}, for any $\epsilon > 0$, there exists $\delta>0$ such that for all $n \in \mathbb{N}$ and $\thetabf, \thetabf' \in \Theta$ satisfying $\| \thetabf - \thetabf'\|_{\infty} \leq \delta$ ,
\[
 D_{\overline{\psi}_n}(\thetabf^* || \thetabf) \leq \epsilon.
\]
Therefore, for each $\epsilon>0$ one can choose $\delta>0$ such that
\[
\min_{\thetabf \in \Theta_{i,\delta}} D_{\overline{\psi}_n}(\thetabf^* || \thetabf) \leq \min_{\thetabf \in \Theta_{i}} D_{\overline{\psi}_n}(\thetabf^* || \thetabf) + \epsilon.
\]
This shows $\lim\inf \,n^{-1}(\log V_{n,i} - \log \alpha_{n,i}) \geq -\epsilon$ for all $\epsilon>0$, and hence 
\[
\underset{n\to\infty}{\lim\inf} \,\,n^{-1}(\log V_{n,i} - \log \alpha_{n,i}) \geq 0.
\]
\end{proof}

\section{Simplifying and Bounding the Error Exponent}
\label{sec: appendix error exponent}
\subsection{Proof of Lemma \ref{lem: solves minimum over theta bar}}
To begin, we restate the results of Lemma \ref{lem: solves minimum over theta bar} in the order in which they will be proved. Recall, from  Section \ref{section: appendix preliminaries} that $A(\theta)$ is increasing and strictly convex, and, by \eqref{eq: derivative is mean}, $A'(\theta)$ is the mean observation under $\theta$.
\begin{lemman}[\ref{lem: solves minimum over theta bar}] Define for each $i\neq I^*$,$\psi \geq 0$,
\begin{equation}\label{eq: repeat definition of C}
C_{i}(\beta, \psi) \triangleq  \min_{x\in \mathbb{R}} \, \beta d(\theta^*_{I^*}|| x) + \psi d(\theta^*_i || x).
\end{equation}
 \\
 {\bf (a)} For any $i\neq I^*$ and probability distribution $\psibf$ over $\{1,...,k\}$
\[\min_{\thetabf \in \overline{\Theta}_{i}} D_{\psibf}(\thetabf^* || \thetabf) = C_{i}(\psi_{I^*}, \psi_i).\]
where $\overline{\Theta}_{i} \triangleq  \left\{ \thetabf \in \Theta | \theta_i \geq \theta_{I^*}\right\}$. 
\\{\bf (b)} Each $C_i$ is a concave function.
  \\{\bf (c)} The unique solution to the minimization problem \eqref{eq: repeat definition of C} is $\overline{\theta}\in \mathbb{R}$ satisfying
\[
A'(\overline{\theta}) = \frac{\psi_{I^*} A'(\theta^*_{I^*}) + \psi_{i}A'(\theta^*_{i})}{\psi_{I^*}+\psi_i}.
\]
Therefore,
\[
C_{i}(\psi_{I^*}, \psi_i) = \psi_{I^*}d(\theta^*_{I^*}||\overline{\theta}) + \psi_{i}d(\theta^*_{i}||\overline{\theta}).
\]
{\bf (d)} Each $C_i$ is a strictly increasing function.
\end{lemman}
\begin{proof} {\bf (a)}
\begin{eqnarray*}
\min_{\thetabf \in \overline{\Theta}_i} D_{\psibf}(\thetabf^* ||\thetabf ) &=& \min_{\thetabf \in \Theta: \theta_i \geq \theta_{I^*}} \sum_{j=1}^{k} \psi_{n,j} d(\theta^*_j ||\theta_j ) \\
&=& \min_{\overline{\theta} \geq \theta_i \geq \theta_{I^*} \geq \underline{\theta}} \, \psi_{I^*} d(\theta^*_{I^*} ||\theta_{I^*} )+\psi_{i} d(\theta^*_i ||\theta_i ) + \sum_{j \notin \{i, I^*\}}\min_{\theta_j} \psi_{n,j} d(\theta^*_j ||\theta_j ) \\
& = & \min_{\overline{\theta} \geq \theta_i \geq \theta_{I^*} \geq \underline{\theta}} \, \psi_{I^*} d(\theta^*_{I^*} ||\theta_{I^*} )+\psi_{i} d(\theta^*_i ||\theta_i )
\end{eqnarray*}
where the last equality uses that the minimum occurs when $\theta_j = \theta^*_j$ for $j\notin \{I^*, i\}$, and this is feasible for any choice of $(\theta_i, \theta_{I^*})$. Then, by the monotonicity properties of KL-divergence (see Section \ref{section: appendix preliminaries}, equation \eqref{eq: KL monotone}), there is always a minimum with $\theta_i = \theta_{I^*}$. Therefore this objective value is equal to
\[
\min_{\theta \in [\underline{\theta}, \overline{\theta}]} \psi_{I^*} d(\theta^*_{I^*} ||\theta )+\psi_{i} d(\theta^*_i ||\theta ) = \min_{x \in \mathbb{R}} \psi_{I^*} d(\theta^*_{I^*} ||x )+\psi_{i} d(\theta^*_i ||x ) = C_{i}(\psi_{I^*}, \psi_{i}).
\]
\\
\\ {\bf (b)} $C_i$ is the minimum over a family of linear functions and therefore is concave (See Chapter 3.2 of \citet{boyd2004convex}). In particular $C_i(\beta, \psi) = \min_{x\in \mathbb{R}} g((\beta, \psi) ; x)$ where $g((\beta, \psi) ; x)= \beta d(\theta^*_{I^*}|| x) + \psi d(\theta^*_i || x) $ is linear in $(\beta, \psi)$.
\\
\\{\bf (c)} Direct calculation using the formula for KL divergence in exponential families (see \eqref{eq: KL for exponential families} in Section \ref{section: appendix preliminaries}) shows
\[
\beta d(\theta^*_{I^*} || x) + \psi_i d(\theta^*_{i} || x) = (\beta+\psi_i)A(x) - (\beta A'(\theta^*_{I^*})+\psi_i A'(\theta^*_{i}))x + f(\beta, \theta^*_{I^*}, \psi_i, \theta^*_{i})
\]
where $f(\beta, \theta^*_{I^*}, \psi_i, \theta^*_{i})$ captures terms that are independent of $x$. Setting the derivative with respect to $x$ to zero yields the result since $A(x)$ is strictly convex.\\
\\{\bf (d)}  We will show $C_{i}$ is strictly increasing in the second argument. The proof that it is strictly increasing in its first argument follows by symmetry.  Set
\[
f(\psi_i, x) = \beta d(\theta^*_{I^*}|| x) + \psi_i d(\theta^*_{i}|| x)
\]
so that $C_{i}(\beta, \gamma_i)= \min_{x\in \mathbb{R}} f(\psi_i, x)$. Since KL divergences are non-negative, $f(\psi_i, x)$ is weakly increasing in $\psi_i$.  To establish the claim, fix two nonnegative numbers $\psi'< \psi'' $. Let $x'= \arg\min_{x} f(\psi', x)$ and $x''=\arg\min_x f(\psi'', x)$. By part (c), these are unique and $x' < x''$. Then
\[
f(\psi', x') < f(\psi', x'') \leq f(\psi'', x'')
\]
where the first inequality uses that $x'\neq x''$ and $x'$ is a unique minimum and the second uses the $f$ is non-decreasing.
\end{proof}

\subsection{Proof of Proposition \ref{prop: optimal constrained allocation}}
We will begin by restating Proposition \ref{prop: optimal constrained allocation}.
\begin{propn}[\ref{prop: optimal constrained allocation}]
The solution to the optimization problem \eqref{eq: optimal constrained error exponent} is the unique allocation $\psibf^*$  satisfying $\psi^*_{I^*}=\beta$ and
\begin{equation}\label{eq: equal effort}
C_{i}(\beta , \psi_i) = C_{j}(\beta, \psi_j) \qquad \forall \, i,j \neq I^*.
\end{equation}
If $\psibf_n = \psibf^*$ for all $n$, then
\[
\Pi_{n}(\Theta_{I^*}^{c}) \doteq \exp\{ -n \Gamma^*_{\beta} \}.
\]
Moreover under any other adaptive allocation rule, if $\overline{\psi}_{n, I^*}\to \beta$ as $n\to \infty$ then
\[
\underset{n\to \infty}{\lim \sup}  \, -\frac{1}{n} \log \Pi_{n}(\Theta_{I^*}^{c}) \leq \Gamma^*_\beta
\]
almost surely.
\end{propn}
\begin{proof}
By Lemma \ref{lem: solves minimum over theta bar}, each function $C_{i}$ is continuous, and therefore $\min_{i\neq I^*} C_{i}(\beta, \psi_i)$ is continuous in $(\psi_{i}: i \neq I^*)$. Since continuous functions on a compact space attain their minimum, there exists an optimal solution $\psibf^*$ to \eqref{eq: optimal constrained error exponent}, which satisfies
\[
\min_{ i \neq I^*} C_{i}(\beta, \psi^*_i) = \max_{ \psibf : \psi_{I^*}=\beta  }\min_{i \neq I^*} C_{i}(\beta, \psi_i).
\]
Suppose $\psibf^*$ does not satisfy \eqref{eq: equal effort}, so for some $j\neq I^*$,
\[
C_{j}(\beta, \psi^*_j) > \min_{i \neq I^*} C_{i}(\beta, \psi^*_i).
\]
This yields a contradiction. Consider a new vector $\psibf^{\epsilon}$ with $\psibf^{\epsilon}_{j} = \psi^{*}_j - \epsilon$ and $\psi_{i}^{\epsilon} = \psi_{i}^*+ \epsilon/(k-2)$ for each $i \notin \{I^*, j\}$. For sufficiently small $\epsilon$, one has
\[
C_{j}(\beta, \psi^{\epsilon}_j) > \min_{i \neq I^*} C_{i}(\beta, \psi^{\epsilon}_i) > \min_{i \neq I^*} C_{i}( \beta ,\psi^*_i)
\]
and so $\psibf^{\epsilon}$ attains a higher objective value. To show the solution to \eqref{eq: equal effort} must be unique, imagine $\psibf$ and $\psibf'$ both satisfy \eqref{eq: equal effort} and $\psi_{I^*}=\psi'_{I^*}=\beta$. If $\psi_j > \psi'_j$ for some $j$, then $C_{j}(\beta, \psi_j) > C_{j}(\beta, \psi'_j)$ since $C_j$ is strictly increasing. But by \eqref{eq: equal effort} this implies that $C_{j}(\beta, \psi_j) > C_{j}(\beta, \psi'_j)$ for {\emph every} $j \neq I^*$, which implies $\psi_j > \psi'_j$ for every $j$, and contradicts that that $\sum_{j\neq I^*} \psi_j = \sum_{j\neq I^*} \psi'_j = 1-\beta$.

The remaining claims follow immediately from Propoosition \ref{prop: posterior concentration} and Lemma \ref{lem: solves minimum over theta bar}, which together show that under any adaptive allocation rule
\[
\Pi_{n}(\Theta_{I^*}^{c}) \doteq \exp\{ -n \min_{i \neq I^*} C_{i}(\overline{\psi}_{n,I^*}, \overline{\psi}_{n,i}) \}.
\]
This implies that if $\overline{\psibf}_{n} = \psibf^*$ for all $n$, then $\Pi_{n}(\Theta_{I^*}^{c}) \doteq  \exp\{-n \Gamma^*_\beta\}$. Similarly, by the continuity of each $C_i$, if $\overline{\psi}_{n,I^*} \to \beta$, then 
\[
\Pi_{n}(\Theta_{I^*}^{c}) \doteq \exp\{ -n \min_{i \neq I^*} C_{i}(\beta, \overline{\psi}_{n,i}) \} \geq \exp\{ -n \Gamma^*_{\beta} \}
\]
which establishes the final claim. 
\end{proof}

\subsection{Proof of Lemma \ref{lem: relating the constrained exponent to the unconstrained}}
Recall, the notation
\[
\Gamma^* = \max_{\psibf} \min_{i\neq I^*} C_{i}(\psi_{I^*}, \psi_i) \qquad \Gamma^*_{\beta} \triangleq \max_{\psibf: \psi_{I^*}=\beta} \min_{i\neq I^*} C_{i}(\beta, \psi_i)
\]
where
\begin{equation*}
C_{i}(\beta, \psi)  = \min_{x\in \mathbb{R}} \,\beta d(\theta^*_{I^*} || x) + \psi d(\theta^*_{i^*} || x).
\end{equation*}
\begin{lemman}[\ref{lem: relating the constrained exponent to the unconstrained}] For $\beta^* = \arg\max_{\beta} \Gamma^*_{\beta}$ and any $\beta \in (0,1)$,
	\[
	\frac{\Gamma^*}{\Gamma^*_{\beta}}  \leq \max\left\{\frac{\beta^*}{\beta}, \frac{1-\beta^*}{1-\beta} \right\}.
	\]
	Therefore $\Gamma^* \leq 2\Gamma^*_{1/2}$
\end{lemman}

\begin{proof}
Define for each non-negative vector $\psibf$,
\[
f(\psibf) = \min_{i \neq I^*} C_{i}(\psi_{I^*}, \psi_i)\]
The optimal exponent $\Gamma^*$ is the maximum of  $f(\psibf)$ over probability vectors $\psibf$. Here, we instead define $f$ for all non-negative vectors, and proceed by varying the total budget of measurement effort available $\sum_{i=1}^{k} \psi_i$.

Because each $C_{i}$ is non-decreasing (see Lemma \ref{lem: solves minimum over theta bar}), $f$ is non-decreasing. Since the minimum over $x$ in the definition of $C_i$ only depends on the relative size of the components of $\psibf$, $f$ is homogenous of degree 1. That is $f(c\psibf) = cf(\psibf)$ for all $c\geq 1$. For each $c_1, c_2>0 $ define
\[
g(c_1, c_2) = \max\{f(\psibf) : \psi_{I^*}= c_1, \sum_{i\neq I^* } \psi_i \leq c_2, \psibf\geq 0\}.
\]
The function $g$ inherits key properties of $f$; it is also non-decreasing and homogenous of degree 1. We have
\begin{eqnarray*}
\Gamma^*_{\beta} &=& \max\{f(\psibf) : \psi_{I^*}=\beta, \, \sum_{i=1}^{k}\psi_i =1, \psibf \geq 0 \} \\
&=& \max\{f(\psibf) : \psi_{I^*}=\beta, \, \sum_{i\neq I^*}\psi_i \leq 1-\beta, \psibf \geq 0 \}\\
&=& g(\beta, 1-\beta)
\end{eqnarray*}
where the second equality uses that $f$ is non-decreasing. Similarly, $\Gamma^* =g(\beta^*, 1-\beta^*)$. Setting
\[
     r:= \max\left\{\frac{\beta^*}{\beta}, \frac{1-\beta^*}{1-\beta} \right\}
\]
implies $r\beta \geq \beta^*$ and $r(1-\beta)\geq 1-\beta^*$. Therefore
\[
r\Gamma^*_{\beta} = rg(\beta, 1-\beta) = g(r\beta, r(1-\beta)) \geq g(\beta^*, 1-\beta^*) = \Gamma^*.
\]

\end{proof}

\subsection{Sub-Gaussian Bound: Proof of Proposition \ref{prop: subgaussian bound}}
The proof of Proposition \ref{prop: subgaussian bound} relies on the following variational form of Kullback--Leibler divergence, which is given in Theorem 5.2.1 of Robert Gray's textbook \emph{Entropy and Information Theory} \cite{gray2011entropy}.
\begin{fact}\label{fact: variational defn of kl}
Fix two probability measures $\mathbf{P}$ and $\mathbf{Q}$ defined on a common measureable space $(\Omega, \mathcal{F}).$ Suppose that $\mathbf{P}$ is absolutely continuous with respect to $\mathbf{Q}$. Then
\[
D\left( \mathbf{P}  || \mathbf{Q} \right) = \sup_{X} \left\{ \E_{\mathbf{P}}[ X] - \log \E_{\mathbf{Q}} [e^{ X } ]\right\},
\]
where the supremum is taken over all random variables $X$ such that the expectation of $X$ under $\mathbf{P}$ is well defined, and $e^{ X }$ is integrable under $\mathbf{Q}$.
\end{fact}

When comparing two normal distributions $\mathcal{N}(\theta, \sigma^2)$ and $\mathcal{N}(\theta', \sigma^2)$ with common variance, the KL-divergence can be expressed as
$d(\theta || \theta') = (\theta-\theta')^2/ (2\sigma^2)$.
We follow \citet{russo2015controlling} in deriving the following corollary of Fact \ref{fact: variational defn of kl}, which provides and analogous lower bound on the KL-divergences when distributions are sub-Gaussian.
Recall that, $\mu(\theta)=\intop y p(y|\theta) d\nu(y)$ denotes the mean observation under $\theta$.
\begin{corollary}\label{cor: sub-Gaussian KL bound} Fix any $\theta, \theta' \in [\underline{\theta}, \overline{\theta}]$. If when $Y \sim p(y| \theta')$, $Y$ is sub-Gaussian with parameter $\sigma$, then,
\[
d(\theta || \theta') \geq \frac{\left(\mu(\theta)-\mu(\theta') \right)^2}{2 \sigma^2}
\]
\end{corollary}
\begin{proof}
Consider two alternate probability distributions for a random variable $Y$, one where $Y\sim p(y| \theta)$ and one where $Y\sim p(y| \theta')$ We apply Fact $\ref{fact: variational defn of kl}$ where $X= \lambda(Y- \E_{\theta'}[Y])$,
$\mathbf{P}$ is the probability measure when $Y\sim p(y|\theta)$ and $\mathbf{Q}$ is the measure when $Y\sim p(y| \theta') $.
By the sub-Gaussian assumption $
\log \E_{\theta'}\left[ \exp\{ X\} \right] \leq\lambda^2 \sigma^2/2.$
Therefore, Fact $\ref{fact: variational defn of kl}$ implies
\[
d(\theta || \theta') \geq \lambda(\E_{\theta}[X])  -  \frac{\lambda^2\sigma^2}{2} =  \lambda(\E_{\theta}[Y]-\E_{\theta'}[Y])  -  \frac{\lambda^2\sigma^2 }{2}.
\]
The result follows by choosing $\lambda= (\E_{\theta}[Y]-\E_{\theta'}[Y]) / \sigma^2$ which minimizes the right hand side.
\end{proof}

We are now ready to prove Proposition \ref{prop: subgaussian bound}. Recall that in an exponential family, $A'(\theta) = \intop T(y) p(y|\theta)d\nu(y)$, so if $T(y)=y$ then $A'(\theta) = \mu(\theta)$.
\begin{proof}[Proof of Proposition \ref{prop: subgaussian bound}]

By Lemma \ref{lem: solves minimum over theta bar},
\[
\Gamma^*_{1/2} = \max_{\psibf: \psi_{I^*}=1/2} \min_{i \neq I^*} C_{i}(1/2, \psi_i)
\]
Let $\mu_{I^*} = A'(\theta^*_{I^*})$ and $\mu_{i} = A'(\theta^*_{i})$ denote the means of designs $I^*$ and $i$ so $\Delta_i = \mu_{I^*}- \mu_i$. By Lemma \ref{lem: solves minimum over theta bar},
\[
C_{i}(1/2, \psi_i) =   (1/2) d(\theta^*_{I^*} || \overline{\theta}) + \psi_i d(\theta_i^* || \overline{\theta})    .
\]
where $\overline{\theta}$ is the unique parameter with mean
\[
A'(\overline{\theta}) = \frac{(1/2)\mu_{I^*}+\psi_i \mu_i}{1/2+\psi_i}.
\]
For $\psi_i \leq 1/2$,
\[
A'(\overline{\theta})\geq\frac{\mu_{I^*} + \mu_i }{2}  = \mu_i +\Delta_i / 2.
\]
Now, using Corollary \ref{cor: sub-Gaussian KL bound} and the non-negativity of KL-divergence
\[
C_{i}(1/2, \psi_i)  \geq \psi_i d(\theta_i^* || \overline{\theta}) \geq \frac{\psi_i (\mu_i - \mu_i + \Delta_i/2)^2 }{2\sigma^2} = \frac{\psi_i \Delta_i^2}{8 \sigma^2}.
\]

Choosing $\psi_{I^*}=1/2$, and $\psi_i \propto \Delta_i^{-2}$, so
\[
\psi_i=  \frac{1}{2}\left( \sum_{j-2}^{k} \Delta_j^{-2} \right)^{-1}\Delta_i^{-2}
\]
yields
\[
\min_{i\neq I^*} C_{i}(1/2, \psi_i) \geq \frac{1}{16\sigma^2 \sum_{2}^{k} \Delta_j^{-2} }.
\]
\end{proof}

%Consider any $\psibf$ with $\psi_i >0$ for all $i$, and let $\mu^* = \E_{\thetabf^*}[Y_{n,1}]$
%\[
%2\sigma^2 D_{\psibf}(\thetabf^* || \thetabf) \geq \min_{i  \neq 1} \frac{1}{2} d(\theta_1
%\]
\subsection{Convergence of Uniform Allocation: Proof of Proposition \ref{prop: convergence rate of uniform allocation}}
\begin{proof}
Without loss of generality, assume the problem is parameterized so that the mean of design $i$ is $\theta^*_i$
By Proposition \ref{prop: optimal allocation}, we have
\[
\Pi_{n}(\Theta_{I^*}^{c}) \doteq \exp\{ -n \min_{i\neq I^*} C_{i}( k^{-1}, k^{-1}) \}
\]
By Lemma \ref{lem: solves minimum over theta bar},
\[
C_{i}( k^{-1}, k^{-1}) = k^{-1}d(\theta^*_{I^*} || \overline{\theta})+k^{-1}d(\theta^*_{i} || \overline{\theta})
\]
where $\overline{\theta} = (\theta^*_{I^*} + \theta^*_i )/2$. Therefore, using the formula for the KL-divergence of standard Gaussian random variables
\[
C_{i}( k^{-1}, k^{-1}) = \frac{(\theta^*_{I^*} - \overline{\theta} )^2}{2\sigma^2} + \frac{(\theta^*_{i} - \overline{\theta} )^2}{2\sigma^2}= \frac{(\theta^*_{I^*}- \theta^*_i)^2}{4\sigma^2} = \frac{\Delta_i^2}{4\sigma^2}.
\]
\end{proof}

\section{Analysis of the Top-Two Allocation Rules: Proof of Proposition \ref{prop: TS converges to optimal allocation}} 
\label{sec: appendix proof of main result}
\begin{propn}[\ref{prop: TS converges to optimal allocation}]
Under the TTTS, TTPS, or TTVS algorithm with parameter $\beta>0$, $\overline{\psibf}_{n} \to \psibf^{\beta}$, where $\psibf^\beta$ is the unique allocation with $\psi^{\beta}_{I^*}=\beta$ satisfying
\begin{equation}
C_{i}(\beta, \psi^\beta_i)=C_{j}(\beta, \psi^\beta_j) \qquad  \forall i,j\neq I^*.
\end{equation}
Therefore, 
\begin{equation}\label{eq: optimal convergence rate}
\Pi_{n}(\Theta_{I^*}^{c}) \doteq e^{-n \Gamma^*_{\beta}}.
\end{equation}
\end{propn}
Because each $C_{i}$ is continuous, if $\overline{\psibf}_n \to \psibf^{\beta}$ then $C_{i}(\overline{\psi}_{n,I^*}, \overline{\psi}_{n,i})\to C_{i}(\beta, \psi^\beta_i)$ for all $i \neq I^*$. Equation \eqref{eq: optimal convergence rate} then follows  by invoking Proposition \ref{prop: optimal constrained allocation}, which establishes the optimality of the allocation $\psibf^\beta$.

The remainder of this section establishes that $\overline{\psibf}_n \to \psibf^{\beta}$ almost surely the proposed top-two rules. The proof is broken into a number of steps. In order to provide a nearly unified treatment of the three algorithms, we begin with several results that hold for any allocation rule.

\subsection{Results for a general allocation rule}
As in other sections, all arguments here hold for any sample path (up to a set of measure zero).  The first result provides a sufficient condition under which $\overline{\psibf}_{n} \to \psibf^{\beta}$. Roughly speaking, if $\overline{\psi}_{n,j} \geq \psi^{\beta}_{j}+\delta$, then too much measurement effort has been allocated to design $j$ relative to the optimal proportion $\psi^{\beta}_j$. Algorithms satisfying \eqref{eq: sufficient condition for optimality} allocate negligible measurement effort to such designs, and therefore the average measurement effort they receive must decrease toward the optimal proportion. 
\begin{lemma}[Sufficient condition for optimality]\label{lem: sufficient condition for optimality}
Consider any adaptive allocation rule. If $\overline{\psi}_{n, I^*} \to \beta$
and
\begin{equation}\label{eq: sufficient condition for optimality}
\sum_{n \in \mathbb{N}} \psi_{n,j} \mathbf{1}(\overline{\psi}_{n,j} \geq \psi^{\beta}_{j}+\delta) < \infty \qquad \forall \,\, j \neq I^*, \, \delta>0,
\end{equation}
then $\overline{\psibf}_{n}\to \psibf^{\beta}$.
\end{lemma}
\begin{proof}
Fix a sample path for which $\psi_{n,I^*}\to \beta$, and \eqref{eq: sufficient condition for optimality} holds. Fix some $j\neq I^*$. We first show
$ \underset{n \to \infty}{\lim\inf}\,\, \overline{\psi}_{n,j} \leq \psi_j^*.$ Suppose otherwise. Then, with positive probability, for some $\delta>0$, there exists $N$ such that for all $n\geq N$,  $\overline{\psi}_{n, j} \geq \psi_j^* + \delta.$ But then,
\[
\sum_{n \in \N} \psi_{n,j} = \sum_{n=1}^{N} \psi_{n,j} + \sum_{n=N+1}^{\infty} \mathbf{1}(\overline{\psi}_{n,j} \geq \psi_j^* + \delta )\psi_{n,j}<\infty.
\]
But since $\overline{\psi}_{n,j}= \sum_{\ell =1}^{n} \psi_{n,j}/n$ this implies $ \overline{\psi}_{n,j} \to 0.$

Now, we show $\underset{n \to \infty}{\lim\sup}\,\, \overline{\psi}_{n,j} \leq \psi_j^*.$ Proceeding by contradiction again, suppose otherwise. Then, with positive probability
\[
\underset{n \to \infty}{\lim\sup}\,\, \overline{\psi}_{n,j} > \psi_j^\beta  \qquad \& \qquad \underset{n \to \infty}{\lim\inf}\,\, \overline{\psi}_{n,j} \leq \psi_j^\beta.
\]
On any sample path where this occurs, for some $\delta >0$, there exists an infinite sequence of times  $N_1 < N_2 < N_3<...$ such that $\overline{\psi}_{N_{\ell}, j} \geq \psi^\beta_j +2\delta$ when $\ell$ is odd and $\overline{\psi}_{N_{\ell}, j} \leq \psi^\beta_j +\delta$ when $\ell$ is even.  This can only occur if,
\[
\sum_{n \in \N} \psi_{n,j} \mathbf{1}( \overline{\psi}_{n,j} \geq \psi_j^* + \delta) = \infty,
\]
which violates the hypothesis.

Together with the hypothesis that $\overline{\psi}_{n,I^*} \to \beta$, this implies that for all $i \in \{1,...,k\}$, $\underset{n\to\infty}{\lim \sup} \, \overline{\psi}_{n,i} \leq \psi^{\beta}_i$.  But since  $\sum_{i} \overline{\psi}_{n,i} = \sum_{i} \psi_{i}^{\beta}$, this implies $\overline{\psibf}_{n} \to \psibf^{\beta}$.
\end{proof}

The next lemma will be used to establish that \eqref{eq: sufficient condition for optimality} holds for each of the  proposed algorithms. It shows that if too much measurement effort has been allocated to some design $i\neq I^*$, in the sense that $\overline{\psi}_{n,i} > \psi^{\beta}_i + \delta$ for a constant $\delta>0$, then $\alpha_{n,i}$ is exponentially small compared $\max_{j \neq I^*} \alpha_{n,j}$.
\begin{lemma}[Over-allocation implies negligible probability]\label{lem: Over-allocation implies negligible probability}
Fix any $\delta>0$ and $j \neq I^*$. With probability 1, under any allocation rule, if $\overline{\psi}_{n,I^*} \to \beta$, there exists $\delta'>0$ and a sequence $\epsilon_n$ with $\epsilon_n \to 0$ such that for any $n\in \mathbb{N}$,
\[
\overline{\psi}_{n,j} \geq \psi^{\beta}_j + \delta \implies \frac{\alpha_{n,j}}{\max_{i \neq I^*} \alpha_{n,i}} \leq e^{-n(\delta' + \epsilon_n)}.
\]\end{lemma}
\begin{proof}
Since $\Pi_{n}(\Theta_{I^*}^{c})= \sum_{i\neq I^*} \alpha_{n,i}$, $\Pi_{n}(\Theta_{I^*}^{c}) \doteq \max_{i \neq I^*} \alpha_{n,i}$. Then, by invoking Proposition \eqref{prop: optimal constrained allocation}, since $\overline{\psibf}_{n, I^*} \to \beta$,
\[
\underset{n \to \infty}{\lim\sup} \,\, - \frac{1}{n} \log\left( \max_{i \neq I^*} \alpha_{n,i}\right) \leq \Gamma^*_{\beta}.
\]
Recall the definition $\overline{\Theta}_i \triangleq \{\thetabf | \theta_i \geq \theta_{I^*} \}$. Now, by Proposition \ref{prop: posterior concentration} and Lemma \ref{lem: solves minimum over theta bar},
\[
\alpha_{n,j} =\Pi_{n}(\Theta_{j})\leq \Pi_{n}(\overline{\Theta}_{j}) \doteq  \exp\{ - nC_{j}(\overline{\psi}_{n,I^*}, \overline{\psi}_{n,j} )\}
\doteq \exp\{ - nC_{j}(\beta, \overline{\psi}_{n,j} )\}.
\]
Combining these equations implies that there exists a non-negative sequence $\epsilon_n \to 0$ with
\[
\frac{\alpha_{n,j}}{\max_{i \neq I^*} \alpha_{n,i}} \leq \frac{\exp\{ - n(C_{j}(\beta, \overline{\psi}_{n,j})-\epsilon_n/2)\}}{\exp\{-n(\Gamma^*_{\beta}+\epsilon_n/2)\}} = \exp\left\{-n\left((C_{j}(\beta, \overline{\psi}_{n,j})-\Gamma^*_\beta) - \epsilon_n \right)  \right\}
\]
Since $C_{j}(\beta, \psi_j)$ is strictly increasing in $\psi_j$ (See lemma \ref{lem: solves minimum over theta bar}) and $C_{j}(\beta, \psi^{\beta}_j)=\Gamma^*_{\beta}$, there exists some $\delta'>0$ such that
\[
\overline{\psi}_{n,j} \geq \psi^{\beta}_j + \delta \implies  C_{j}(\beta, \overline{\psi}_{n,j}) - \Gamma^*_{\beta} > \delta'.
\]
\end{proof}
The next result builds on Proposition \ref{prop: posterior consistency}. It shows that the quality of any design which receives infinite measurement effort is identified to arbitrary precision. On the other hand, for designs receiving finite measurement effort, there is always nonzero probability under the posterior that one of them significantly exceeds the highest quality that has been confidently identified. Therefore, $\alpha_{n,i}$ and $V_{n,i}$ remain bounded away from 0 for designs that receive finite measurement effort. This result will be used to show that all designs receive infinite measurement effort under the proposed top-two allocation rules, and as a result the posterior converges on the truth asymptotically. 
\begin{lemma}[Implications of finite measurement] \label{lem: implications of infinite sampling}
Let \[
\mathcal{I}=\{i \in \{1,..,k\} : \sum_{n=1}^{\infty} \psi_{n,i} <\infty \}
 \]
denote the set of designs to which a finite amount of measurement effort is allocated. Then, for any $i \notin \mathcal{I}$
\begin{equation}\label{eq: TTS consistency of infinitely sampled actions}
\Pi_n\left(\{\thetabf :  \theta_i \in (\theta^*_i - \epsilon, \theta^*_i + \epsilon) \right) \to 1,
\end{equation}
and if $\mathcal{I}$ is empty
\[
V_{n,i} \to
\begin{cases}
0  \qquad \text{if } i\neq I^*\\
v_{I^*}(\thetabf^*) >0 \quad \text{if } i=I^*
\end{cases}
\qquad {\rm and} \qquad \alpha_{n,i} \to
\begin{cases}
0  \quad \text{if } i\neq I^*\\
1 \quad \text{if } i=I^*.
\end{cases}
\]
If $\mathcal{I}$ is nonempty,
then for every $i \in \mathcal{I},$
\begin{eqnarray*}
\underset{n\to \infty}{\lim\inf} \,\,\alpha_{n,i} > 0 \quad {\rm and} \quad
\underset{n\to \infty}{\lim\inf} \,\,V_{n,i} > 0.
\end{eqnarray*}
\end{lemma}
\begin{proof}
Equation \eqref{eq: TTS consistency of infinitely sampled actions} is implied by by Proposition \ref{prop: posterior consistency}. 
Now, set
\[
\Theta_{i,\epsilon} = \{\thetabf \in \Theta : \theta_i \geq \max_{j\neq i} \theta_j + \epsilon\}
\]
to be the set of parameters under which the quality of design $i$ exceeds that of all others by at least $\epsilon$. Let $\rho^* = \max_{i \notin \mathcal{I}} \theta^*_i$ denote the quality of the best design among those that are sampled infinitely often, and choose $\epsilon>0$ small enough that $\rho^* + 2\epsilon < \overline{\theta}$. For $i\in \mathcal{I}$, we have
\[
\Pi_{n}(\Theta_{i,\epsilon}) \geq \Pi_{n}(A) - \Pi_n\left(B \right)
\]
for
\[
A\equiv \{\thetabf | \theta_i \geq \rho^*+2\epsilon \,\, \& \,\, \theta_j < \rho^* \,\, \forall j\in \mathcal{I}\setminus\{i\} \}
\]
defined to be parameters under which $\theta_i \geq \rho^* + 2\epsilon$ but none of the other designs in $\mathcal{I}$ exceed $\rho^*$, and
\[
B \equiv \{\thetabf :  \max_{i \notin \mathcal{I}} \theta_i \geq \rho^* +\epsilon \}
\]
defined to be the parameter vectors under which there is no design in $\mathcal{I}^{c}$ with quality exceeding $\rho^*+\epsilon$. By \eqref{eq: TTS consistency of infinitely sampled actions},
\[
\Pi_n\left(B\right)\to 0,
\]
but by the second part of Proposition \ref{prop: posterior consistency}, the set of parameters $A$ cannot be completely ruled based on a finite amount of measurement effort, and
\[
\inf_{n \in \mathbb{N}} \Pi_{n}(A) > 0.
\]
Together this shows
\[
\underset{n\to \infty}{\lim\inf}\,\, \Pi_n(\Theta_{i,\epsilon})  > 0,
\]
which implies the result.

\end{proof}

\subsection{Results specific to the proposed algorithms}\label{subsec: results specific to proposed algorithms}
We now leverage the general results of the previous subsection to show $\overline{\psibf} \to \psibf^\beta$ under each proposed top-two allocation rule. Proofs are provided separately for each of the three algorithms, but they follow a similar structure. In the first step, we use Lemma \ref{lem: implications of infinite sampling} to argue that $\overline{\psi}_{n,I^*}\rightarrow \beta$ almost surely. The proof then uses Lemma \ref{lem: Over-allocation implies negligible probability} to show \eqref{eq: sufficient condition for optimality} holds, which  by Lemma \ref{lem: sufficient condition for optimality} is sufficient to establish that $\overline{\psibf}_n \to \psibf^{\beta}$. 
\subsubsection{Top-Two Thompson Sampling}
Recall that under top-two Thompson sampling, for every $i\in \{1,...,k\}$, 
\[ 
	\psi_{n,i} = \alpha_{n,i}\left(\beta + (1-\beta) \sum_{j\neq i} \frac{\alpha_{n,j}}{1-\alpha_{n,j}}\right).
\] 
\begin{proof}[Proof for TTTS]$ $\\
	\emph{Step 1: Show} $\overline{\psi}_{n,I^*} \to \beta$.  To begin, we show $\sum_{n\in \mathbb{N}} \psi_{n,i}= \infty$ for each design $i$. Suppose otherwise. Let $\mathcal{I} = \{i\in \{1,..,k\}: \sum_{1}^{\infty} \psi_{n,i}<\infty\}$ be the set of designs to which finite measurement effort is allocated. Under the TTTS sampling rule, $\psi_{n,i} \geq \beta \alpha_{n,i}$. Therefore, by Lemma \ref{lem: implications of infinite sampling}, if $i \in \mathcal{I}$ then $\underset{n\to \infty}{\lim\inf} \,\,\alpha_{n,i} >0$, which implies $\sum_{n \in \mathbb{N}}  \psi_{n,i} = \infty,$ a contradiction.
	
	Since $\sum_{1}^{\infty} \psi_{n,i}=\infty$ for all $i$, by applying Lemma \ref{lem: implications of infinite sampling} we conclude that $\alpha_{n,I^*} \to 1$.  For TTTS, this implies  $\overline{\psi}_{n,I^*} \to \beta$.\\
	\\
	\emph{Step 2: Show \eqref{eq: sufficient condition for optimality} holds.}
	By Lemma \ref{lem: sufficient condition for optimality}, it is enough to show that $\eqref{eq: sufficient condition for optimality}$ holds under TTTS. Let $\hat{I}_n = \arg\max_{i} \alpha_{n,i}$, and $\hat{J}_n = \arg\max_{i\neq \hat{I}_n} \alpha_{n,i}$. Since $\alpha_{n,I^*} \to 1$, for each sample path there is a finite time $\tau<\infty$ such that for all $n\geq \tau$, $\hat{I}_n = I^*$ and therefore $\hat{J}_n = \arg\max_{i\neq I^*} \alpha_{n,i}$. Under TTTS,
	\[ 
	\psi_{n,i} \leq 
	\beta\alpha_{n,i} + (1-\beta) \frac{\alpha_{n,i}}{\alpha_{n,J_n}} \leq \frac{\alpha_{n,i}}{\alpha_{n,J_n}},
	\] 
	where the first inequality follows since 
	\[
	\sum_{j\neq i} \frac{\alpha_{n,j}}{1-\alpha_{n,j}}\leq
	\frac{\sum_{j\neq i} \alpha_{n,i}}{1-\alpha_{n, \hat{I}_n}}
	 \leq \frac{\sum_{j\neq i} \alpha_{n,j}}{\alpha_{n,\hat{J}_n}} \leq \frac{1}{\alpha_{n,\hat{J}_n}}.
	\]
	For $n \geq \tau$, this means $\psi_{n,i} \leq \alpha_{n,i} / (\max_{j \neq I^*} \alpha_{n,i})$ for any $i\neq I^*$.
	By Lemma \ref{lem: Over-allocation implies negligible probability}, there is a constant $\delta'>0$ and a sequence $\epsilon_n \to 0$ such that
	\[
	\overline{\psi}_{n,i} \geq \psi^{\beta}_i + \delta \implies \frac{\alpha_{n,i}}{\max_{j \neq I^*} \alpha_{n,j}} \leq e^{-n(\delta' -\epsilon_n)}.
	\]
	Therefore for all $i\neq I^*$
	\[
	\sum_{n \geq \tau} \psi_{n,i} \mathbf{1}(\overline{\psi}_{n,i} \geq \psi^{\beta}_i+\delta ) \leq \sum_{n\geq \tau} e^{-n(\delta' - \epsilon_n)} < \infty.
	\]
	
\end{proof}

\subsubsection{Top-Two Probability Sampling}
Recall that top-two probability sampling sets $\psi_{n, \hat{I}_n} = \beta$ and $\psi_{n, \hat{J}_n} = 1-\beta$ where $\hat{I}_{n} = \arg\max_{i} \alpha_{n,i}$ and $\hat{J}_{n} = \arg\max_{j\neq \hat{I}_n} \alpha_{n,i}$ are the two designs with the highest posterior probability of being optimal. 
\begin{proof}[Proof for TTPS]$ $\\
\emph{Step 1: Show} $\overline{\psi}_{n,I^*} \to \beta$. To  begin, we show $\sum_{n\in \mathbb{N}} \psi_{n,i}= \infty$ for each design $i$. Suppose otherwise. Let $\mathcal{I} = \{i\in \{1,..,k\}: \sum_{1}^{\infty} \psi_{n,i}<\infty\}$ be the set of designs to which finite measurement effort is allocated.  Proceeding by contradiction, suppose $\mathcal{I}$ is nonempty. By Lemma \ref{lem: implications of infinite sampling}, there is a time $\tau$ and some probability $\alpha'>0$ such that $\alpha_{n,i}> \alpha'$ for all $n\geq \tau$ and $i\in \mathcal{I}$. However, because of the assumption that $\theta^*_i \neq \theta^*_j$, for $i\neq j$, $I= \arg\max_{i \notin \mathcal{I}} \theta^*_{i}$ is unique. By \eqref{eq: TTS consistency of infinitely sampled actions}, the algorithm identifies $\arg\max_{i \notin \mathcal{I}} \theta^*_{i}$ with certainty, and $\alpha_{n,i} \to 0$ for every $i\notin \mathcal{I}$ except for $I$. This means there is a time $\tau' > \tau$ such that for $n\geq \tau'$
\begin{eqnarray*}
\alpha_{n,i} > \alpha' &\text{if} & i\in \mathcal{I}  \\
\alpha_{n,i} \leq \alpha' &\text{if} & i\notin \mathcal{I} \,\, \text{and} \,\, i\neq I.   
\end{eqnarray*}
When this occurs at least one of the two designs with highest probability $\alpha_{n,i}$ of being optimal must be in the set $\mathcal{I}$, which implies designs in $\mathcal{I}$ receive infinite measurement effort, yielding a contradiction.  

 Since $\sum_{1}^{\infty} \psi_{n,i}=\infty$ for all $i$, Lemma \ref{lem: implications of infinite sampling} implies $\alpha_{n,I^*} \to 1$.  Therefore, there is a finite time $\tau$ such that $\hat{I}_n \triangleq \arg\max_{i} \alpha_{n,i} = I^*$ for all $n\geq \tau $. By the definition of the algorithm $\psi_{n, \hat{I}_n}=\beta$, and so $\psi_{n,I^*} = \beta$ for all $n\geq \tau$. We conclude that $\overline{\psi}_{n,I^*}\to \beta$.  
\\
\\
\emph{Step 2: Show $\eqref{eq: sufficient condition for optimality}$ holds.} As argued above, for each sample path there is a finite time $\tau<\infty$ such that for all $n\geq \tau$, $\hat{I}_n = I^*$ and therefore $\hat{J}_n = \arg\max_{i\neq I^*} \alpha_{n,i}$. By Lemma \ref{lem: Over-allocation implies negligible probability}, one can choose $\tau'\geq\tau$ such that for all $n\geq \tau'$,
\[
\overline{\psi}_{n,j} \geq \psi^{\beta}_j+\delta  \implies \alpha_{n,j} < \max_{i\neq I^*} \alpha_{n,i}
\]
and therefore by definition $\hat{J}_{n} \neq j$. This concludes the proof, as it shows that for each sample path there is  a finite time $\tau'$  after which TTPS never allocates any measurement effort to design $j$ when $\overline{\psi}_{n,j} \geq \psi^{\beta}_j+ \delta$.
\end{proof}

\subsubsection{Top-Two Value Sampling}
Recall that top-two value sampling sets 
 $\psi_{n, \hat{I}_n} = \beta$ and $\psi_{n, \hat{J}_n} = 1-\beta$ where $\hat{I}_{n} = \arg\max_{i} V_{n,i}$ and $\hat{J}_{n} = \arg\max_{j\neq \hat{I}_n} V_{n,i}$ are the two designs with the highest posterior value. 
\begin{proof}[Proof for TTVS]
\emph{Step 1: Show} $\overline{\psi}_{n,I^*} \to \beta$. The proof is essentially identical to that for TTPS.To  begin, we show $\sum_{n\in \mathbb{N}} \psi_{n,i}= \infty$ for each design $i$. Suppose otherwise. Let $\mathcal{I} = \{i\in \{1,..,k\}: \sum_{1}^{\infty} \psi_{n,i}<\infty\}$ be the set of designs to which finite measurement effort is allocated.  Proceeding by contradiction, suppose $\mathcal{I}$ is nonempty. By Lemma \ref{lem: implications of infinite sampling}, there is a time $\tau$ and some  $v>0$ such that $V_{n,i}> v$ for all $n\geq \tau$ and $i\in \mathcal{I}$. However, because of the assumption that $\theta^*_i \neq \theta^*_j$, for $i\neq j$, $I= \arg\max_{i \notin \mathcal{I}} \theta^*_{i}$ is unique\footnote{If the arg-max is not unique, then one could show that $V_{n,i}\rightarrow 0$ for all $i \notin \mathcal{I}$, and that therefore there is a finite time after both of the top-two designs are always in the set $\mathcal{I}$, yielding a contradiction}. By \eqref{eq: TTS consistency of infinitely sampled actions}, the algorithm identifies $\arg\max_{i \notin \mathcal{I}} \theta^*_{i}$ with certainty, and $V_{n,i} \to 0$ for every $i\notin \mathcal{I}$ except for $I$. Then there is a time $\tau' > \tau$ such that for $n\geq \tau'$
\begin{eqnarray*}
V_{n,i} > v &\text{if} & i\in \mathcal{I}  \\
V_{n,i} \leq v &\text{if} & i\notin \mathcal{I} \,\, \text{and} \,\, i\neq I^*.   
\end{eqnarray*}
When this occurs at least one of the two designs with highest value $V_{n,i}$ must be in the set $\mathcal{I}$, which implies designs in $\mathcal{I}$ receive infinite measurement effort, yielding a contradiction. 

 Since $\sum_{1}^{\infty} \psi_{n,i}=\infty$ for all $i$, Lemma \ref{lem: implications of infinite sampling} implies $V_{n,I^*} \to v_{I^*}(\thetabf^*)>0$ and $V_{n,i}\to 0$ for all $i\neq I^*$.  Therefore, there is a finite time $\tau$ such that $\arg\max_{i} V_{n,i} = I^*$ for all $\tau \geq n$. By the definition of the algorithm $\arg\max_{i}V_{n,i}$ is sampled with probability $\beta$, and so $\psi_{n,I^*} = \beta$ for all $n\geq \tau$. We conclude that $\overline{\psi}_{n,I^*}\to \beta$.  
\\
\\
\emph{Step 2: Show $\eqref{eq: sufficient condition for optimality}$ holds.} Again, the proof is essentially identical to that for TTPS. As argued above, for each sample path there is a finite time $\tau<\infty$ such that for all $n\geq \tau$, $\hat{I}_n = I^*$ and therefore $\hat{J}_n = \arg\max_{i\neq I^*} V_{n,i}$. By Lemma \ref{lem: value and probability are log equivalent}, $V_{n,i} \doteq \alpha_{n,i}$. Combining this with Lemma \ref{lem: Over-allocation implies negligible probability} shows one can choose $\tau'\geq\tau$ such that for all $n\geq \tau'$,
\[
\overline{\psi}_{n,j} \geq \psi^{\beta}_j+\delta  \implies V_{n,j} < \max_{i\neq I^*} V_{n,i}
\]
and therefore by definition $\hat{J}_{n} \neq j$. This concludes the proof, as it shows that for each sample path there is  a finite time $\tau'$  after which TTVS never allocates any measurement effort to design $j\neq I^*$ when $\overline{\psi}_{n,j} \geq \psi^{\beta}_j+ \delta$.
\end{proof}

\section{Results on Adaptive Tuning}

\begin{propn}[\ref{prop: optimal rate with adaptive tuning}]
	Suppose TTTS, TTVS, TTPS are applied with an adaptive sequence of tuning parameters $(\beta_{n} : n \in \mathbb{N})$ where for each $n$, $\beta_{n}$ is $\hist$ measurable. Then, with probability 1, on any sample path on which $\beta_n \to \beta^*$, 
\[
\Pi_{n}(\Theta_{I^*}^{c}) \doteq e^{-n \Gamma^*}.
\]
\end{propn}
\begin{proof}
\emph{Step 1: Show $\sum_{n\in \mathbb{N}} \psi_{n,i}= \infty$ for each design $i$.}
The proof follows identically to the case of fixed $\beta$. For example, consider the case of TTTS and, proceeding by contradiction, suppose $\sum_{n\in \mathbb{N}} \psi_{n,i} <\infty$ on some sample path. Under TTTS, $\psi_{n,i} \geq \beta_{n} \alpha_{n,i}$. Therefore, by Lemma \ref{lem: implications of infinite sampling}, if $i \in \mathcal{I}$ then $\underset{n\to \infty}{\lim\inf} \,\,\alpha_{n,i} >0$, so $\liminf_{n\in\mathbb{N}} \psi_{n,i}>0$. This implies $\sum_{n \in \mathbb{N}}  \psi_{n,i} = \infty,$ a contradiction. Proofs for TTPS and TTVS also follows as before, and are omitted\\
\\	
\emph{Step 2: Show that $\overline{\psi}_{n, I^*} \to \beta^*$}

It is sufficient to show that $\psi_{n,I^*} - \beta_{n} \to 0$. This would imply $\frac{1}{n}\sum_{\ell=1}^{n} \left(\psi_{\ell,I^*} - \beta_{\ell} \right) \to 0$, which since $\beta_{n}\to \beta^*$, implies $\frac{1}{n}\sum_{\ell=1}^{n} \psi_{\ell,I^*} \to \beta^*$ as desired. 

Now, since $\sum_{n\in \mathbb{N}} \psi_{n,i} =\infty$ for all arms $i$, $\alpha_{n,I^*} \to 1$ Lemma \ref{lem: implications of infinite sampling} implies 
\[
V_{n,i} \to
\begin{cases}
0  \qquad \text{if } i\neq I^*\\
v_{I^*}(\thetabf^*) >0 \quad \text{if } i=I^*
\end{cases}
\qquad {\rm and} \qquad \alpha_{n,i} \to
\begin{cases}
0  \quad \text{if } i\neq I^*\\
1 \quad \text{if } i=I^*.
\end{cases}
\]
For top-two probability sampling, this implies there exists a time after which $\arg\max_{i} \alpha_{n,i}=I^*$ and hence $\psi_{n,I^*} = \beta_{n}$ for all $n$ sufficiently large. For top-two value sampling, the same result applies, since there exists a time after which $\arg\max_{i} V_{n,i}=I^*$. For top-two Thompson sampling, 
\[
\psi_{n,i} = \alpha_{n,i}\left(\beta_n + (1-\beta_n) \sum_{j\neq i} \frac{\alpha_{n,j}}{1-\alpha_{n,j}}\right).
\]
from which we conclude $\psi_{n,I^*}-\beta_n \to 0$ as $\alpha_{n,I^*}\to 1$. \\
\\
\emph{Step 3: Show sufficient condition for optimality in \eqref{eq: sufficient condition for optimality}}.  
By Lemma \ref{lem: sufficient condition for optimality}, it is enough to show 
\[
\sum_{n \in \mathbb{N}} \psi_{n,j} \mathbf{1}(\overline{\psi}_{n,j} \geq \psi^{\beta^*}_{j}+\delta) < \infty \qquad \forall \,\, j \neq I^*, \, \delta>0,
\]
For each proposed algorithm, a proof of the corresponding result was given in Step 2 of Subsection \ref{subsec: results specific to proposed algorithms}, but for the case of arbitrary $\beta\in (0,1)$. Since $\overline{\psi}_{n, I^*} \to \beta^*$, for each of proposed algorithm the proof of this follows line by line as before, but replacing $\beta$ with $\beta^*$ everywhere it occurs. 

%\arg\max_{\psibf} \min_{\thetabf \in \Theta_{\hat{I}}^{c}} D_{\psi}(\hat{\thetabf} || \thetabf)$ }

\end{proof}

\begin{lemman}[\ref{lem : consistency of adaptive tuning}]
	Under TTTS, TTPS, or TTVS with an adaptive sequence of tuning parameters $(\beta_{n} : n \in \mathbb{N})$ adjusted according to Algorithm \ref{alg: TTTS tuned}, $\beta_{n}\to \beta^*$ almost surely. Therefore $\Pi_{n}(\Theta_{I^*}^{c}) \doteq e^{-n \Gamma^*}$.
\end{lemman}
\begin{proof}
	First, let us define some notation. Let $\hat{\thetabf}_{n} = \intop_{\thetabf \in \Theta} \thetabf \pi_{n}(\thetabf) d\thetabf$ denote the posterior mean at time $n$. Recall that $\beta_{n}$ denotes the tuning parameter used by the top-two algorithm at time $n$ and this is updated only at certain time periods. Define 
	\[ 
	\ell_{n} = \max\{ \ell \in \N :   \min_{i\in \{1,\ldots k\}} S_{n,i} \geq  \kappa^{\ell}  \} = \left\lfloor \log_{\kappa}\left( \min_{i} S_{n,i} \right) \right\rfloor
	\] 
	denote the number of time periods in which an update to $\hat{\beta}_n$ has been attempted. Now, the proof proceeds in two steps. 
	\\
	\\
	\emph{Step 1: Show $\sum_{n\in \mathbb{N}} \psi_{n,i} = \infty$ for all $i \in \{1,\ldots, k\}$ almost surely}
	Suppose otherwise. Then by Corollary 1, \ref{cor: convergence of action selection probabilities to the empirical distribution}, we know that there is some arm $i\in \{1,\ldots k\}$ with $\lim_{n\to \infty} S_{n,i} <\infty$. This in tern implies $\sup_{n\in \mathbb{N}} \ell_{n} <\infty$, so on this sample path there exists a time $N$ with $\beta_{n}=\beta_{N}$ for all $n\geq N$. Let us consider the sample path from time $N$ onwards. We have concluded that there is an infinite period of times $\{N_1,N_2, \ldots \}$ over which (1) Top-two sampling is applied with a constant parameter $\beta_{n}=\beta_{N}$ with initial beliefs $\pi_{N}$ over the hyper-rectangle $\Theta=(\underline{\theta}, \overline{\theta} )^{k}$ and (2) there exists an arm $i$ with $\sum_{n=N}^{\infty} \psi_{n,i}< \infty$. We know by Proposition \ref{prop: TS converges to optimal allocation} that the set of such sample paths has measure zero. \\
	\\
	\emph{Step 2: Show that therefore $\beta_{n} \to \beta^*$} We first show the consistency of the posterior mean $\hat{\thetabf}_n = \intop_{\thetabf \in \Theta} d\pi_{n}(\thetabf)$. Since $\sum_{n \in \N} \psi_{n,i}=\infty$ for all $i$, Proposition \ref{prop: posterior consistency} implies that for any open set $\tilde{\Theta}$ containing $\thetabf^*$, $\Pi_{n}\left(\tilde{\Theta} \right) \to 1$. Since $\Theta$ is compact, this implies $\hat{\thetabf}_{n} \to \thetabf^*$ almost surely. Now, because the function $f(\psibf; \thetabf):=\min_{\thetabf' \in \Theta_{\hat{I}}^{c}} D_{\psi}(\thetabf || \thetabf')$ is continuous in both arguments, the correspondence $\thetabf \mapsto \arg\max_{\psi} f(\psibf, \thetabf)$ is upper hemi-continuous at any $\thetabf$. Since $\psi^*(\thetabf^*)=\arg\max_{\psibf} f(\psibf, \thetabf^*)$ is unique, we know $\psi^*(\thetabf)$ is continuous in a neighborhood of $\thetabf^*$. This implies $\psi^*(\hat{\thetabf}_n) \to \psi^*(\thetabf^*)$. Since the parameter $\beta_{n}$ is updated an infinite number of times, this implies $\beta_{n}\to \beta^*$. \\
	\\
\end{proof}

\end{document}